%% file: neutral.tex
\documentclass{article}

\usepackage[T1]{fontenc}
\usepackage[htt]{hyphenat}

\usepackage{microtype}
\usepackage{graphicx}
\usepackage{subfigure}
\usepackage{booktabs} 
\usepackage{natbib}
\usepackage{algorithm}
\usepackage{algorithmic}
\input{preamble}

\usepackage{breakcites}
\usepackage{authblk}
\usepackage[margin=1in]{geometry}
\usepackage{enumitem}

\usepackage{hyperref}

\hypersetup{
     colorlinks = true,
     linkcolor = brown,
     anchorcolor = blue,
     citecolor = teal,
     filecolor = blue,
     urlcolor = black
     }

\title{On Energy-Based Models with Overparametrized Shallow Neural Networks}

\author[a]{Carles Domingo-Enrich}
\author[b]{Alberto Bietti}
\author[a]{Eric Vanden-Eijnden}
\author[a,b]{Joan Bruna}
\affil[a]{Courant Institute of Mathematical Sciences, New York University}
\affil[b]{Center for Data Science, New York University}

\begin{document}





\maketitle



\begin{abstract}
Energy-based models (EBMs) are a simple yet powerful framework for generative modeling. They are based on a trainable energy function which defines an associated Gibbs measure, and they can be trained and sampled from via well-established statistical tools, such as MCMC. Neural networks may be used as energy function approximators, providing both a rich class of expressive models as well as a flexible device to incorporate data structure.
In this work we focus on shallow neural networks. Building from the incipient theory of overparametrized neural networks, we show that models trained in the so-called ``active'' regime provide a statistical advantage over their associated ``lazy'' or kernel regime, leading to improved adaptivity to hidden low-dimensional structure in the data distribution, as already observed in supervised learning. Our study covers both maximum likelihood and Stein Discrepancy estimators, and we validate our theoretical results with numerical experiments on synthetic~data.
\end{abstract}

\section{Introduction}
\label{sec:introduction}
\input{introduction}

\section{Related work}
\label{sec:related_work}
\input{related_work}

\section{Setting}
\label{sec:setting}
\input{setting}

\section{Statistical guarantees for shallow neural network EBMs}
\label{sec:statistical_guarantees}
\input{statistical_neutral}

\section{Algorithms}
\label{sec:algorithms}
\input{algorithms}

\section{Experiments}
\label{sec:experiments}
\input{experiments_neutral}

\section{Conclusions and discussion}
\label{sec:conclusions}
\input{conclusions}

\newpage

\bibliography{biblio}

\onecolumn

\appendix
\section{Proofs of \autoref{sec:statistical_guarantees}}
\label{sec:proofs_statistical}
\input{proofs_statistical}

\section{Qualitative convergence results}
\label{sec:convergence}
\input{convergence}

\section{Additional experiments}
\label{sec:additional_exp}
\input{additional_exp_neutral}

\section{Duality theory for $\mathcal{F}_1$ and $\mathcal{F}_2$ MLE EBMs}
\label{sec:duality}
\input{duality}

\section{Proofs of \autoref{sec:duality}}
\label{sec:duality_proofs}
\input{proofs_duality}



\end{document}

%% file: preamble.tex
\usepackage{amsmath}
\usepackage{amssymb}
\usepackage{graphicx}
\usepackage{mathtools}
\usepackage{amsfonts}
\usepackage{amsthm,bm}
\usepackage{color}
\usepackage[table]{xcolor}
\usepackage{enumerate}
\usepackage{dsfont}
\usepackage{pgfplots}
\pgfplotsset{width=10cm,compat=1.9}
 \usepackage{pifont}
\usepackage{thmtools} 
\usepackage{thm-restate}
\usepackage{sidecap}

\newcommand{\R}{\mathbb{R}}

\renewcommand{\phi}{\varphi}

\graphicspath {{figures/}}

\usepackage{hyperref}
\usepackage{caption}
\usepackage[super]{nth}
\newtheorem{lemma}{Lemma}

\newtheorem{assumption}{Assumption}

\newtheorem{observation}{Observation}

\mathtoolsset{showonlyrefs}

\DeclareMathOperator*{\argmin}{argmin}
\DeclareMathOperator*{\argmax}{argmax}

%% file: introduction.tex
A central problem in machine learning is to learn generative models of a distribution through its samples.
Such models may be needed simply as a modeling tool in order to discover properties of the data, or as a way to generate new samples that are similar to the training samples.
Generative models come in various flavors. In some cases very few assumptions are made on the distribution and one simply tries to learn generator models in a black-box fashion~\citep{goodfellow2014generative,kingma2013auto}, while other approaches make more precise assumptions on the form of the data distribution.
In this paper, we focus on the latter approach, by considering Gibbs measures defined through an \emph{energy function}~$f$, with a density proportional to~$\exp\{-f(x)\}$.
Such \emph{energy-based models} (EBMs) originate in statistical physics~\citep{ruelle1969statistical}, and have become a fundamental modeling tool in statistics and machine learning~\citep{wainwright2008graphical,ranzato2007efficient,lecun2006tutorial,du2019implicit,song2021train}.
If data is assumed to come from such a model, the learning algorithms then attempt to estimate the energy function~$f$. The resulting learned model can then be used to obtain new samples, typically through Markov Chain Monte Carlo (MCMC) techniques.


In this paper, we study the statistical problem of learning such EBMs from data, in a non-parametric setting defined by a function class $\mathcal{F}$, and with possibly arbitrary target energy functions.
If we only assume a simple Lipschitz property on the energy, learning such models will generally suffer from the curse of dimensionality \citep{von2004distance}, in the sense that an exponential number of samples in the dimension is needed to find a good model.
However, one may hope to achieve better guarantees when additional structure is present in the energy function.

An important source of structure comes from energy functions which capture local rather than global interactions between input features, such as those in Local Markov Random Fields or Ising models.  Such energies can be expressed as linear combinations of potential functions depending only on low-dimensional projections, and are therefore amenable to efficient approximation by considering classes $\mathcal{F}$ given by shallow neural networks endowed with a sparsity-promoting norm \citep{bach2017breaking}. Analogously to the supervised regime \citep{bach2017breaking,chizat2020implicit}, learning in such \emph{variation-norm} spaces $\mathcal{F}=\mathcal{F}_1$ admits a convex formulation in the overparametrized limit, whose 
corresponding class of Gibbs measures $\{\nu(dx) \propto \exp\{-f(dx)\},\,f\in \mathcal{F}_1 \}$ is the natural infinite-dimensional extension of exponential families \citep{wainwright2008graphical}. 
Our main contribution is to show that such EBMs lead to a well-posed learning setup with strong statistical guarantees, breaking the curse of dimensionality.

These statistical guarantees can be combined with qualitative optimization guarantees in this overparamerised limit under an appropriate `active' or `mean-field' scaling \citep{mei2018mean,rotskoff2018neural,chizat2018global,sirignano2019mean}. 
As it is also the case for supervised learning, the benefits of variation-norm spaces $\mathcal{F}_1$ contrast with their RKHS counterparts $\mathcal{F}_2$, which cannot efficiently adapt to  the low-dimensional structure present in such structured Gibbs models. 

The standard method to train EBMs is maximum likelihood estimation.
One generic approach for this is to use gradient descent, where gradients may be approximated using MCMC samples from the current trained model.
Such sampling procedures may be difficult in general, particularly for complex energy landscapes, thus we also consider different estimators based on un-normalized measures which avoid the need of sampling.
We focus here on approaches based on minimizing Stein discrepancies~\citep{gorham2015measuring,liu2016stein}, which have recently been found to be useful in deep generative models~\citep{grathwohl2020learning},
though we note that alternative approaches may be used, such as score matching~\citep{hyvarinen05estimation,song2021train,song2019generative,block2020generative}.

Our main focus is to study the resulting estimators when using gradient-based optimization over infinitely-wide neural networks in different regimes, showing the statistical benefits of the `feature learning' regime when the target models have low-dimensional structure, thus extending the analogous results for supervised least-squares \citep{bach2017breaking} and logistic \citep{chizat2020implicit} regression. 
More precisely, we make the following contributions:
\begin{itemize}[itemsep=0pt,topsep=1pt,leftmargin=0.3cm]
    \item We derive generalization bounds for the learned measures in terms of the same metrics used for training (KL divergence or Stein discrepancies). Using and extending results from the theory of overparametrized neural networks, we show that when using energies in the class~$\mathcal{F}_1$ we can learn target measures with certain low-dimensional structure at a rate controlled by the intrinsic dimension rather than the ambient dimension (\autoref{cor:statf1} and \autoref{cor:corsdsphere}).
    \item We show in experiments that while $\mathcal{F}_1$ energies succeed in learning simple synthetic distributions with low-dimensional structure, $\mathcal{F}_2$ energies fail (\autoref{sec:experiments}).
\end{itemize}

%% file: related_work.tex
A recent line of research has studied the question of how neural networks compare to kernel methods, with a focus on supervised learning problems.
\citet{bach2017breaking} studies two function classes that arise from infinite-width neural networks with different norms penalties on its weights, leading to the two different spaces~$\mathcal F_1$ and~$\mathcal F_2$, and shows the approximation benefits of the~$\mathcal F_1$ space for adapting to low-dimensional structures compared to the (kernel) space~$\mathcal F_2$, an analysis that we leverage in our work.
The function space~$\mathcal F_1$ was also studied by~\citet{ongie2019function,savarese2019infinite,williams2019gradient} by focusing on the ReLU activation function. 
More recently, this question has gained interest after several works have shown that wide neural networks trained with gradient methods may behave like kernel methods in certain regimes~\citep[see, e.g.,][]{jacot2018neural}.
Examples of works that compare `active/feature learning' and `kernel/lazy' regimes include~\citep{chizat2020implicit,ghorbani2019limitations,wei2020regularization,woodworth2020kernel}.
We are not aware of any works that study questions related to this in the context of generative models in general and EBMs in particular. 


Other related work includes the Stein discrepancy literature. Although Stein's method \citep{stein1972abound} dates to the 1970s, it has been popular in machine learning in recent years. \citet{gorham2015measuring} introduced a computational approach to compute the Stein discrepancy in order to assess sample quality. Later, \citet{chwialkowski2016gretton} and \citet{liu2016akernelized} introduced the more practical kernelized Stein discrepancy (KSD) for goodness-of-fit tests, which was also studied by \citet{gorham2017measuring}. 
\citet{liu2016stein} introduced SVGD, which was the first method to use the KSD to obtain samples from a distribution, and \citet{barp2019minimum} where the first to employ KSD to train parametric generative models.
More recently, \citet{grathwohl2020learning} used neural networks as test functions for Stein discrepancies, which arguably yields a stronger metric, and have shown how to leverage such metrics for training EBMs. The empirical success of their method provides an additional motivation for our theoretical study of the $\mathcal{F}_1$ Stein Discrepancy (\autoref{subsec:guarantees_sd}).

Finally, another notable paper close in spirit to our goal is \cite{block2020generative}, which provides a detailed theoretical analysis of a score-matching generative model using Denoising Autoencoders followed by Langevin diffusion. While their work makes generally weaker assumptions and also includes a non-asymptotic analysis of the sampling algorithm, the resulting rates are unsuprisingly cursed by dimension. Our focus is on the statistical aspects which allow faster rates, leaving the quantitative computational aspects aside. 







%% file: setting.tex
In this section, we present the setup of our work, recalling basic properties of EBMs, maximum likelihood estimators, Stein discrepancies, and functional spaces arising from infinite-width shallow neural networks.

\paragraph{Notation.}
If $V$ is a normed vector space, we use $\mathcal{B}_V(\beta)$ to denote the closed ball of $V$ of radius $\beta$, and $\mathcal{B}_V := \mathcal{B}_V(1)$ for the unit ball. If $K$ denotes a subset of the Euclidean space, $\mathcal{P}(K)$ is the set of Borel probability measures, $\mathcal{M}(K)$ is the space of signed Radon measures and $\mathcal{M}^{+}(K)$ is the space of (non-negative) Radon measures. For $\nu_1, \nu_2 \in \mathcal{P}(K)$, we define the Kullback-Leibler (KL) divergence $D_{\text{KL}}(\nu_1|| \nu_2) := \int_{K} \log(\frac{d\nu_1}{d\nu_2}(x)) d\nu_1(x)$ when $\nu_1$ is absolutely continuous with respect to $\nu_2$, and $+\infty$ otherwise, and the cross-entropy $H(\nu_1,\nu_2) := - \int_{K} \log(\frac{d\nu_2}{d\tau}(x)) d\nu_1(x)$, where $\frac{d\nu_2}{d\tau}(x)$ is the Radon-Nikodym derivative w.r.t. the uniform probability measure $\tau$ of $K$, and the differential entropy $H(\nu_1) := - \int_{K} \log(\frac{d\nu_1}{d\tau}(x)) d\nu_1(x)$. If $\gamma$ is a signed measure over $K$, then ${|\gamma|}_{\text{TV}}$ is the total variation (TV) norm of $\gamma$. $\mathbb{S}^d$ is the $d$-dimensional hypersphere, and for functions $f : \mathbb{S}^d \rightarrow \R$, $\nabla f$ denotes the Riemannian gradient of $f$. We use $\sigma(\langle \theta, x \rangle) = \max\{0, \langle \theta, x \rangle\}$ to denote a ReLU with parameter $\theta$.

\subsection{Generative energy-based models} \label{subsec:generative_EBMs}
If $\mathcal{F}$ is a class of functions (or energies) mapping a measurable set $K \subseteq \R^{d+1}$ to $\R$, for any $f \in \mathcal{F}$ we can define the probability measure $\nu_{f}$ as a Gibbs measure with density:
\begin{align} \label{eq:nu_f_def}
    \frac{d\nu_{f}}{d\tau}(x) := \frac{e^{-f(x)}}{Z_f}, \text{ with } Z_f := \int_{K} e^{-f(y)} d\tau(y)~,
\end{align}
where $\frac{d\nu_{f}}{d\tau}(x)$ is the Radon-Nikodym derivative w.r.t to the uniform probability measure over $K$, denoted~$\tau$, and $Z_f$ is the partition function. 

Given samples $\{x_i\}_{i=1}^n$ from a target measure $\nu$, training an EBM consists in selecting the best $\nu_f$ with energy $f\in \mathcal{F}$ according to a given criterion. A natural estimator $\hat{f}$ for the energy is the \textbf{maximum likelihood} estimator (MLE), i.e., $\hat{f} = \argmax_{f \in \mathcal{F}} \prod_{i=1}^n  \frac{d\nu_{f}}{d\tau}(x_i)$,
or equivalently, the one that minimizes the cross-entropy with the samples:
\begin{align} 
\begin{split} \label{eq:hat_mu}
    \hat{f} &= \argmin_{f \in \mathcal{F}} H(\nu_n, \nu_{f}) = \argmin_{f \in \mathcal{F}} -\frac{1}{n} \sum_{i=1}^n \log \left(\frac{d\nu_{f}}{d\tau}(x_i) \right)\\ &= \argmin_{f \in \mathcal{F}} \frac{1}{n} \sum_{i=1}^n f(x_i) + \log Z_f .
\end{split}
\end{align}
The estimated distribution is simply $\nu_{\hat{f}}$, and samples can be obtained by the MCMC algorithm of choice.

An alternative estimator is the one that arises from minimizing the \textbf{Stein discrepancy} (SD) corresponding to a function class $\mathcal{H}$. If $\mathcal{H}$ is a class of functions from $K$ to $\R^{d+1}$, the Stein discrepancy \citep{gorham2015measuring, liu2016akernelized} for $\mathcal{H}$ is a non-symmetric functional defined on pairs of probability measures over $K$ as
\begin{align} \label{eq:sd_def}
    \text{SD}_{\mathcal{H}}(\nu_1, \nu_2) = \sup_{h \in \mathcal{H}} \mathbb{E}_{\nu_1} [\text{Tr}(\mathcal{A}_{\nu_2} h(x))],
\end{align}
where $\mathcal{A}_{\nu} : K \rightarrow \R^{(d+1)\times(d+1)}$ is the Stein operator.
In order to leverage approximation properties on the sphere, we will consider functions~$h$ defined on~$K = \mathbb S^d$.
In this case,
the Stein operator is defined by $\mathcal{A}_{\nu} h(x) := (s_{\nu}(x) - d \cdot x) h(x)^{\top} + \nabla h(x)$ (see \autoref{lem:stein_operator}), where $s_{\nu}(x) = \nabla \log(\frac{d\nu}{d\tau}(x))$ is named the score function.
The term~$d \cdot x$ is important for the spherical case in order to have~$\text{SD}_{\mathcal{H}}(\nu, \nu) = 0$, while it does not appear when considering~$K=\R^d$.
The Stein discrepancy estimator is
\begin{align} \label{eq:sd_estimator}
    \hat{f} = \argmin_{f \in \mathcal{F}} \text{SD}_{\mathcal{H}}(\nu_n, \nu_{f}).
\end{align}

If $\mathcal{H} = \mathcal{B}_{\mathcal{H}_0^{d+1}} = \{ (h_i)_{i=1}^{d+1} \in \mathcal{H}_0^{d+1} \ | \ \sum_{i=1}^{d+1} \|h_i\|_{\mathcal{H}_0}^2 \leq 1 \}$ for some reproducing kernel Hilbert space (RKHS) $\mathcal{H}_0$ with kernel $k$ with continuous second order partial derivatives, there exists a closed form for the problem \eqref{eq:sd_def} and the corresponding object is known as \textbf{kernelized Stein discrepancy} (KSD) \citep{liu2016akernelized, gorham2017measuring}.
For $K = \mathbb{S}^d$, the KSD takes the following form (\autoref{lem:ksd_expression}):
\begin{align} \label{eq:KSD_def}
    \text{KSD}(\nu_1, \nu_2) = \text{SD}^2_{\mathcal{B}_{\mathcal{H}_0^{d+1}}}(\nu_1, \nu_2) = \mathbb{E}_{x,x' \sim \nu_1} [u_{\nu_2}(x,x')], 
\end{align}
where $u_{\nu}(x,x') = (s_{\nu}(x) - d \cdot x)^\top(s_{\nu}(x') - d \cdot x') k(x,x') + (s_{\nu}(x) - d \cdot x)^\top \nabla_{x'} k(x,x') + (s_{\nu}(x') - d \cdot x')^\top \nabla_{x} k(x,x') + \text{Tr}(\nabla_{x,x'} k(x,x'))$, and we use $\tilde{u}_{\nu}(x,x')$ to denote the sum of the first three terms (remark that the fourth term does not depend on $\nu$). One KSD estimator that can be used is
\begin{align} \label{eq:ksd_estimator}
    \hat{f} = \argmin_{f \in \mathcal{F}} \frac{1}{n^2} \sum_{i,j=1}^n \tilde{u}_{\nu_{f}}(x_i,x_j).
\end{align}
The optimization problem for this estimator is convex (\autoref{sec:algorithms}), but it is biased. On the other hand, the estimator
\begin{align} \label{eq:ksd_estimator2}
    \hat{f} = \argmin_{f \in \mathcal{F}} \frac{1}{n(n-1)} \sum_{i \neq j} \tilde{u}_{\nu_{f}}(x_i,x_j),
\end{align}
is unbiased, but the optimization problem is not convex. 

\subsection{Neural network energy classes}
We are interested in the cases in which $\mathcal{F}$ is one of two classes of functions related to shallow neural networks, as studied by \citet{bach2017breaking}.

\paragraph{Feature learning regime.}
$\mathcal{F}$ is the ball $\mathcal{B}_{\mathcal{F}_1}(\beta)$ of radius $\beta > 0$ of $\mathcal{F}_1$, which is the Banach space of functions $f : K \rightarrow \R$ such that for all $x \in K$ we have $f(x) = \int_{\mathbb{S}^d} \sigma(\langle \theta, x \rangle) \ d\gamma(\theta)$, for some Radon measure $\gamma \in \mathcal{M}(\mathbb{S}^{d})$. The norm of $\mathcal{F}_1$ is defined as
    $\|f\|_{\mathcal{F}_1} = \inf \left\{ {|\gamma|}_{\text{TV}} \ | \ f(\cdot) = \int_{\mathbb{S}^d} \sigma(\langle \theta, \cdot \rangle) \ d\gamma(\theta) \right\}.$
   
\paragraph{Kernel regime.}
$\mathcal{F}$ is the ball $\mathcal{B}_{\mathcal{F}_2}(\beta)$ of radius $\beta > 0$ of $\mathcal{F}_2$, which is the (reproducing kernel) Hilbert space of functions $f : K \rightarrow \R$ such that for some absolutely continuous $\rho \in \mathcal{M}(\mathbb{S}^d)$ with $\frac{d\rho}{d\tilde{\tau}} \in \mathcal{L}^2(\mathbb{S}^d)$ (where $\tilde{\tau}$  the uniform probability measure over $\mathbb{S}^d$), we have that for all $x \in K$,  $f(x) = \int_{\mathbb{S}^d} \sigma(\langle \theta, x \rangle) \ d\rho(\theta)$. The norm of $\mathcal{F}_2$ is defined as $\|f\|_{\mathcal{F}_2}^2 = \inf \left\{ \int_{\mathbb{S}^d} |h(\theta)|^2 \ d\tilde{\tau}(\theta) \ | \ f(\cdot) = \int_{\mathbb{S}^d} \sigma(\langle \theta, \cdot \rangle) h(\theta) \ d\tilde{\tau}(\theta) \right\}$. As an RKHS, the kernel of $\mathcal{F}_2$ is $k(x,y) = \int_{\mathbb{S}^d} \sigma(\langle x, \theta \rangle) \sigma(\langle y, \theta \rangle) \ d\tilde{\tau}(\theta)$. 

Remark that since $\int |h(\theta)| d\tilde{\tau}(\theta) \leq (\int |h(\theta)|^2 \ d\tilde{\tau}(\theta))^{1/2}$ by the Cauchy-Schwarz inequality, we have $\mathcal{F}_2 \subset \mathcal{F}_1$ and $\mathcal{B}_{\mathcal{F}_2} \subset \mathcal{B}_{\mathcal{F}_1}$.
The TV norm in~$\mathcal F_1$ acts as a sparsity-promoting penalty, which encourages the selection of few well-chosen neurons and may lead to favorable adaptivity properties when the target has a low-dimensional structure.
In particular, \cite{bach2017breaking} shows that single ReLU units belong to $\mathcal{F}_1$ but not to $\mathcal{F}_2$, and their $L^2$ approximations in $\mathcal{F}_2$ have exponentially high norm in the dimension. Ever since, several works have further studied the gaps arising between such nonlinear and linear regimes \citep{wei2019regularization,ghorbani2020neural,malach2021quantifying}.
In \autoref{sec:duality}, we present dual characterizations of the maximum likelihood $\mathcal{F}_1$ and $\mathcal{F}_2$ EBMs as entropy maximizers under $L^{\infty}$ and $L^2$ moment constraints (an infinite-dimensional analogue of \citet{pietra1997inducing}; see also \citet{mohri2012foundations}, Theorem 12.2).

The ball radius $\beta$ acts as an inverse temperature. The low temperature regime $\beta \gg 1$ corresponds to expressive models with lower approximation error but higher statistical error: the theorems in \autoref{sec:statistical_guarantees} provide bounds on the two errors and the results of optimizing such bounds w.r.t. $\beta$. 
In the following, we will assume that the set $K \subset \mathbb{R}^{d+1}$ is compact. We note that there are two interesting choices for $K$: (i) for $K = \mathbb{S}^d$, we obtain neural networks without bias term; and (ii) for $K = K_0 \times \{R\}$, where $K_0 \subset \R^d$ with norm bounded by $R$, we obtain neural networks on $K_0$ with a bias~term.

%% file: statistical_neutral.tex
In this section, we present our statistical generalization bounds for various EBM estimators based on maximum likelihood and Stein discrepancies, highligting the adaptivity to low-dimensional structures that can be achieved when learning with energies in $\mathcal F_1$. All the proofs are in \autoref{sec:proofs_statistical}.

\subsection{Guarantees for maximum likelihood EBMs}
\label{sub:mle_bounds}
The following theorem provides a bound of the KL divergence between the target probability measure and the maximum likelihood estimator in terms of a statistical error and an approximation error. 
\begin{restatable}{thm}{generalbound} \label{thm:generalbound}
Assume that the class $\mathcal{F}$ has a (distribution-free) Rademacher complexity bound $\mathcal{R}_n(\mathcal{F}) \leq \frac{\beta C}{\sqrt{n}}$ and $L^{\infty}$ norm uniformly bounded by $\beta$.
Given $n$ samples $\{x_i\}_{i=1}^n$ from the target measure $\nu$, consider the maximum likelihood estimator (MLE) $\hat{\nu} := \nu_{\hat{f}}$, where $\hat{f}$ is the estimator defined in \eqref{eq:hat_mu}. With probability at least $1-\delta$, we have
\begin{align} \label{eq:thm1_1}
    D_{\text{KL}}(\nu || \hat{\nu}) \leq 
    \frac{4\beta C}{\sqrt{n}} + \beta \sqrt{\frac{8 \log(1/\delta)}{n}} + \inf_{f \in \mathcal{F}} D_{KL}(\nu || \nu_{f}).
\end{align}
If $\frac{d\nu}{d\tau}(x) = e^{- g(x)}/\int_{K} e^{- g(y)} d\tau(y)$ for some $g : K \rightarrow \R$, i.e. $-g$ is the log-density of $\nu$ up to a constant term, then with probability at least $1-\delta$,
\begin{align} \label{eq:thm1_2}
    D_{\text{KL}}(\nu || \hat{\nu}) \leq \frac{4\beta C}{\sqrt{n}} + \beta \sqrt{\frac{8 \log(1/\delta)}{n}} + 2\inf_{f \in \mathcal{F}} \|g-f\|_{\infty}.
\end{align}
\end{restatable}

Equation \eqref{eq:thm1_1} follows from using a classical argument in statistical learning theory. To obtain equation \eqref{eq:thm1_2} we bound the last term of \eqref{eq:thm1_1} by $2\inf_{f \in \mathcal{F}} \|g-f\|_{\infty}$ using \autoref{lem:kl_l_infty} in \autoref{sec:proofs_statistical}. 
We note that other metrics than~$L_\infty$ may be used for the approximation error, such as the Fisher divergence, but these will likely lead to similar guarantees under our assumptions.
Making use of the bounds developed by \cite{bach2017breaking}, \autoref{cor:statf1} below applies~\eqref{eq:thm1_2} to the case in which $\mathcal{F}$ is the $\mathcal{F}_1$ ball $\mathcal{B}_{\mathcal{F}_1}(\beta)$ for some $\beta > 0$ and the energy of the target distribution is a sum of Lipschitz functions of orthogonal projection to low-dimensional subspaces.

\begin{assumption} \label{ass:mle_assumption}
The target probability measure $\nu$ is absolutely continuous w.r.t. the uniform probability measure $\tau$ over $K$ and it satisfies $\forall x \in K_0, \ \frac{d\nu}{d\tau}(x,R) = \exp(-\sum_{j=1}^J \phi_j(U_j x))/\int_{K_0} \exp(-\sum_{j=1}^J \phi_j(U_j y)) d\tau(y)$, where $\phi_j$ are $(\eta R^{-1})$-Lipschitz continuous functions on the $R$-ball of $\R^k$ such that $\|\phi_j\|_{\infty} \leq \eta$, and $U_j \in \R^{k \times d}$ with orthonormal rows.
\end{assumption}

\begin{restatable}{cor}{statfone} \label{cor:statf1}
Let $\mathcal{F} = B_{\mathcal{F}_1}(\beta)$. Suppose $K = K_0 \times \{R\}$, where $K_0 \subseteq \{ x \in \R^d | \|x\|_2 \leq R\}$ is compact. Assume that \autoref{ass:mle_assumption} holds.
Then, we can choose $\beta > 0$ such that with probability at least $1-\delta$ we have
\begin{align}
    D_{KL}(\nu || \hat{\nu}) \leq \tilde{O} \left( \left(1+\sqrt{\log(1/\delta)} \right) J \eta R^{-\frac{2}{k+3}} n^{-\frac{1}{k+3}} \right)
\end{align}
where the notation $\tilde{O}$ indicates that we overlook logarithmic factors and constants depending only on the dimension $k$.
\end{restatable}

Remarkably, \autoref{cor:statf1} shows that for our class of target measures with low-dimensional structure, the
KL divergence between $\nu$ and $\hat{\nu}$ decreases as $n^{-\frac{1}{k+3}}$.
That is, the rate ``breaks'' the curse of dimensionality since the exponent only depends on the dimension $k$ of the low-dimensional spaces, not to the ambient dimension $d$.
This can be seen as an alternative, more structural approach to alleviate dimension-dependence compared to other standard assumptions such smoothness classes for density estimation~\citep[e.g.,][]{singh2018nonparametric,tsybakov2008introduction}.
As discussed earlier, a motivation for \autoref{ass:mle_assumption} comes from Markov Random Fields, where each $\varphi_j$ corresponds to a local potential defined on a neighborhood determined by $U_j$. Note that the bound scales linearly with respect to the number of local potentials $J$. 
As our experiments illustrate (see~\autoref{sec:experiments}), it is easy to construct target energies that are much better approximated in~$\mathcal F_1$ than in~$\mathcal F_2$.
Indeed, we find that the test error tends to decrease more quickly as a function of the sample size when training both layers of shallow networks rather than just the second layer, which corresponds to controlling the~$\mathcal F_1$ norm.

\subsection{Guarantees for Stein Discrepancy EBMs} \label{subsec:guarantees_sd}
We now consider EBM estimators obtained by minimizing Stein discrepancies, and establish bounds on the Stein discrepancies between the target measure and the estimated one.
As in~\autoref{sub:mle_bounds}, we begin by providing error decompositions in terms of estimation and approximation error.
The following theorem applies to the Stein discrepancy estimator when the set of test functions~$\mathcal{H}$ is the unit ball of the space of $\mathcal F^{d+1}$ in a mixed $\mathcal F/\ell_2$ norm, with $\mathcal F = \mathcal{F}_1$ or $\mathcal F_2$.
For $\mathcal F_1$, we will denote this particular setting as $\mathcal{F}_1$-Stein discrepancy, or $\mathcal{F}_1$-SD. Although $\mathcal{F}_1$-SD has not been studied before to our knowledge,
the empirical work of~\citet{grathwohl2020learning} does use Stein discrepancies with neural network test functions, which provides practical motivation for considering such a metric.

\begin{restatable}{thm}{coronesd}
\label{thm:coronesd}
Let $K = \mathbb{S}^d$. Assume that the class $\mathcal{F}$ is such that $\sup_{f \in \mathcal{F}} \{ \|\nabla_i f\|_{\infty} | 1 \leq i \leq d+1 \} \leq \beta C_1$. If $\mathcal{H} = \mathcal{B}_{\mathcal{F}_1^{d+1}} = \{ h = {(h_i)}_{i=1}^{d+1} \ | \ h_i \in \mathcal{F}_1, \sum_{i=1}^{d+1} \|h_i\|_{\mathcal{F}_1}^2 \leq 1 \}$ or $\mathcal{H} = \mathcal{B}_{\mathcal{F}_2^{d+1}} = \{ h = {(h_i)}_{i=1}^{d+1} \ | \ h_i \in \mathcal{F}_2, \sum_{i=1}^{d+1} \|h_i\|_{\mathcal{F}_2}^2 \leq 1 \}$, we have that for the estimator $\hat{\nu}$ defined in \eqref{eq:sd_estimator}, with probability at least $1-\delta$,
\begin{align}
\begin{split}
    \text{SD}_{\mathcal{H}}&(\nu, \hat{\nu}) \leq \frac{4 \sqrt{d+1} (\beta C_1 + C_2 \sqrt{d+1} + d)}{\sqrt{n}} \\ &+  2(\beta C_1 + d + 1)\sqrt{\frac{(d+1)\log(\frac{d+1}{\delta})}{2n}} \\ &+ \inf_{f \in \mathcal{F}} \mathbb{E}_{\nu} \bigg[ \bigg\|  - \nabla f(x) - \nabla \log \left(\frac{d\nu}{d\tau}(x) \right) \bigg\|_2 \bigg]
\end{split}
\end{align}
where $C_2$ is a universal constant and $\nabla f$ denotes the Riemannian gradient of $f$.
\end{restatable}

Notice that unlike in \autoref{thm:generalbound}, the statistical error terms in \autoref{thm:coronesd} depend on the ambient dimension $d$.
While we do not show that this dependence is necessary, studying this question would be an interesting future direction.
Remark as well the similarity of the approximation term with the term $2\inf_{f \in \mathcal{F}} \|g-f\|_{\infty}$ from equation \eqref{eq:thm1_2}, albeit in this case it involves the $L^{\infty}$ norm of the gradients. Furthermore, note that the only assumption on the set $\mathcal{F}$ is a uniform $L^{\infty}$ bound on $\mathcal{F}_1$, while \autoref{thm:generalbound} also requires a more restrictive Rademacher complexity bound on~$\mathcal F$. This illustrates the fact that the Stein discrepancy is a weaker metric than the KL divergence.

In \autoref{thm:ksdsecond} we give an analogous result for the unbiased KSD estimator~\eqref{eq:ksd_estimator2}, under the following reasonable assumptions on the kernel $k$, which follow~\citep{liu2016akernelized}.

\begin{assumption} \label{ass:kernel}
The kernel~$k$ has continuous second order partial derivatives, and it satisfies that for any non-zero function $g \in L^2(\mathbb{S}^d)$, $\int_{\mathbb{S}^d} \int_{\mathbb{S}^d} g(x) k(x,x') g(x') d\tau(x) d\tau(x') > 0$, and that $\sup_{x, x' \in \mathbb{S}^d} k(x,x') \leq C_2$, $\sup_{x, x' \in \mathbb{S}^d} \|\nabla_x k(x,x')\|_2 \leq C_3$. 
\end{assumption}

\begin{restatable}{thm}{ksdsecond} \label{thm:ksdsecond}
Let $K=\mathbb{S}^d$. Assume that the class $\mathcal{F}$ is such that $\sup_{f \in \mathcal{F}} \{ \|\nabla f\|_{\infty} \} \leq \beta C_1$. Let $\mathrm{KSD}$ be the kernelized Stein discrepancy for a kernel that satisfies \autoref{ass:kernel}. If we take $n$ samples $\{x_i\}_{i=1}^n$ of a target measure $\nu$ with almost everywhere differentiable log-density, and consider the unbiased KSD estimator \eqref{eq:ksd_estimator2}, we have with probability at least $1-\delta$,
\begin{align} 
\mathrm{KSD}&(\nu, \hat{\nu}) \leq \frac{2}{\sqrt{\delta n}} ((\beta C_1 + d)^2 C_2 + 2C_3(\beta C_1 + d)) \\ &+ C_2 \inf_{f \in \mathcal{F}} \mathbb{E}_{x \sim \nu} \left[\bigg\|\nabla \log \left(\frac{d\nu}{d\tau}(x) \right) - \nabla f(x) \bigg\|^2 \right]~.
\end{align}
\end{restatable}

The statistical error term in \autoref{thm:ksdsecond} is obtained using the expression of the variance of the estimator \eqref{eq:ksd_estimator2} \citep{liu2016akernelized}.
Note that \autoref{ass:kernel} is fulfilled, for example, for the radial basis function (RBF) kernel $k(x,x') = \exp(-\|x-x'\|^2/(2\sigma^2))$ with $C_2 = 1$, $C_3 = 1/\sigma^2$. 

Making use of \autoref{thm:coronesd} (for $\mathcal{F}_1$-SD) and \autoref{thm:ksdsecond} (for KSD), in \autoref{cor:corsdsphere} we obtain adaptivity results for target measures with low-dimensional structures similar to \autoref{cor:statf1}, also for $\mathcal{F} = \mathcal{B}_{\mathcal{F}_1}(\beta)$. The class of target measures that we consider are those satisfying \autoref{ass:stein}, which is similar to \autoref{ass:mle_assumption} but for $K = \mathbb{S}^d$ and with an additional Lipschitz condition on the gradient of $\nabla \phi_j$.  

\begin{assumption} \label{ass:stein}
Let $K = \mathbb{S}^d$. Suppose that the target probability measure $\nu$ is absolutely continuous w.r.t. the Hausdorff measure over $\mathbb{S}^d$ and it satisfies $\forall x \in \mathbb{S}^d, \ \frac{d\nu}{d\tau}(x) = \exp(-\sum_{j=1}^J \phi_j(U_j x))/\int_{K_0} \exp(-\sum_{j=1}^J \phi_j(U_j y)) d\tau(y)$, where $\phi_j$ are 1-homogeneous differentiable functions on the unit ball of $\R^k$ such that $\|\phi_j\|_{\infty} \leq \eta$, $\sup_{x \in \mathbb{S}^d} \|\nabla \phi_j(x)\|_{2} \leq \eta$ and $\nabla \phi_j$ is $L$-Lipschitz continuous, and $U_j \in \R^{k \times d}$ with orthonormal rows.
\end{assumption}

\begin{restatable}{cor}{corsdsphere}
\label{cor:corsdsphere}
Let $\mathcal{F} = B_{\mathcal{F}_1}(\beta)$. Let \autoref{ass:stein} hold.
(i) When $\hat{\nu}$ is the $\mathcal{F}_1$-SD estimator \eqref{eq:sd_def} and the assumptions of \autoref{thm:coronesd} hold, we can choose the inverse temperature $\beta > 0$ such that with probability at least $1-\delta$ we have that $SD_{\mathcal{B}_{\mathcal{F}_1}^{d+1}}(\nu, \hat{\nu})$ is upper-bounded by
\begin{equation} \label{eq:cor_f1_sd}
\tilde{O} \left( \left(1+\sqrt{\log(1/\delta)} \right) J (L + \eta) (\eta J)^{\frac{2}{k+1}} d^{\frac{1}{k+3}} n^{-\frac{1}{k+3}} \right)
\end{equation}
where the notation $\tilde{O}$ indicates that we overlook logarithmic factors and constants depending only on the dimension.
(ii) When $\hat{\nu}$ is the unbiased KSD estimator \eqref{eq:ksd_estimator2} and the assumptions of \autoref{thm:ksdsecond} hold, $\beta > 0$ can be chosen so that with probability at least $1-\delta$ we have that $\mathrm{KSD}(\nu, \hat{\nu})$ is upper-bounded by
\begin{align} \label{eq:cor_ksd}
\tilde{O} \left( \delta^{-\frac{1}{k+3}} 
\left( J (L + \eta)\right)^{\frac{2(k+1)}{k+3}} (\eta J)^{\frac{4}{k+3}}  n^{-\frac{1}{k+3}} \right).
\end{align}
\end{restatable}

Noticeably, the rates in \autoref{cor:corsdsphere} are also of the form $\mathcal{O}(n^{-\frac{1}{k+3}})$, which means that just as in \autoref{cor:statf1}, the low-dimensional structure in the target measure helps in breaking the curse of dimensionality. 

\paragraph{Proof sketch.} The main challenge in the proof of \autoref{cor:corsdsphere} is to bound the approximation terms in \autoref{thm:coronesd} and \autoref{thm:ksdsecond}. To do so, we rely on \autoref{lem:approx_derivative2} in \autoref{sec:proofs_statistical}, which shows the existence of $\hat{g}$ in a ball of $\mathcal{F}_2$ such that $\sup_{x \in \mathbb{S}^d} \|\nabla \hat{g}(x) - \nabla g(x)\|_2$ has a certain bound when $g$ is bounded and has bounded and Lipschitz gradient. \autoref{lem:approx_derivative2} might be of independent interest: in particular, it can be used to obtain a similar adaptivity result for score-matching EBMs, which optimize the Fisher divergence $\mathbb{E}_{x \sim \nu} [\|\nabla \log (\frac{d\nu}{dp}(x)) - \nabla f(x) \|^2]$.

%% file: algorithms.tex



This section provides a description of the optimization algorithms used for learning $\mathcal F_{1/2}$-EBMs using the estimators studied in~\autoref{sec:statistical_guarantees}, namely maximum likelihood, KSD, and~$\mathcal F_1$-SD.

\subsection{Algorithms for $\mathcal{F}_1$ EBMs}
\label{subsec:algorithms_f1}
We provide the algorithms for the three models using a common framework. We define the function $\Phi : \R \times \R^{d+1} \rightarrow \mathcal{F}_1$ as $\Phi(w,\theta)(x) = w \sigma(\langle \theta, x \rangle)$. Given a convex loss $R : \mathcal{F}_1 \rightarrow \R$, we consider the problem 
\begin{align} 
\begin{split} \label{eq:penalized_f1_R}
    &\inf_{\mu \in \mathcal{P}(\R^{d+2})} F(\mu), \\
    &F(\mu) := R\left(\int \Phi(w,\theta) d\mu \right) + \lambda \int (|w|^2 + \|\theta\|_2^2) d\mu.
\end{split}
\end{align}
for some $\lambda > 0$. It is known \citep[e.g.,][]{neyshabur2015insearch} that, since $|w|^2 + \|\theta\|_2^2 \geq 2 |w|\|\theta\|_2$ with equality when moduli are equal, this problem is equivalent to
\begin{align} \label{eq:penalized_f1}
    \inf_{\mu \in \mathcal{P}(\R \times \mathbb{S}^{d})} R\left(\int \Phi(w,\theta) d\mu \right) + \lambda \int_{\R \times \mathbb{S}^{d}} |w| d\mu.
\end{align}
And by the definition of the $\mathcal{F}_1$ norm, this is equivalent to $\inf_{f \in \mathcal{F}_1} R\left(f\right) + \lambda \|f\|_{\mathcal F_1}$, which is the penalized form of $\inf_{f \in \mathcal{B}_{\mathcal{F}_1}(\beta)} R\left(f\right)$ for some $\beta > 0$. Our $\mathcal{F}_1$ EBM algorithms solve problems of the form \eqref{eq:penalized_f1_R} for different choices of $R$, or equivalently, minimize the functional $R$ over an $\mathcal{F}_1$ ball. The functional $R$ takes the following forms for the three models considered: 

\begin{enumerate}[label=(\roman*),itemsep=0pt,topsep=1pt,leftmargin=0.4cm]
\item Cross-entropy: We have that $R(f) = \frac{1}{n} \sum_{i=1}^n f(x_i) + \log\left(\int_{K} e^{-f(x)} d\tau(x) \right)$, which is convex (and differentiable) because the free energy obeys such properties~\citep[e.g., by adapting][Prop 3.1 to the infinite-dimensional case]{wainwright2008graphical}. 
\item Stein discrepancy: the estimator \eqref{eq:ksd_estimator} corresponds to $R(f) = \sup_{h \in \mathcal{H}} \mathbb{E}_{\nu_n} [\sum_{j=1}^{d+1} -(\nabla_j f(x) + d x_j) h_j(x) + \nabla_j h_j(x)]$, which is convex as the supremum of convex (linear) functions. 
\item Kernelized Stein discrepancy: we have $R(f) = \frac{1}{n^2} \sum_{i,j=1}^n \tilde{u}_{\nu_{f}}(x_i,x_j)$, which is convex (in fact, it is quadratic in $\nabla f$).
\end{enumerate}

In order to optimize \eqref{eq:penalized_f1_R}, we discretize measures in $\mathcal{P}(\R^{d+2})$ as averages of point masses $\frac{1}{m} \sum_{i=1}^{m} \delta_{(w^{(i)}, \theta^{(i)})}$, each point mass corresponding to one neuron. Furthermore, we define the function $G : (\R^{d+2})^m \rightarrow \R$ as
\begin{align} \label{eq:G_definition}
    &G((w^{(i)}, \theta^{(i)})_{i=1}^{m}) := F\left(\frac{1}{m} \sum_{i=1}^{m} \delta_{(w^{(i)}, \theta^{(i)})} \right) 
    \\ &= R\left(\frac{1}{m} \sum_{i=1}^{m} \Phi(w^{(i)}, \theta^{(i)}) \right) + \frac{\lambda}{m} \sum_{i=1}^{m} (|w^{(i)}|^2 + \|\theta^{(i)}\|_2^2).
\end{align}
Then, as outlined in Algorithm \ref{alg:F1_ebm}, we use gradient descent on $G$ to optimize the parameters of the neurons, albeit possibly with noisy estimates of the gradients.
\begin{algorithm}
\caption{Generic algorithm to train $\mathcal{F}_1$ EBMs}
\label{alg:F1_ebm}
\begin{algorithmic}
\INPUT $m$, stepsize $s$
\STATE Get $m$ i.i.d. samples $(w_t^{(i)}, \theta_t^{(i)})$ from $\mu_0 \in \mathcal{P}(\R^{d+2})$.
\FOR {$t=0,\dots,T-1$} 
        \FOR{$i=1,\dots,m$}
            \STATE Compute estimates $\hat{\nabla}_{w^{(i)}}G((w_t^{(i)}, \theta_t^{(i)})_{i=1}^{m})$ and $\hat{\nabla}_{\theta^{(i)}}G((w_t^{(i)}, \theta_t^{(i)})_{i=1}^{m})$.
            \STATE $w_{t+1}^{(i)} \leftarrow w_{t}^{(i)} - s \hat{\nabla}_{w^{(i)}}G((w_t^{(i)}, \theta_t^{(i)})_{i=1}^{m})$
            \STATE $\theta_{t+1}^{(i)} \leftarrow \theta_{t}^{(i)} - s \hat{\nabla}_{\theta^{(i)}}G((w_t^{(i)}, \theta_t^{(i)})_{i=1}^{m})$
        \ENDFOR
\ENDFOR
\OUTPUT Energy $\frac{1}{m} \sum_{i=1}^{m} \Phi(w_T^{(i)}, \theta_T^{(i)}) \in \mathcal{F}_1$.
\end{algorithmic}
\end{algorithm}

Computing an estimate the gradient of $G$ involves computing the gradient of $R\left(\frac{1}{m} \sum_{i=1}^{m} \Phi(w^{(i)}, \theta^{(i)}) \right)$. Denoting by $z_{i} = (w^{(i)}, \theta^{(i)}), \mathbf{z} = (z_i)_{i=1}^m$ and by $\nu_{\mathbf{z}}$ the Gibbs measure corresponding to the energy $f_{\mathbf{z}} := \frac{1}{m} \sum_{i=1}^{m} \Phi(w^{(i)}, \theta^{(i)})$, we have
\begin{enumerate}[label=(\roman*),itemsep=0pt,topsep=1pt,leftmargin=0.5cm]
    \item Cross-entropy: The gradient of $R(f_{\mathbf{z}})$ with respect to $z_i$ takes the expression
    $\mathbb{E}_{\nu_n} \nabla_{z_i} \Phi(z_i)(x) - \mathbb{E}_{\nu_{\mathbf{z}}} \nabla_{z_i} \Phi(z_i)(x)$. The expectation under $\nu_{\mathbf{z}}$ is estimated using MCMC samples of the EBM. Thus, the quality of gradient estimation depends on the performance of the MCMC method of choice, which can suffer for non-convex energies and low temperatures.
    \item $\mathcal{F}_1$ Stein discrepancy: The (sub)gradient of $R(f_{\mathbf{z}})$ w.r.t. $z_i$ equals 
    $\mathbb{E}_{\nu_n}[ -\beta \sum_{j=1}^{d+1} \nabla_{z_i} \nabla_x(\Phi(z_i)(x)) h^{\star}_j(x)]$, in which $h^{\star}_j$ are respectively maximizers of $-(\beta \nabla_j f(x) + d x_j) h_j(x) + \nabla_j h_j(x)$ over $\mathcal{B}_{\mathcal{F}_1}$. The gradient estimation involves $d+1$ optimization procedures over balls of $\mathcal{F}_1$ to compute $h^{\star}_j$, which we solve using Algorithm~\ref{alg:F1_ebm}. Thus, the algorithm operates on two timescales.  
    \item Kernelized Stein discrepancy: Using \eqref{eq:KSD_def}, the gradient of
    $R(f_{\mathbf{z}})$ with respect to $z_i$ takes the expression $\mathbb{E}_{x,x' \sim \nu_n} [\nabla_{z_i} u_{\nu_{\mathbf{z}}}(x,x')]$, which can be developed into closed form. The only issue is the quadratic dependence on the number of samples.
\end{enumerate}

\subsection{Algorithms for $\mathcal{F}_2$ EBMs}
Considering convex losses $R : \mathcal{F}_1 \rightarrow \R$ as in \autoref{subsec:algorithms_f1}, the penalized form of the problem $\inf_{f \in \mathcal{B}_{\mathcal{F}_2}(\beta)} R\left(f\right)$ is 
\begin{align}
    \inf_{\|h\|_{2} \leq 1} R\left(\int_{\mathbb{S}^d} \sigma(\langle \theta, \cdot \rangle) h(\theta) d\tau(\theta) \right) + \lambda \int_{\mathbb{S}^d} h^2(\theta) d\tau(\theta).
\end{align}
To optimize this, we discretize the problem: we take $m$ samples $(\theta^{(i)})_{i=1}^{m}$ of the uniform measure $\tau$ that we keep fixed, and then solve the random features problem
\begin{align} \label{eq:F2_problem}
    \inf_{\substack{w \in \R^{m} \\ \|w\|_{2} \leq 1}} R\left(\frac{1}{m} \sum_{i=1}^{m} w^{(i)} \sigma(\langle \theta^{(i)}, \cdot \rangle) \right) + \frac{\lambda}{m} \sum_{i=1}^{m} |w^{(i)}|^2.
\end{align}
Remark that this objective function is equivalent to the objective function $G((w^{(i)},\theta^{(i)})_{i=1}^{m})$ in equation \eqref{eq:G_definition} when $(\theta^{(i)})_{i=1}^{m}$ are kept fixed. Thus, we can solve \eqref{eq:F2_problem} by running Algorithm~\ref{alg:F1_ebm} without performing gradient descent updates on $(\theta^{(i)})_{i=1}^{m}$. That is, while for the $\mathcal{F}_1$ EBM training both the features and the weights are learned via gradient descent, for $\mathcal{F}_2$ only the weights are learned.

\subsection{Qualitative convergence results}
\label{sec:qualitativeconv}
The overparametrized regime corresponds to taking a large number of neurons $m$. In the limit $m \rightarrow \infty$, under appropriate assumptions the empirical measure dynamics corresponding to the gradient flow of $G((w^{(i)}, \theta^{(i)})_{i=1}^{m})$ converge weakly to the mean-field dynamics \cite{mei2018mean,chizat2018global,rotskoff2018neural}.
Leveraging a result from \citet{chizat2018global} we argue informally that in the limit $m \rightarrow \infty, t \rightarrow \infty$, with continuous time and exact gradients, the gradient flow of $G$ converges to the global optimum of $F$ over $\mathcal{P}(\R^{d+2})$ (see more details in \autoref{sec:convergence}). 

In contrast with this positive qualitative result, we should mention a computational aspect that distinguishes these algorithms from their supervised learning counterparts: the Gibbs sampling required to estimate the gradient at each timestep. A notorious challenge is that for generic energies (even generic energies in $\mathcal{F}_1$), either the mixing time of MCMC algorithms is cursed by dimension \cite{bakry2014analysis} or the acceptance rate is exponentially small. The analysis of the extra assumptions on the target energy and initial conditions that 
would avoid such curse are beyond the scope of this work, but a framework based on thermodynamic integration and replica exchange \citep{PhysRevLett.57.2607} would be a possible route forward. 

%% file: experiments_neutral.tex
\begin{figure*}
    \centering
    \includegraphics[width=.3\textwidth]{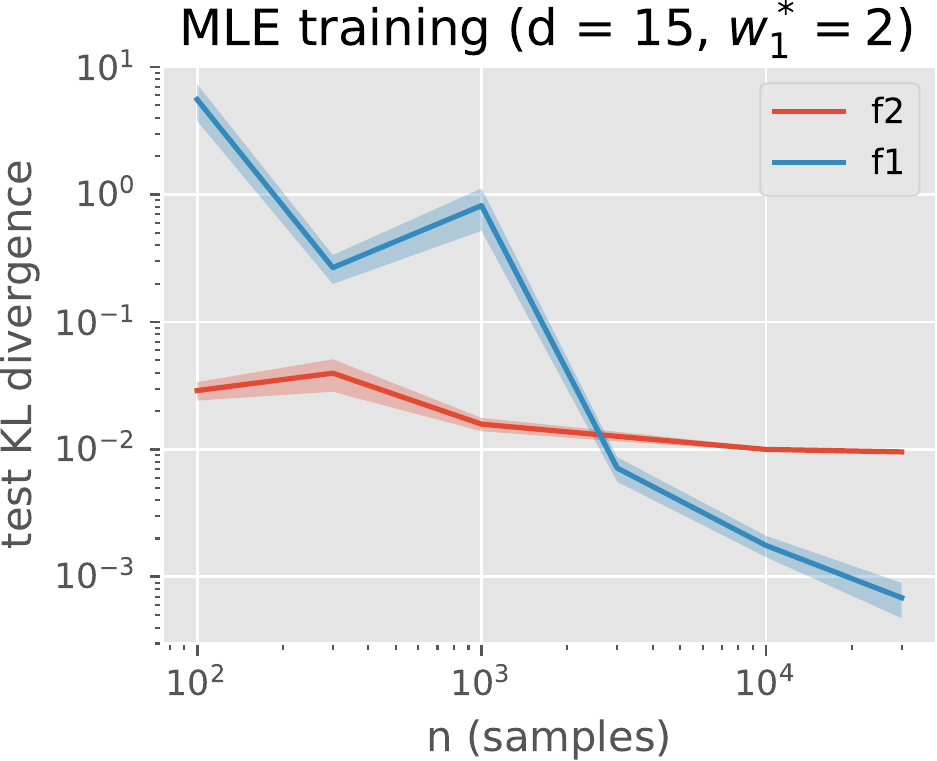}
    \includegraphics[width=.3\textwidth]{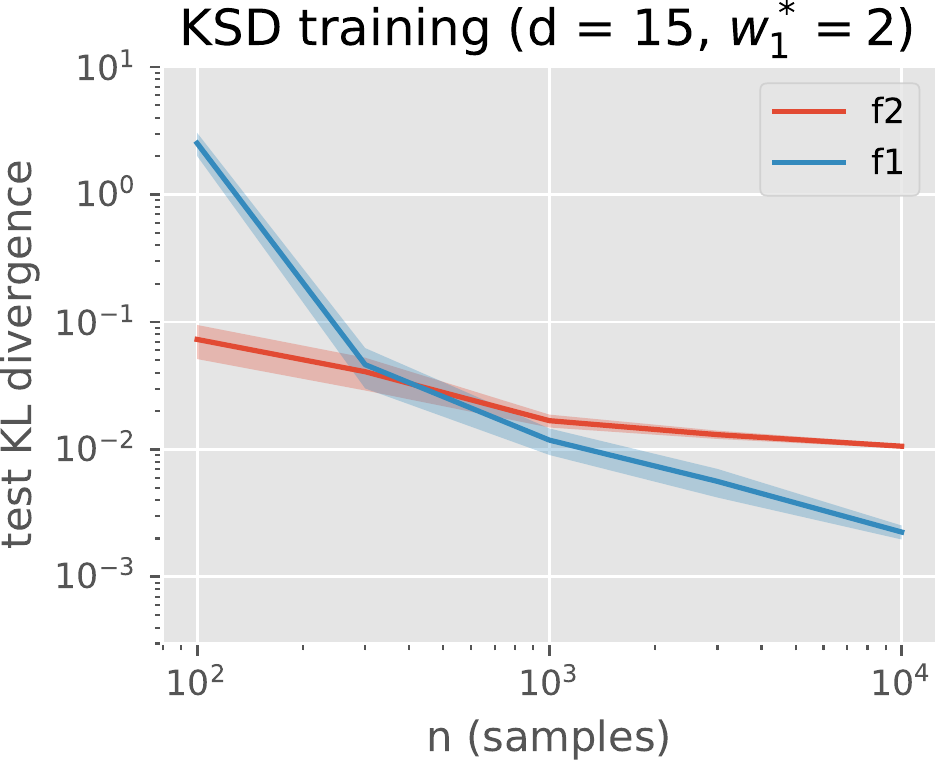}
    \includegraphics[width=.3\textwidth]{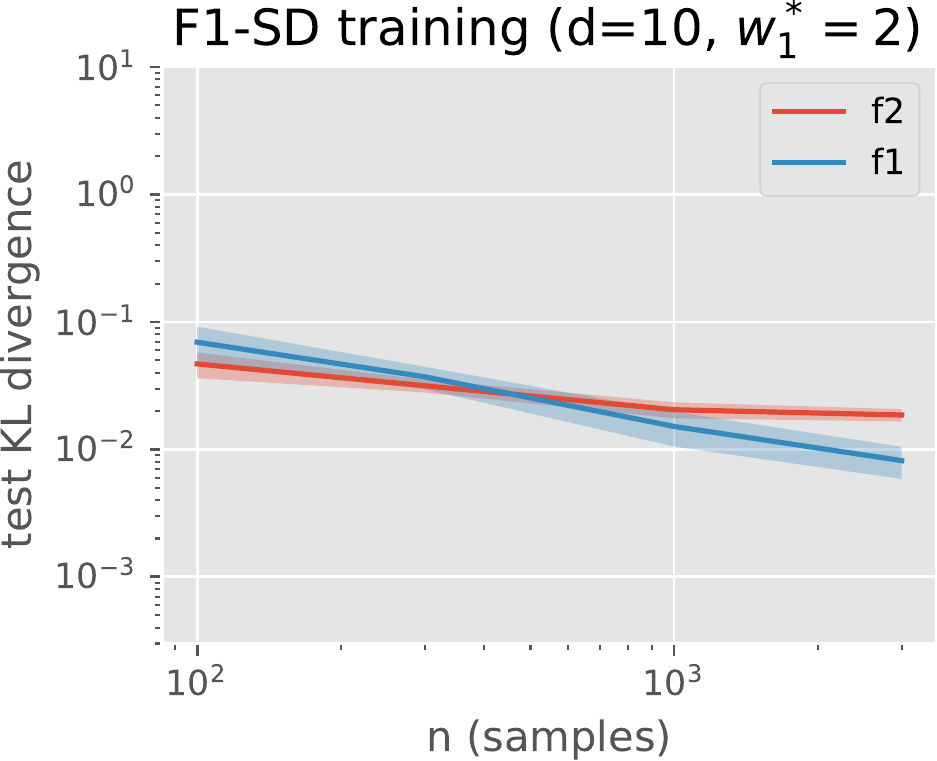} \\
    \includegraphics[width=.3\textwidth]{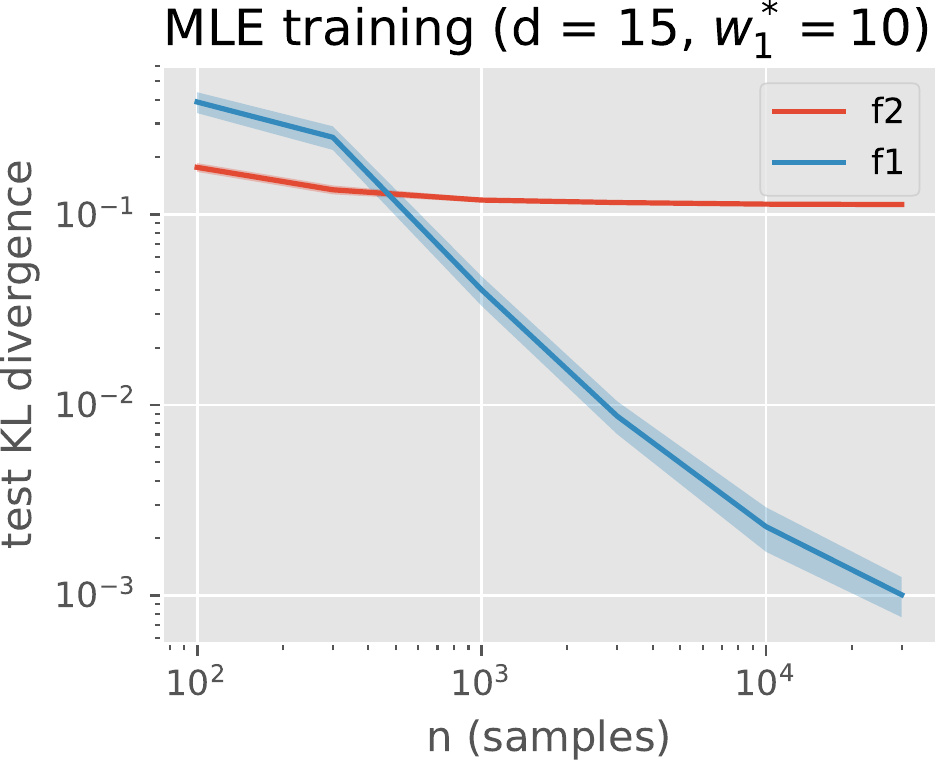}
    \includegraphics[width=.3\textwidth]{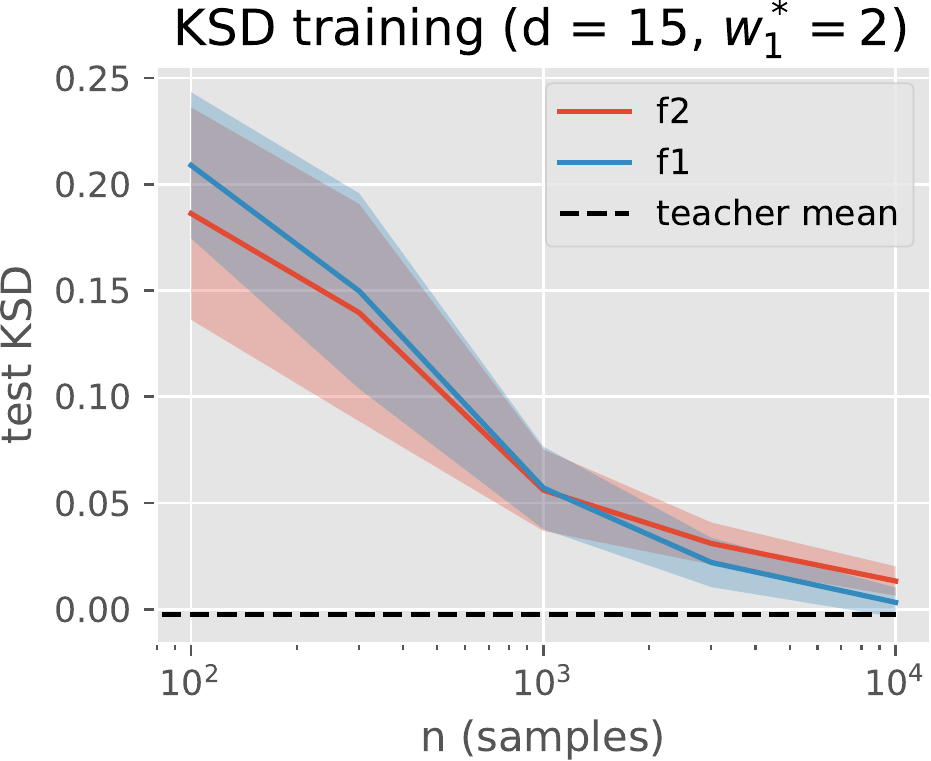}
    \includegraphics[width=.3\textwidth]{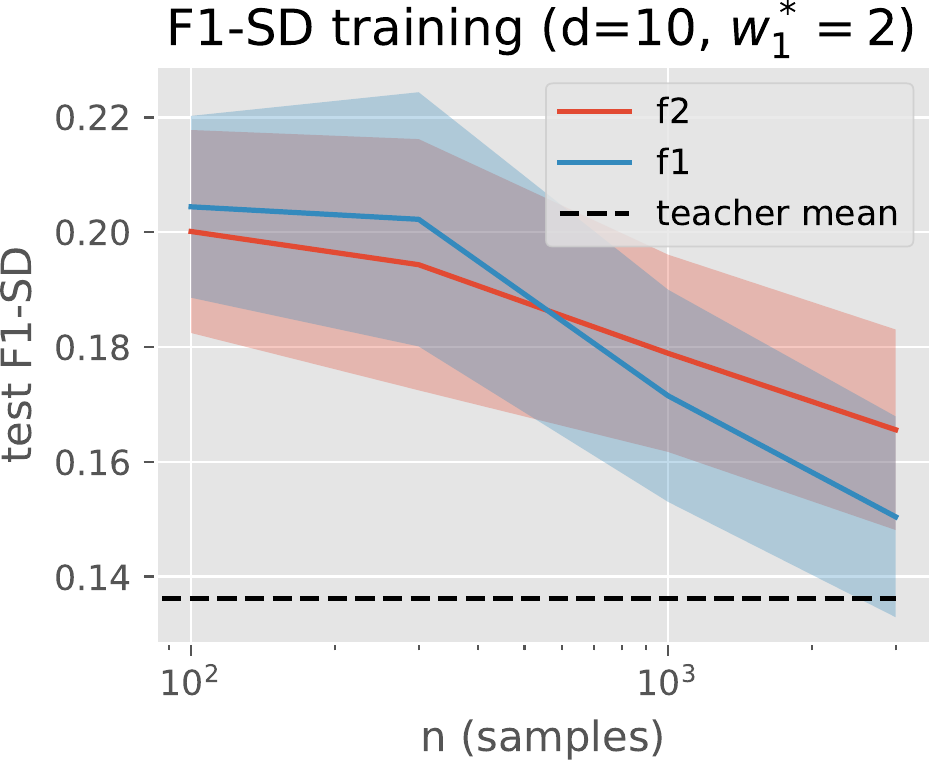}
    \caption{Test metrics obtained for MLE, KSD and $\mathcal{F}_1$-SD training on a one-neuron teacher with positive output weight. (top) Test performance measured with KL divergence estimates for~$w^*_1 = 2$. (bottom left) MLE on a teacher network with larger weight~$w^*_1 = 10$. (bottom center/right) Test KSD and $\mathcal F_1$-SD for models trained with the same metric with $w^*_1 = 2$. For reference, the black discontinuous lines show the teacher KSD and $\mathcal{F}_1$-SD of the teacher model w.r.t. 5000 and 2000 test samples, respectively. Confidence estimates are over 10 different data samplings.}
    \label{fig:one_positive_neuron}
\end{figure*}

\begin{figure*}
    \centering
    \includegraphics[width=.3\textwidth]{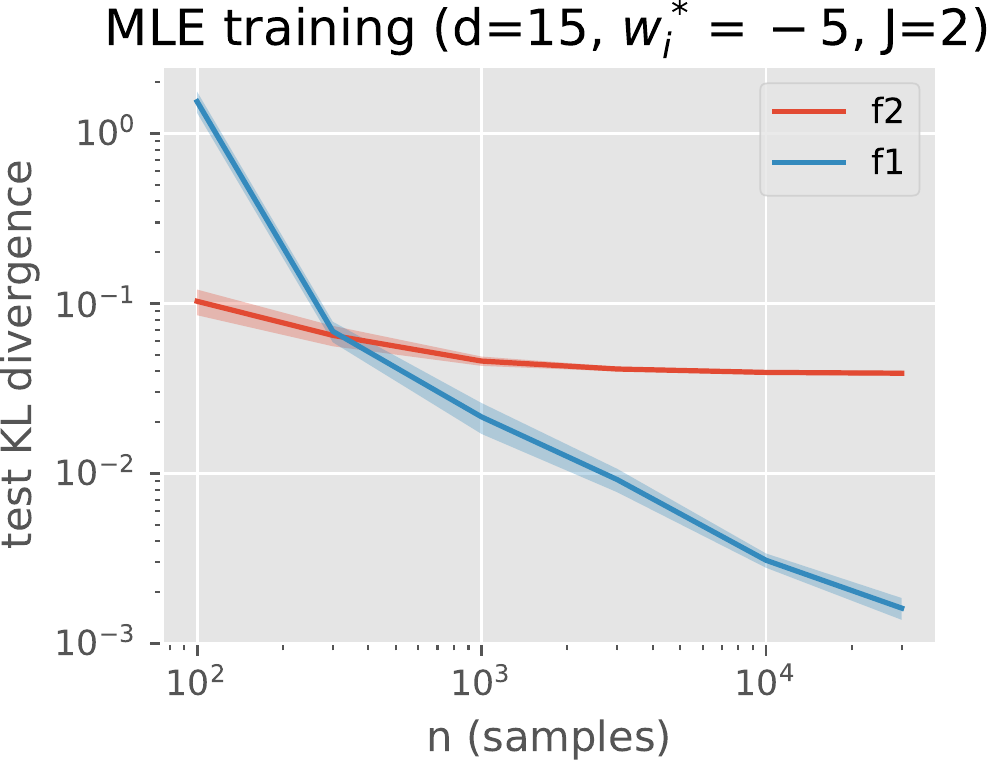}
    \includegraphics[width=.3\textwidth]{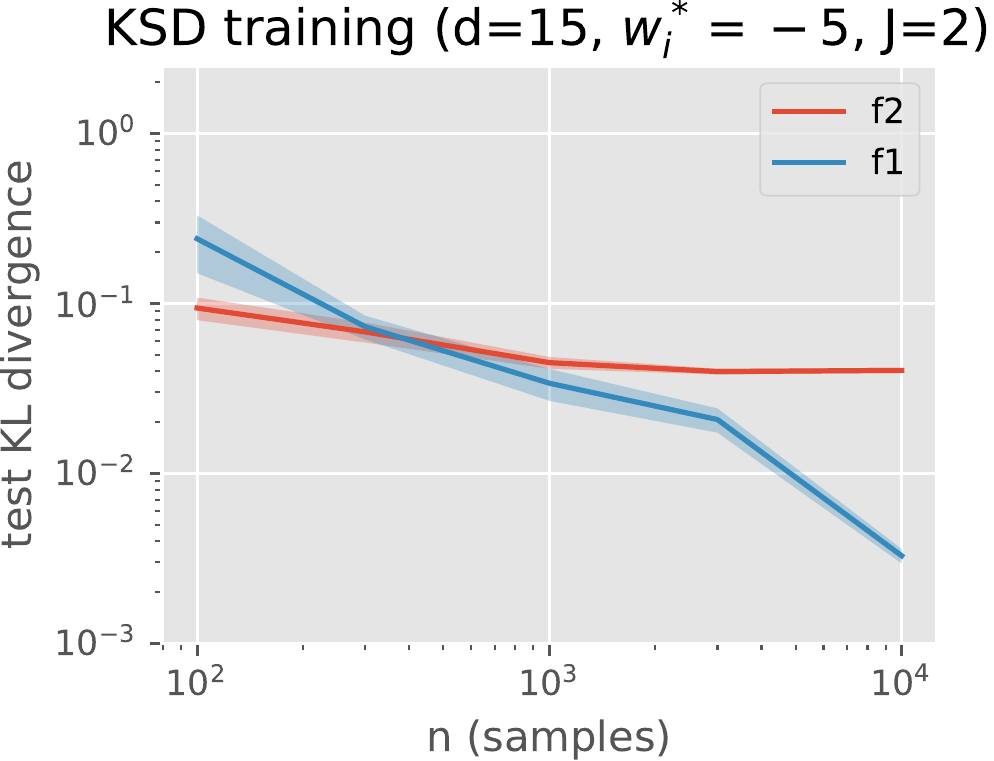}
    \includegraphics[width=.3\textwidth]{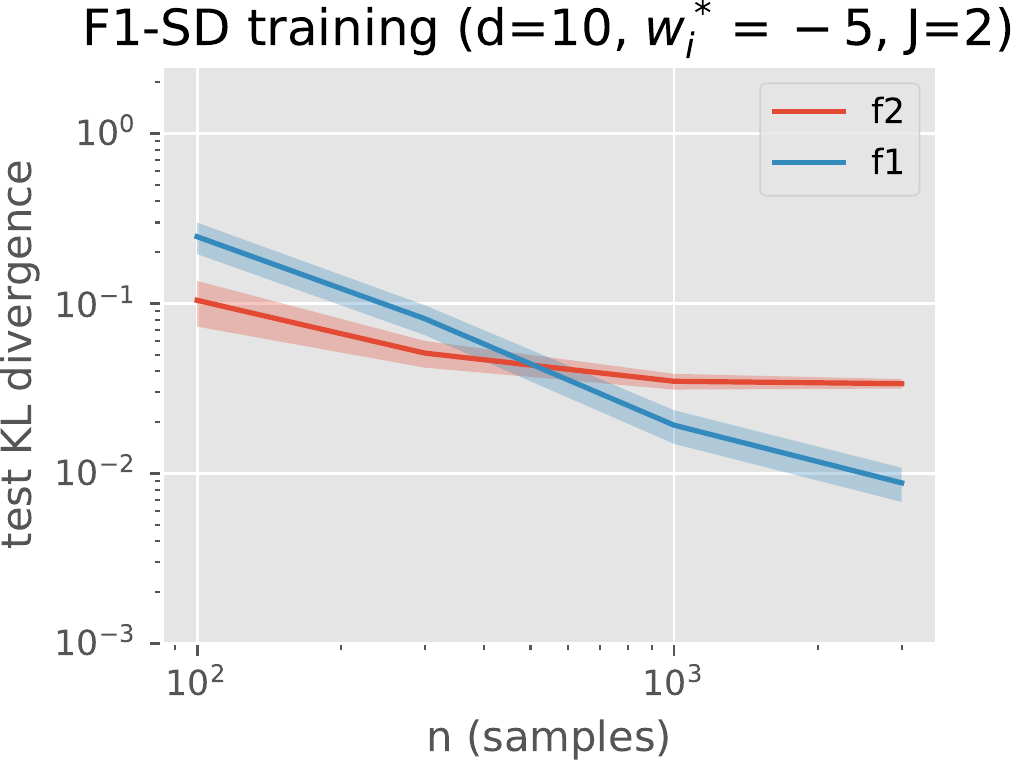} \\
    \includegraphics[width=.3\textwidth]{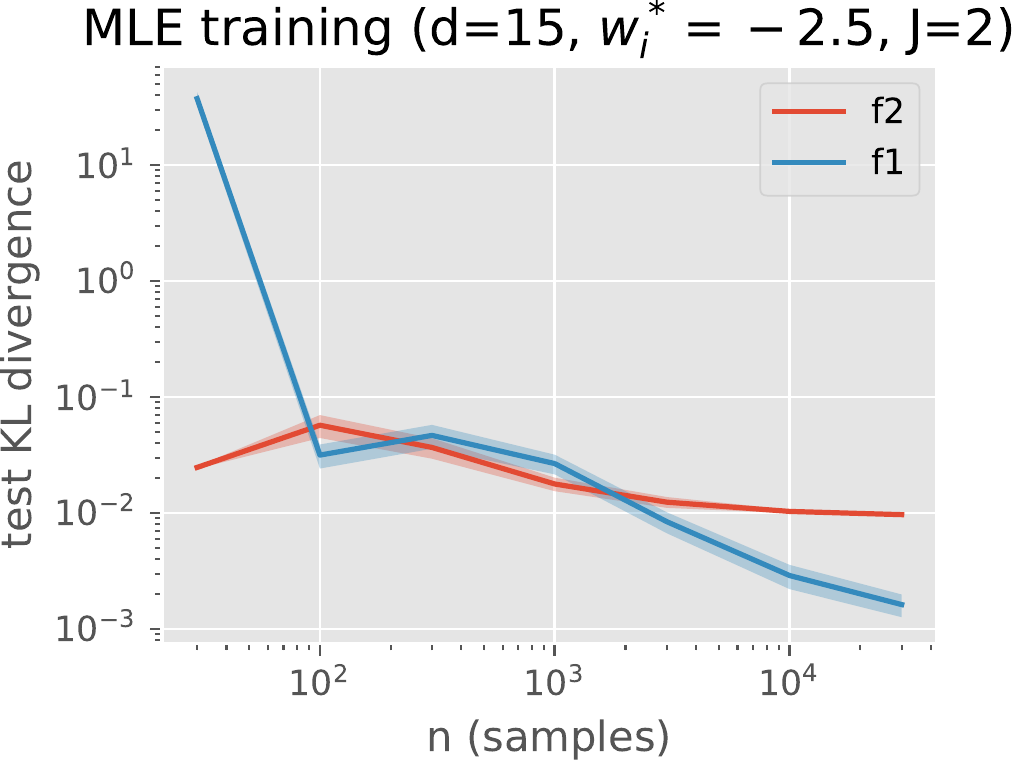}
    \includegraphics[width=.3\textwidth]{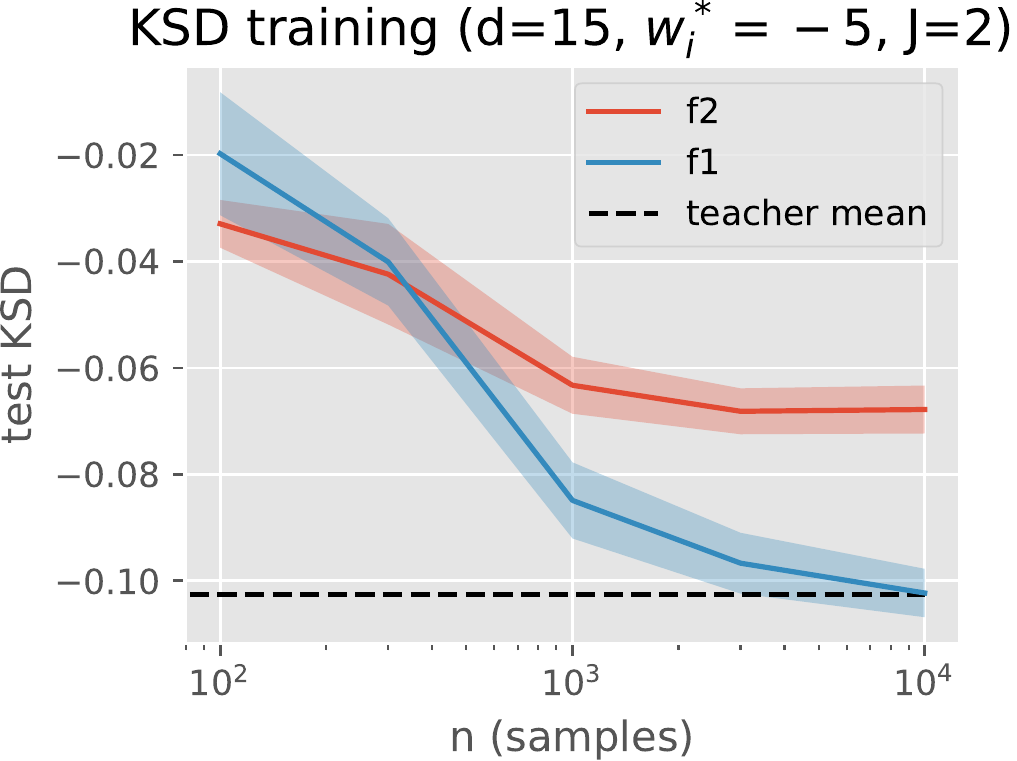}
    \includegraphics[width=.3\textwidth]{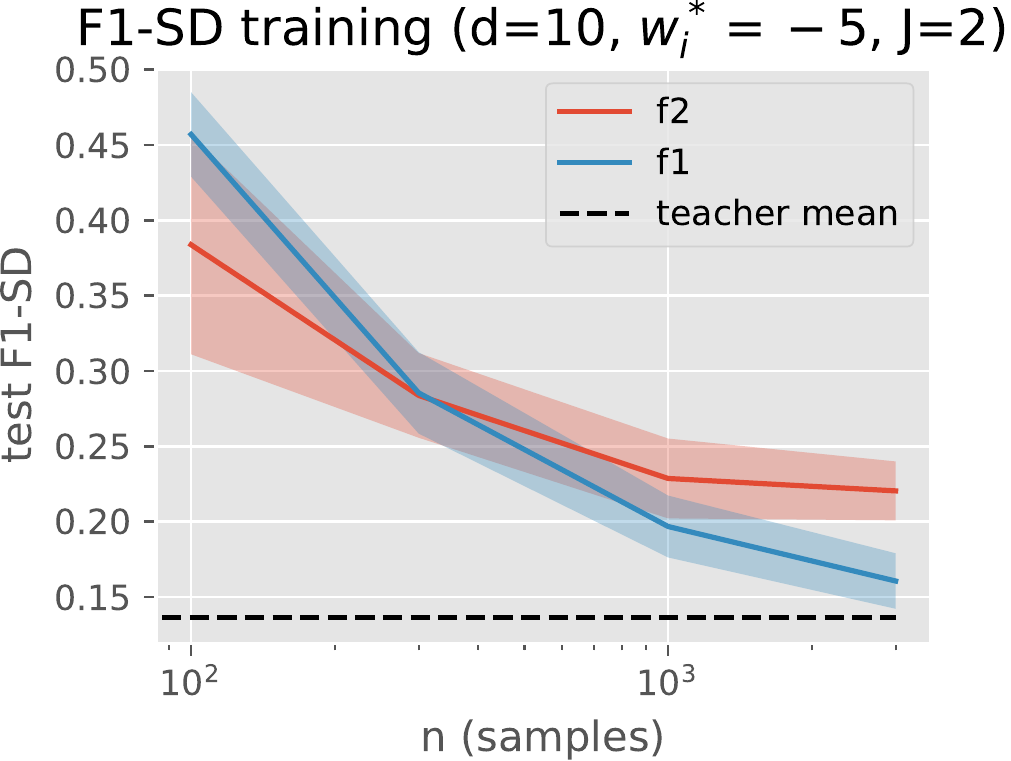}
    \caption{Test metrics obtained for MLE, KSD and $\mathcal{F}_1$-SD training on a two-neuron teacher with negative output weights. (top) Test performance measured with cross-entropy estimates with~$w^*_1, w^*_2 = -5$. (bottom left) MLE on a teacher network with smaller weights~$w^*_1, w^*_2 = -2.5$. (bottom center/right) Test KSD and $\mathcal F_1$-SD for models trained with the same metric, for~$w^*_1, w^*_2 = -5$. For reference, the black discontinuous lines show the teacher KSD and $\mathcal{F}_1$-SD of the teacher model w.r.t. 5000 and 2000 test samples, respectively. Confidence estimates are over 10 different data samplings.}
    \label{fig:two_negative_neurons}
\end{figure*}

\begin{figure*}
    \centering
    \includegraphics[width=.3\textwidth]{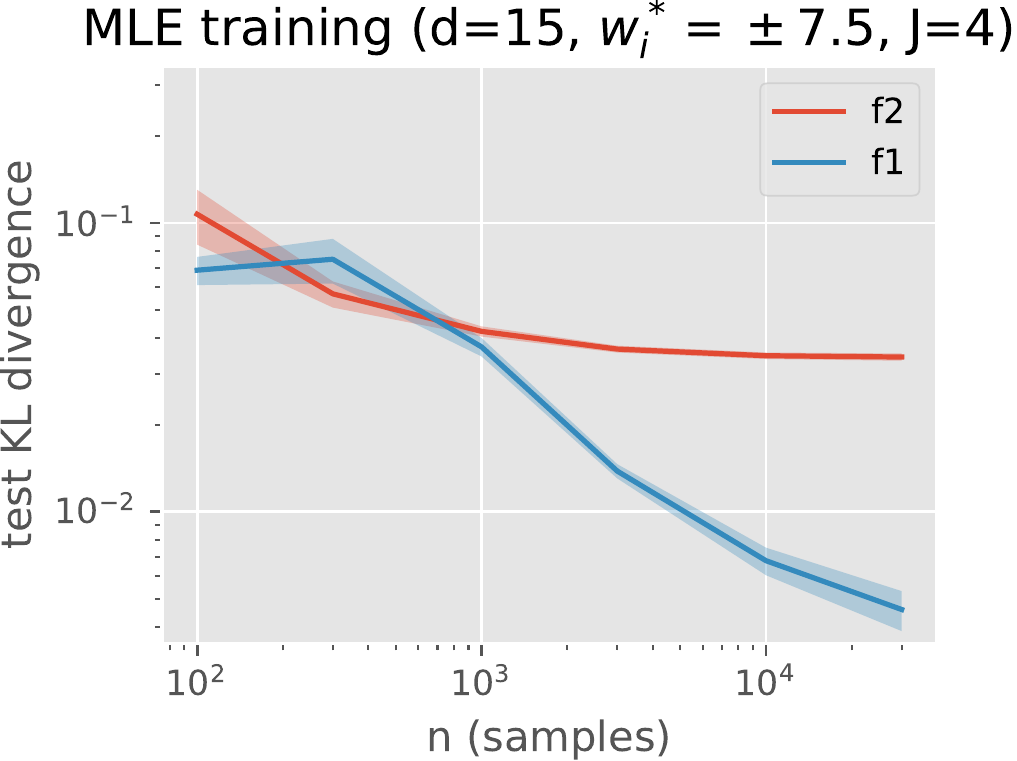}
    \includegraphics[width=.3\textwidth]{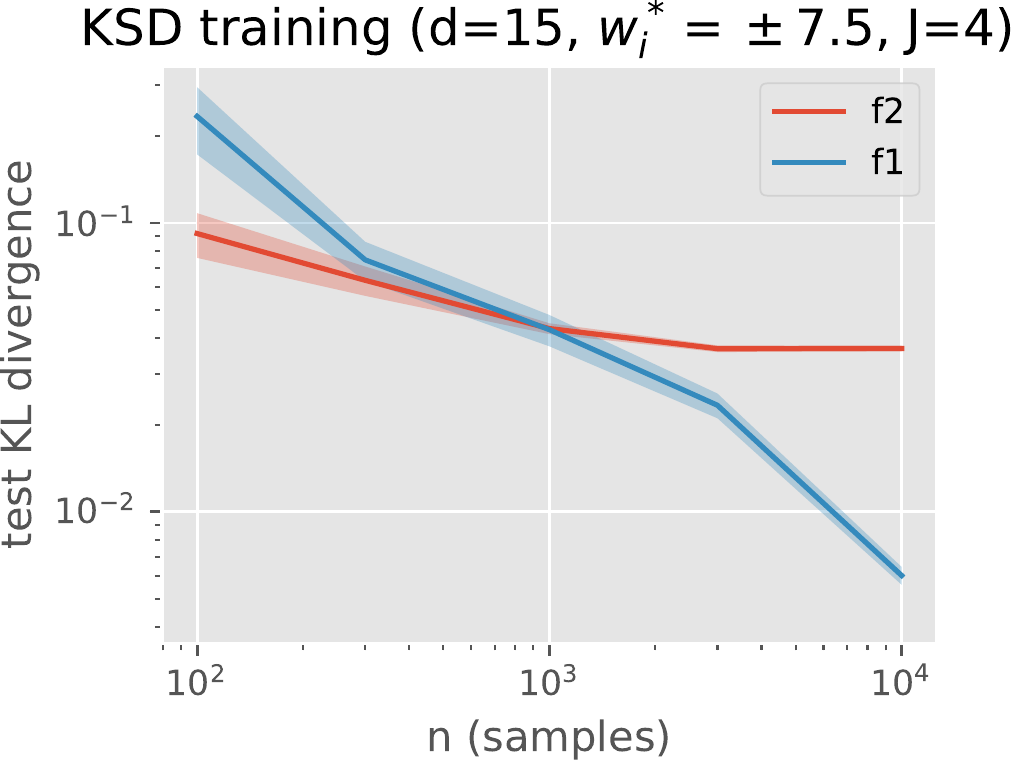}
    \includegraphics[width=.3\textwidth]{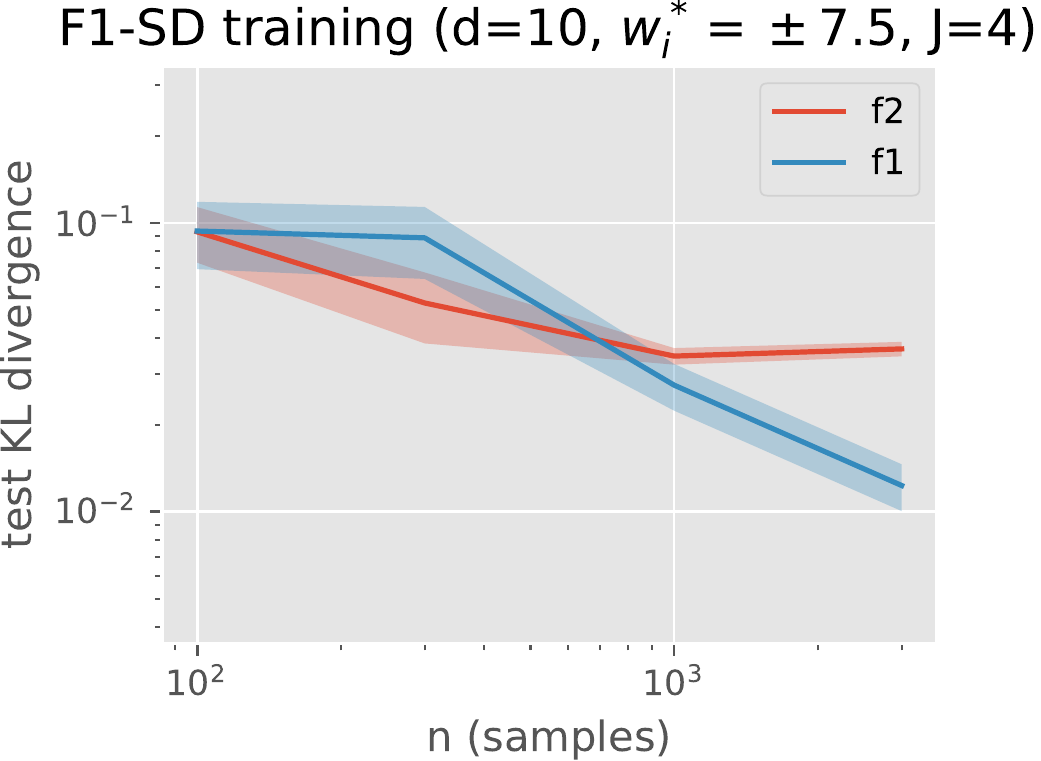} \\
    \includegraphics[width=.3\textwidth]{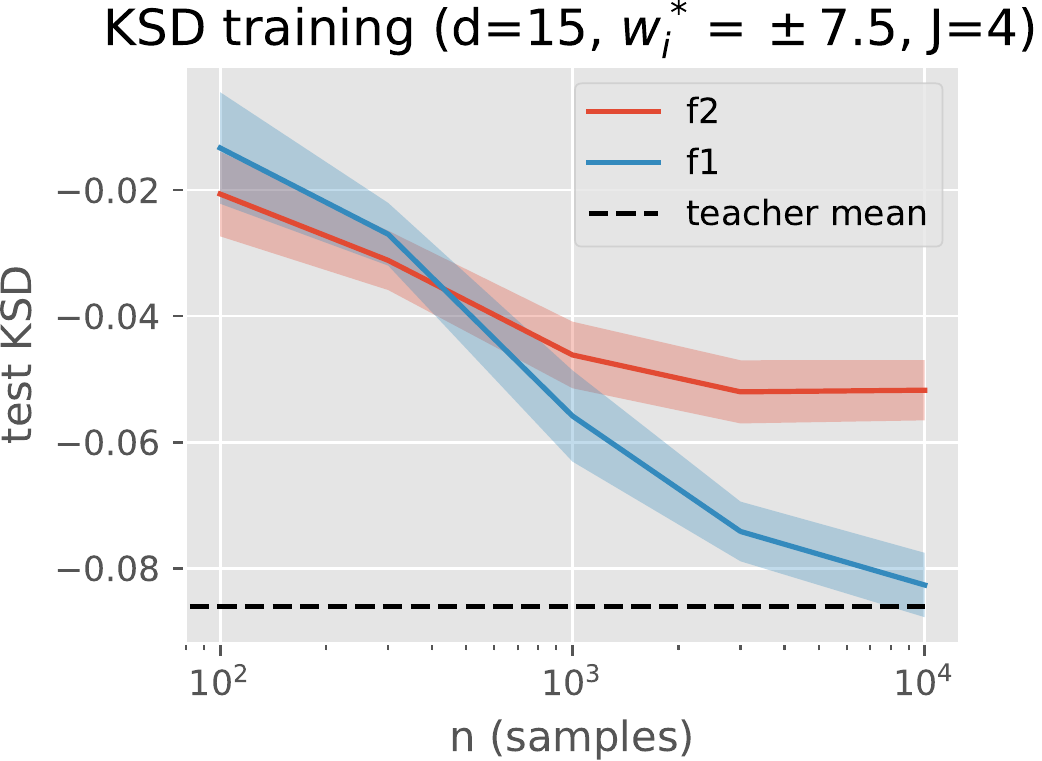}
    \includegraphics[width=.3\textwidth]{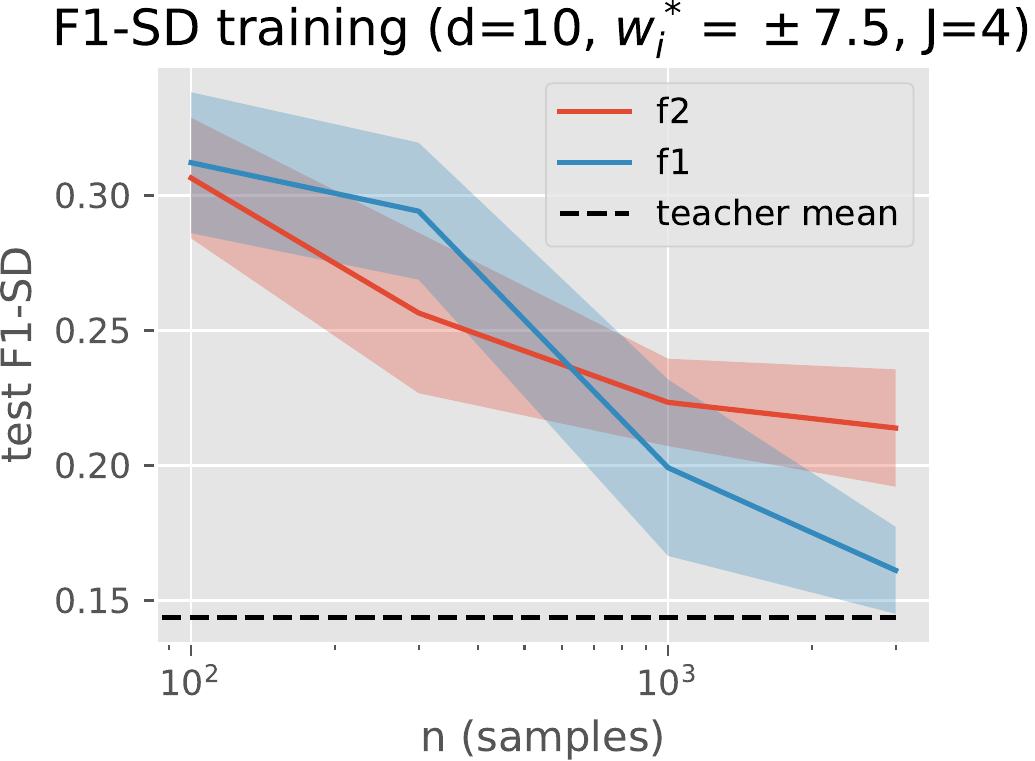}
    \caption{Test metrics obtained for MLE, KSD and $\mathcal{F}_1$-SD training on a four-neuron teacher with weights $w_1^*, w_2^* = 7.5$ and $w_3^*, w_4^* = -7.5$. (top) Test performance measured with cross-entropy estimates. (bottom) Test KSD and $\mathcal F_1$-SD for models trained with the same metric. For reference, the black discontinuous lines show the teacher KSD and $\mathcal{F}_1$-SD of the teacher model w.r.t. 5000 and 2000 test samples, respectively. Confidence estimates are over 10 different data samplings.}
    \label{fig:four_neurons}
\end{figure*}

In this section, we present numerical experiments illustrating our theory on simple synthetic datasets generated by teacher models with energies
    $f^*(x) = \frac{1}{J} \sum_{j=1}^{J} w^*_j \sigma(\langle \theta^*_j, x \rangle)$,
with~$\theta^*_i \in \mathbb{S}^d$ for all~$i$. The code for the experiments is in \url{https://github.com/CDEnrich/ebms_shallow_nn}.

\paragraph{Experimental setup.}
We generate data on the sphere $\mathbb{S}^d$ from teacher models by using a simple rejection sampling strategy, given an estimate of the minimum of~$f^*$ (which provides an estimated upper bound on the unnormalized density~$e^{-f^*}$ for rejection sampling). 
This minimum is estimated using gradient descent with many random restarts from uniform points on the sphere.
For different numbers of training samples, we run our gradient-based algorithms in~$\mathcal F_1$ and $\mathcal F_2$ with different choices of step-sizes and regularization parameters~$\lambda$, using $m=500$ neurons.
We report test metrics after selecting hyperparameters on a validation set of 2000 samples.
For computing gradients in maximum likelihood training, we use a simple Metropolis-Hastings algorithm with uniform proposals on the sphere.
To obtain non-negative test KL divergence estimates, which are needed for the log-log plots, we sample large numbers of points uniformly on the hypersphere, and compute the KL divergence of the restriction of the EBMs to these points. 
The sampling techniques that we use are effective for the toy problems considered, but more refined techniques might be needed for more complex problems in higher dimension or lower temperatures. 

\paragraph{Learning planted neuron distributions in hyperspheres.}
We consider the task of learning planted neuron distributions in $d=15$ and $d=10$. Remark that in this setting, when $\mathcal{F} = \mathcal{B}_{\mathcal{F}_1}(\beta)$ with $\beta$ large enough there is no approximation error.
We compare the behavior of~$\mathcal F_1$ and~$\mathcal F_2$ models with different estimators in Figures \ref{fig:one_positive_neuron}, \ref{fig:two_negative_neurons} and \ref{fig:four_neurons}, corresponding to models with $J=1,2,4$ teacher neurons, respectively. The error bars show the average and standard deviation for 10 runs. 
In the three figures, the top plot in the first column represents the test KL divergence of the $\mathcal{F}_1$ and $\mathcal{F}_2$ EBMs trained with maximum likelihood for an increasing number of samples, showcasing the adaptivity of $\mathcal{F}_1$ to distributions with low-dimensional structure versus the struggle of the $\mathcal{F}_2$ model. In Figures \ref{fig:one_positive_neuron} and \ref{fig:two_negative_neurons} the bottom plot in the first column shows the same information for a teacher with the same structure but different values for the output weights. We observe that the separation between the $\mathcal{F}_1$ and the $\mathcal{F}_2$ models increases when the teacher models have higher weights.

In the three figures, the plots in the second column show the test KL divergence and test KSD, respectively, for EBMs trained with KSD (with RBF kernel with $\sigma^2=1$). We observe that we are able to train EBMs successfully by optimizing the KSD; even though maximum likelihood training is directly optimizing the KL divergence, the test KL divergence values we obtain for the KSD-trained models are on par, or even slightly better, comparing at equal values of $n$.  
It is also worth noticing that in \autoref{fig:one_positive_neuron}, we observe a separation between $\mathcal{F}_1$ and $\mathcal{F}_2$ in the KL divergence plot, but not in the KSD plot. It seems that in this particular instance, although the training is successful, the KSD is too weak of a metric to tell that the $\mathcal{F}_1$ EBMs are better than $\mathcal{F}_2$ EBMs.

In the three figures, the plots in the third column show the test KL divergences and test $\mathcal{F}_1$-SD for EBMs trained with $\mathcal{F}_1$-SD. Remark that the error bars are wider due to the two timescale algorithm used for $\mathcal{F}_1$-SD, which seems to introduce more variability. While the plots only go up to $n=3000$, the test cross-entropy curves show a separation between $\mathcal{F}_1$ and $\mathcal{F}_2$ very similar to maximum likelihood training when comparing at equal values of $n$. 

\autoref{sec:additional_exp} contains additional experiments for the cases $J=1, w_1^{*} = 10$ and $J=2, w_i^{*} = -2.5$, training with KSD and $\mathcal{F}_1$-SD.

\paragraph{3D visualizations and time evolution in $d=3$ ($\mathcal{F}_1$ EBM trained with MLE).}
\autoref{fig:3D} shows a 3D visualization of the teacher and trained models, energies and densities corresponding to two teacher neurons with negative weights in $d=3$. Since the dimension is small and the temperature is not too small, we used train and test sizes for which the statistical error due to train and test samples is negligible. Interestingly, while the $\mathcal{F}_1$ model achieves a KL divergence close to zero at the end of training (\autoref{fig:3D_KL}), in \autoref{fig:3D} we see that the positions of the neurons of the trained model do not match the teacher neurons. In fact, there are some neurons with positive weights in the high energy region. This effect might be linked with the fact that there is a constant offset of around $0.3$ between the teacher energy and the trained energy. The offset is not reflected in the Gibbs measures of the models, which are invariant to constant terms. 

\autoref{fig:3D_KL} also shows that for this particular instance, the convergence is polynomial in the iteration number. We attach a video of the training dynamics: \url{https://github.com/CDEnrich/ebms_shallow_nn/blob/main/KLfeatzunnorm1.mp4}.

\begin{figure}
    \centering
    \includegraphics[width=.49\textwidth]{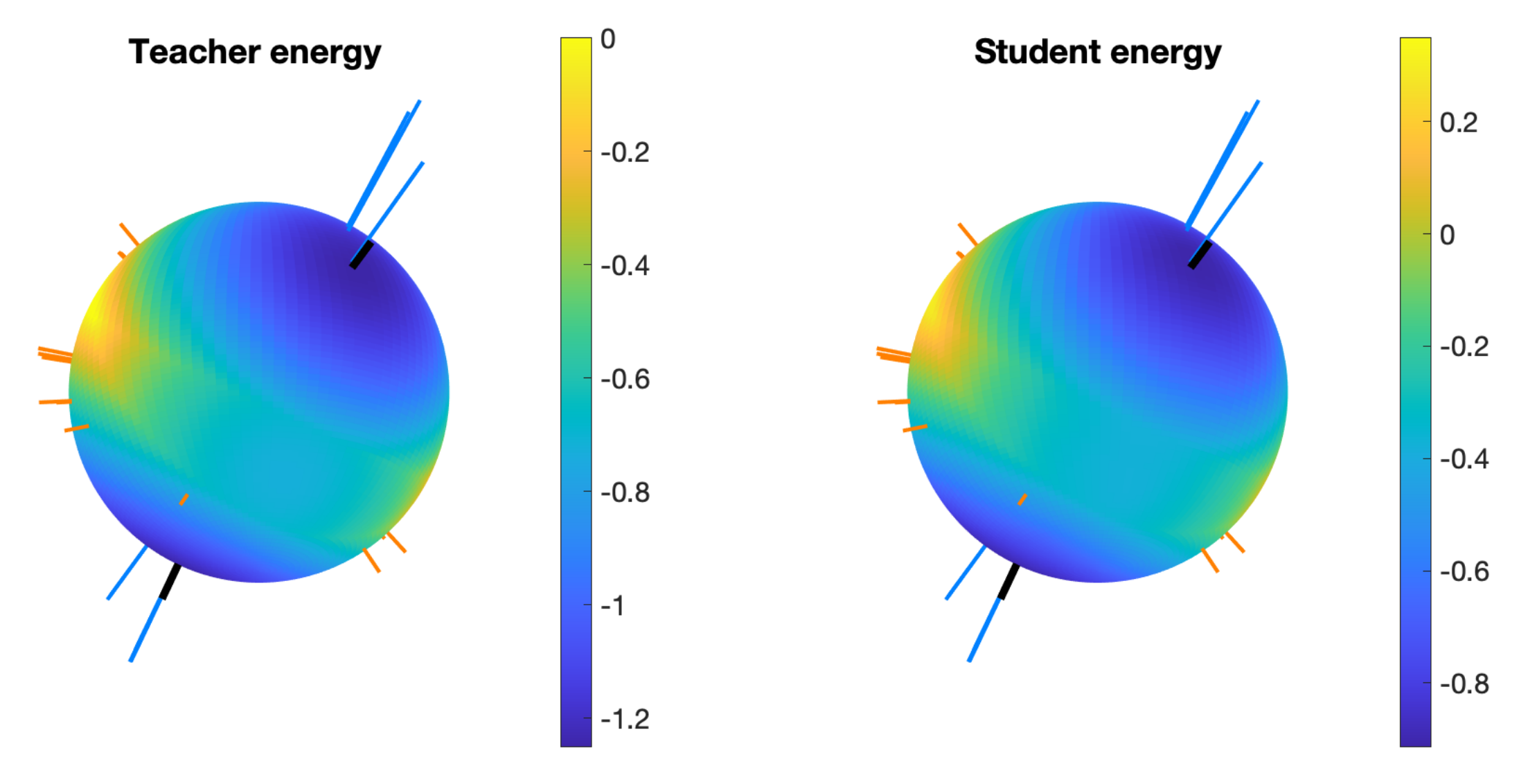}~~~
    \includegraphics[width=.48\textwidth]{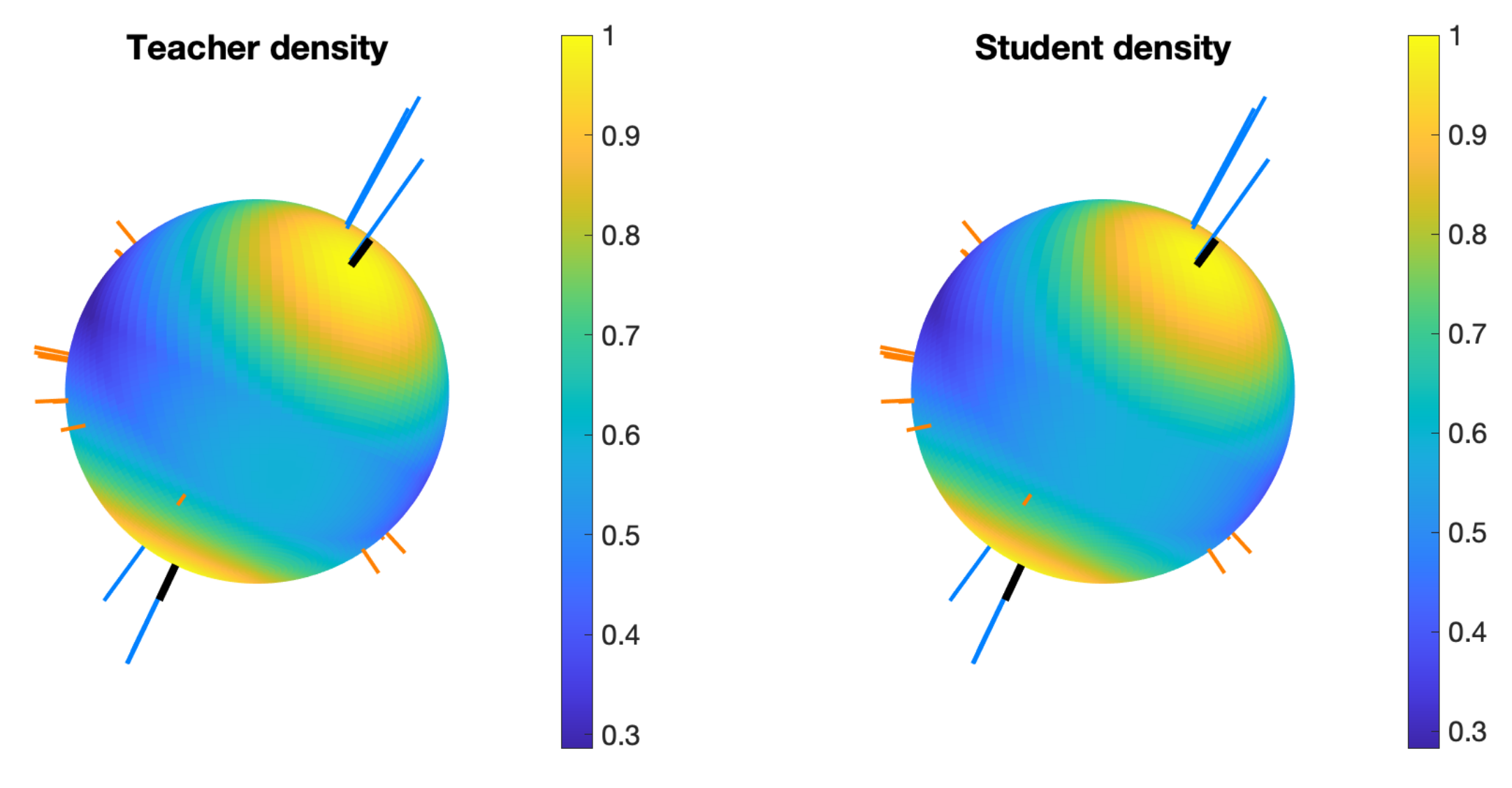}
    \caption{3D visualization of the neuron positions, energies and densities, in $d=3$. The teacher model has two neurons with negative weights $w_1^{*}, w_2^{*} = -2.5$, whose positions are represented by black sticks in all the images. The positions of the neurons of the trained model are represented by blue and orange sticks for negative and positive weights, resp. The two images on the left show the energies of the teacher and trained models, respectively. The energies look qualitatively very similar up to an offset of $\approx 0.3$. The two images on the right show the Gibbs densities of the teacher and trained models, respectively.}
    \label{fig:3D}
\end{figure}


\begin{figure}
\begin{minipage}[c]{6.5cm}
\includegraphics[width=\textwidth]{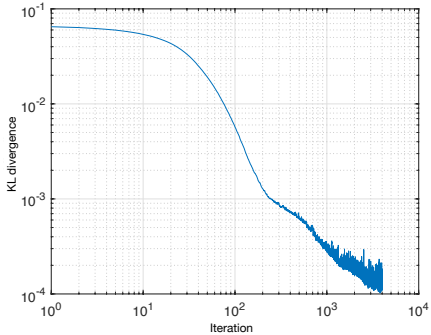}
\end{minipage}
\begin{minipage}[c]{10cm}
\caption{Log-log plot of the KL divergence between the MLE trained model and the teacher model (same as in \autoref{fig:3D}), versus the iteration number. }
\label{fig:3D_KL}
\end{minipage}
\end{figure}

%% file: conclusions.tex
We provide statistical error bounds for EBMs trained with KL divergence or Stein discrepancies, and show benefits of using energy models with infinite-width shallow networks in in ``active'' regimes in terms of adaptivity to distributions with low-dimensional structure in the energy.
We empirically verify that networks in ``kernel'' regimes perform significantly worse in the presence of such structures, on simple teacher-student experiments.

A theoretical separation result in KL divergence or SD between $\mathcal{F}_1$ and $\mathcal{F}_2$ EBMs remains an important open question: one major difficulty for providing a lower bound on the performance for~$\mathcal F_2$ is that $L^2$ (or $L^{\infty}$) approximation may be not be appropriate for capturing the hardness the problem,
since two log-densities differing substantially in low energy regions can have arbitrarily small KL divergence. Another direction for future work is to apply the theory of shallow overparametrized neural networks to other generative models such as GANs or normalizing flows. 

On the computational side, in \autoref{sec:convergence} we leverage existing work to state qualitative convergence results in an idealized setting of infinite width and exact gradients, but it would be interesting to develop convergence results for maximum likelihood that take the MCMC sampling into account, as done for instance by~\citet{debortoli2020efficient} for certain exponential family models. In our setting, this would entail identifying a computationally tractable subset of $\mathcal{F}_1$ energies.  
A more ambitious and long-term goal is to instead move beyond the MCMC paradigm, and devise efficient sampling strategies that can operate outside the class of log-concave densities, as for instance \cite{marylou}.


\section*{Acknowledgements}

We thank Marylou Gabri\'e for useful discussions. CD acknowledges partial support by ``la Caixa'' Foundation (ID 100010434), under agreement LCF/BQ/AA18/11680094.
EVE acknowledges partial support from the National Science Foundation (NSF) Materials Research Science and Engineering Center Program grant DMR-1420073, NSF DMS- 1522767, and the Vannevar Bush Faculty Fellowship. JB acknowledges partial support from the Alfred P. Sloan Foundation, NSF RI-1816753, NSF CAREER CIF 1845360, NSF CHS-1901091 and Samsung Electronics.




%% file: proofs_statistical.tex
\generalbound*

\begin{proof}
In the first place, remark that for all $\nu_1, \nu_2 \in \mathcal{P}(K)$ that are absolutely continuous w.r.t. $p$, we have $D_{KL}(\nu_1|| \nu_2) = \int_{K} \log(\frac{d\nu_1}{d\tau}(x)) d\nu_1(x) - \int_{K} \log(\frac{d\nu_2}{d\tau}(x)) d\nu_1(x) = -H(\nu_1) + H(\nu_1,\nu_2)$, where $H(\nu_1,\nu_2)$ is the cross-entropy and $H(\nu_1)$ is the differential entropy. Hence, for all $\nu_1, \nu_2, \nu_3 \in \mathcal{P}(K)$,
\begin{align} \label{eq:KL_cross_entropy}
    D_{KL}(\nu_1|| \nu_2) - D_{KL}(\nu_1|| \nu_3) = H(\nu_1,\nu_2) - H(\nu_1,\nu_3).
\end{align}

Secondly, notice that for any $\nu \in \mathcal{P}(K)$ and measurable $f:K \rightarrow \R$, 
\begin{align}
\begin{split} \label{eq:average_cross_entropy}
    \int f(x) \ d\nu(x) &= - \int \log(e^{- f(x)}) \ d\nu(x) = - \int \log \left(\frac{d\nu_{ f}}{d\tau}(x) \right) \ d\nu(x) - \log\left(\int e^{-f(x)} d\tau(x) \right)\\ &= H(\nu, \nu_{f})  - \log\left(\int e^{- f(x)} d\tau(x) \right),
\end{split}
\end{align}
Thus, if we apply \eqref{eq:average_cross_entropy} on $\nu$ and its empirical version $\nu_n = \frac{1}{n} \sum_{i=1}^n \delta_{x_i}$, we obtain that with probability at least $1-\delta$, for all $f \in \mathcal{F}$:
\begin{align}
\begin{split} \label{eq:avg_cross_entropy}
    |H(\nu,\nu_{f}) - H(\nu_n,\nu_{f})| &= \left| \frac{1}{n} \sum_{i=1}^n f(x_i) - \int f(x) \ d\nu(x) \right| \leq \left( 2\mathcal{R}_n(\mathcal{F}) + \left(\sup_{f \in \mathcal{F}} \|f\|_{\infty} \right) \sqrt{\frac{2\log(1/\delta)}{n}} \right) \\ &\leq  \frac{2\beta C}{\sqrt{n}} + \beta \sqrt{\frac{2\log(1/\delta)}{n}}, 
\end{split}
\end{align}
where we have used the Rademacher generalization bound 
(\citet{mohri2012foundations}, Theorem 3.3) and the Rademacher complexity bound from the assumption of the theorem.

We have
\begin{align}
\begin{split}
    D_{KL}(\nu || \hat{\nu}) &= D_{KL}(\nu || \hat{\nu}) - \inf_{f \in \mathcal{F}} D_{KL}(\nu || \nu_{f}) + \inf_{f \in \mathcal{F}} D_{KL}(\nu || \nu_{f}) \\ &= \sup_{f \in \mathcal{F}} \left\{ H(\nu, \hat{\nu}) - H(\nu, \nu_{f}) \right\} + \inf_{f \in \mathcal{F}} D_{KL}(\nu || \nu_{f})
    \\ &\leq \sup_{f \in \mathcal{F}} \left\{ H(\nu_n, \hat{\nu}) - H(\nu_n, \nu_{f}) \right\} +  \frac{4\beta C}{\sqrt{n}} + \beta \sqrt{\frac{8 \log(1/\delta)}{n}} + \inf_{f \in \mathcal{F}} D_{KL}(\nu || \nu_{f}) \\ &= \frac{4\beta C}{\sqrt{n}} + \beta \sqrt{\frac{8 \log(1/\delta)}{n}} + \inf_{f \in \mathcal{F}} D_{KL}(\nu || \nu_{f}). 
\end{split}
\end{align}
This proves \eqref{eq:thm1_1}. In the second equality we have used \eqref{eq:KL_cross_entropy}. For the inequality we have used \eqref{eq:avg_cross_entropy} twice, i.e. that $H(\nu, \hat{\nu}) = H(\nu, \nu_{\hat{f}}) \leq H(\nu_n, \nu_{ \hat{f}}) + \frac{2 \beta C}{\sqrt{n}} + \beta \sqrt{\frac{2 \log(1/\delta)}{n}}$ and that $-H(\nu, \nu_{f}) \leq -H(\nu_n, \nu_{f}) + \frac{2\beta C}{\sqrt{n}} + \beta \sqrt{\frac{2 \log(1/\delta)}{n}}$. In the last equality we have used that by the definition of $\hat{f}$, $H(\nu_n, \hat{\nu}) = H(\nu, \nu_{\beta \hat{f}}) = \min_{f \in \mathcal{F}} H(\nu, \nu_{f})$.  

For the proof of \eqref{eq:thm1_2} we apply \autoref{lem:kl_l_infty} into \eqref{eq:thm1_1}.
\end{proof}

\begin{lemma} \label{lem:kl_l_infty}
Let $g : K \rightarrow \R$ be such that $\frac{d\nu}{d\tau}(x) = e^{- g(x)}/\int_{K} e^{- g(y)} d\tau(y)$, i.e. $-g$ is the log-density of $\nu$ up to a constant term. Then,
\begin{align}
    \inf_{f \in \mathcal{F}} D_{\text{KL}}(\nu || \nu_{f}) \leq 2 \inf_{f \in \mathcal{F}} \|g - f\|_{\infty}
\end{align}
\end{lemma}
\begin{proof}
Notice that $\nu = \nu_{g}$. Thus, for any $f \in \mathcal{F}$,
\begin{align}
\begin{split} \label{eq:d_kl_bound}
    D_{KL}(\nu || \nu_{f}) &= D_{KL}(\nu_{g} || \nu_{f}) = \int \log\left(\frac{\frac{d\nu_g}{d\tau}(x)}{\frac{d\nu_{f}}{d\tau}(x)}\right) d\nu_g(x) = \int \log\left(\frac{\frac{e^{-g(x)}}{\int e^{- g(y)} d\tau(y)}}{\frac{e^{-f(x)}}{\int e^{-f(y)} d\tau(y)}}\right) \ d\nu_g(x) \\ &= \int (f(x) - g(x)) \ d\nu_g(x) - \log \left(\int e^{-g(y)} d\tau(y) \right) + \log \left( \int e^{- f(y)} d\tau(y) \right).
\end{split}
\end{align}
Here, we bound
\begin{align}
     \int (f(x) - g(x)) \ d\nu_g(x) \leq \|f-g\|_{\infty},
\end{align}
and applying \autoref{lem:free_energy_equality} to $f$ and $g$, we obtain 
\begin{align}
\begin{split}
    \log \left( \int e^{-f(y)} \ d\tau(y) \right) - \log \left(\int e^{-g(y)} \ d\tau(y) \right) \leq \| f-g\|_{\infty}.
\end{split}
\end{align}
Plugging these two bounds into \eqref{eq:d_kl_bound}, we obtain $D_{KL}(\nu || \nu_{f}) \leq 2 \|f-g\|_{\infty}$, which yields the result.
\end{proof}

We do not claim that the upper-bound in \autoref{lem:kl_l_infty} is tight; it might be possible to provide a bound involving a weaker metric. Regardless, it suffices for our purposes.  

\begin{lemma} \label{lem:free_energy_equality}
Let $f: K \rightarrow \R$, $g: K \rightarrow \R$ be measurable functions. For some $\alpha \in [0,1]$,
\begin{align}
    \log \left( \int_{K} e^{-f(y)} d\tau(y) \right) - \log \left(\int_{K} e^{-g(y)} d\tau(y) \right) = \int_{K} \frac{e^{- \left( \alpha f(y) + (1-\alpha) g(y)\right)}}{\int_{K} e^{- \left(\alpha f(x) + (1-\alpha) g(x)\right)} \ d\tau(x)} \left(f(y) - g(y)\right) \ d\tau(y)
\end{align}
\end{lemma}
\begin{proof}
We define the function 
\begin{align}
    F(\alpha) = \log \left( \int_{K} e^{- \left(\alpha f(y) + (1-\alpha) g(y)\right)} \ d\tau(y) \right),
\end{align}
which has derivative
\begin{align}
    \frac{dF}{d\alpha}(\alpha) = \frac{- \int_{K} e^{- \left(\alpha f(y) + (1-\alpha) g(y)\right)} (f(y) - g(y)) \ d\tau(y)}{\int_{K} e^{- \left(\alpha f(x) + (1-\alpha) g(x)\right)} \ d\tau(x)} = - \int_{K}  (f(y) - g(y)) p_{\alpha}(y) \ d\tau(y),
\end{align}
where $p_{\alpha}(y)$ is the density of the Gibbs probability measure corresponding to the energy $\alpha f + (1-\alpha) g$.
We make use of the mean value theorem:
\begin{align}
\begin{split}
    &\log \left( \int_{K} e^{-f(y)} d\tau(y) \right) - \log \left(\int_{K} e^{- g(y)} d\tau(y) \right) = F(1) - F(0) = \frac{dF}{d\alpha}(\alpha) (1-0) \\ &= - \int_{K}  (f(y) - g(y)) p_{\alpha}(y) \ d\tau(y).
\end{split}
\end{align}
\end{proof}

\begin{lemma} [Approximation of Lipschitz functions by $\mathcal{F}_2$ balls, Proposition 6 of \cite{bach2017breaking}] \label{lem:lipschitz_bach}
For $\delta$ greater than a constant depending only on $d$, 
for any function $f : \{ x \in \mathbb{R}^d | \|x\|_2 \leq R \} \rightarrow \R$ such that for all $x, y$ such that $\|x\|_2 \leq R, \ \|y\|_2 \leq R$ we have $|f(x)| \leq \eta$ and $|f(x) - f(y)| \leq \eta R^{-1} \|x-y\|_2$, 
there exists $h \{ x \in \mathbb{R}^d | \|x\|_2 \leq R \} \times \{R\} \rightarrow \R \in \mathcal{F}_2$, such that $\|h\|_{\mathcal{F}_2} \leq \delta$ and
\begin{align}
\sup_{\|x\|_2 \leq R} |h(x,R) - f(x)| \leq C(d) \eta \left( \frac{R \delta}{\eta} \right)^{-2/(d+1)} \log \left(\frac{R \delta}{\eta} \right)
\end{align}
\end{lemma}

\begin{proof}
From \citet{bach2017breaking}. Notice that the factor in the bound is $\left( \frac{R \delta}{\eta} \right)^{-2/(d+1)} \log \left(\frac{R \delta}{\eta} \right)$, while in the original paper it is $\left( \frac{\delta}{\eta} \right)^{-2/(d+1)} \log \left(\frac{\delta}{\eta} \right)$. 
The $R$ factor stems from the fact that we consider the neural network features to lie in $\mathbb{S}^d$, while \citet{bach2017breaking} considers them in the hypersphere of radius $R^{-1}$.
\end{proof}

\begin{lemma} [Rademacher complexity bound for $\mathcal{B}_{\mathcal{F}_1}$, Section 5.1 of \cite{bach2017breaking}; \citet{kakade2009onthe}] \label{lem:rademacher_b_f1}
Suppose that $K \subseteq \{x \in \mathbb{R}^{d+1} | \|x\|_2 \leq R \}$. The Rademacher complexity of the function class $\mathcal{B}_{\mathcal{F}_1}$ is bounded by
\begin{align}
    \mathcal{R}_n(\mathcal{B}_{\mathcal{F}_1}) \leq \frac{R}{\sqrt{n}}.
\end{align}
\end{lemma}

\statfone*

\begin{proof}
We will use \eqref{eq:thm1_2} from \autoref{thm:generalbound}. We have that $g : K_0 \times \{R\} \rightarrow \R$ is defined as $g(x,R) = \sum_{j=1}^{J} g_j(x,R) = \sum_{j=1}^{J} \phi_j(U_j x, R)$.

By Lemma \ref{lem:lipschitz_bach}, there exists $\psi_j : \{ x \in \R^{k} | \|x\|_2 \leq R \} \times{R} \rightarrow \R$ such that $\psi_j \in \mathcal{F}_2$ and $\|\psi_j\|_{\mathcal{F}_2} \leq \beta/J$, and
\begin{align} \label{eq:psi_phi_1}
    \sup_{x \in \R^k : \|x\|_2 \leq R} |\psi_j(x,R) - \phi_j(x)| \leq C(k) \eta \left( \frac{R \beta}{\eta J} \right)^{-2/(k+1)} \log \left(\frac{R \beta}{\eta J} \right)
\end{align}
Hence, if we define $\tilde{g}_j : K_0 \times \{R\} \rightarrow \R$ as $\tilde{g}_j(x,R) := \psi_j(U_j x, R)$, we have that $\tilde{g}_j$ belongs to $\mathcal{F}_1$ by an argument similar to the one of Section 4.6 of \cite{bach2017breaking}. Namely, if we write $\psi_j(x,R) = \int_{\mathbb{S}^k} \sigma(\langle \theta, (x,R) \rangle) \ d\gamma(\theta)$ for some signed measure $\gamma$, we have
\begin{align}
    \tilde{g}_j(x,R) = \int_{\mathbb{S}^k} \sigma(\langle \theta, (U_j x, R) \rangle) \ d\gamma(\theta) = \int_{\mathbb{S}^k} \sigma(\langle U_j^{\top} \theta_{1:d}, x \rangle + \theta_{d+1} R) \ d\gamma(\theta) = \int_{\mathbb{S}^d} \sigma(\langle \theta'_{1:d}, x \rangle + \theta'_{d+1} R) \ d\gamma'(\theta'),
\end{align}
where we used the change of variable $\theta' = (U_j^{\top} \theta_{1:d}, \theta_{d+1})$, which maps $\mathbb{S}^k$ to $\mathbb{S}^d$. Moreover, this shows that $\tilde{g}_j$ has $\mathcal{F}_1$ norm $\|\tilde{g}_j\|_{\mathcal{F}_1} \leq \|\psi_j\|_{\mathcal{F}_2} \leq \beta/J$, which means that $\tilde{g} = \sum_{j=1}^J \tilde{g}_j \in \mathcal{F}_1$ and $\|\tilde{g}\|_{\mathcal{F}_1} \leq \beta$. Moreover,
\begin{align}
\begin{split} \label{eq:g_f_star}
    \|\tilde{g}_j - g_j\|_{\infty} &= \sup_{x \in K_0} |\tilde{g}_j(x,R) - g_j(x,R)| = \sup_{x \in K_0} |\psi_j(U_j x) - \phi_j(U_j x)| \leq \sup_{x \in \R^k : \|x\|_2 \leq R} |\psi_j(x) - \phi_j(x)| \\ &\leq C(k) \eta \left( \frac{R \beta}{\eta J} \right)^{-2/(k+1)} \log \left(\frac{R \beta}{\eta J} \right)
\end{split}
\end{align}
The first inequality holds because for all $x \in K$, $\|Ux\|_2 \leq \|x\|_2 \leq R$ by the fact that $U$ has orthonormal rows, and the second inequality holds by \eqref{eq:psi_phi_1}.
Thus,
\begin{align} \label{eq:approx_bound}
    \inf_{f \in \mathcal{B}_{\mathcal{F}_1}} \|g - f\|_{\infty} \leq \|g - \tilde{g}\|_{\infty} \leq \sum_{j=1}^{J} \|\tilde{g}_j - g_j\|_{\infty} \leq  C(k) J \eta \left( \frac{R \beta}{\eta J} \right)^{-2/(k+1)} \log \left(\frac{R \beta}{\eta J} \right)
\end{align}
Notice that the assumptions of \autoref{thm:generalbound} are fulfilled: the Rademacher complexity bound for $\mathcal{B}_{\mathcal{F}_1}$ (\autoref{lem:rademacher_b_f1}) implies that $\mathcal{R}_n(\mathcal{B}_{\mathcal{F}_1}(\beta)) \leq \frac{\beta R}{\sqrt{n}}$ and it is also easy to check that $\sup_{f \in \mathcal{B}_{\mathcal{F}_1}(\beta)} \|f\|_{\infty} \leq \beta$.
Plugging \eqref{eq:approx_bound} into \eqref{eq:thm1_2} we obtain
\begin{align}
    D_{\text{KL}}(\nu || \hat{\nu}) \leq 4\beta \frac{\sqrt{2} R}{\sqrt{n}} + \beta \sqrt{\frac{2\log(1/\delta)}{n}} + 2C(k) J \eta \left( \frac{R \beta}{\eta J} \right)^{-2/(k+1)} \log \left(\frac{R \beta}{\eta J} \right).
\end{align}
If we minimize the right-hand side w.r.t. $\beta$ (disregarding the log factor), we obtain that the optimal value is
\begin{align}
    \left(\frac{2B}{k+1} \right)^{\frac{k+1}{k+3}} \left( \frac{A}{\sqrt{n}}\right)^{\frac{2}{k+3}} + B^{\frac{k+1}{k+3}} \left( \frac{A(k+1)}{2 \sqrt{n}}\right)^{\frac{2}{k+3}} \log \left( \frac{R}{\eta} \left( \frac{2B \sqrt{n}}{A (k+1)} \right)^{\frac{k+1}{k+3}}\right),
\end{align}
and the optimal $\beta$ is $\left( 2B \sqrt{n}/(A (k+1)) \right)^{\frac{k+1}{k+3}}$, where
\begin{align}
    A = 4\sqrt{2}R + \sqrt{2\log(1/\delta)}, \quad B = 2 C(k) (J \eta)^{\frac{k+3}{k+1}} R^{-\frac{2}{k+1}}.
\end{align}
\end{proof}

\begin{lemma} [Stein operator for functions on $\mathbb{S}^d$] \label{lem:stein_operator}
For a probability measure $\nu$ on the sphere $\mathbb{S}^d$ with a continuous and almost everywhere differentiable density $\frac{d\nu}{d\tau}$, the Stein operator $\mathcal{A}_{\nu}$ is defined as
\begin{align}
    (\mathcal{A}_{\nu} h)(x) = \left(\nabla \log \left(\frac{d\nu}{d\tau}(x) \right) - d x \right) h(x)^{\top} + \nabla h(x),
\end{align}
for any $h : \mathbb{S}^d \rightarrow \mathbb{R}^{d+1}$ that is continuous and almost everywhere differentiable, where $\nabla$ denotes the Riemannian gradient.
That is, for any $h : \mathbb{S}^d \rightarrow \mathbb{R}^{d+1}$ that is continuous and almost everywhere differentiable, the Stein identity holds:
\begin{align}
    \mathbb{E}_{\nu} [(\mathcal{A}_{\nu} h)(x)] = 0.
\end{align}
\end{lemma}
\begin{proof}
Let $h_i : \mathbb{S}^d \rightarrow \mathbb{R}$ be the $i$-th component of $h$. Notice that
\begin{align} \label{eq:prod_canonical}
    \mathbb{E}_{\nu} \left[\nabla \log \left(\frac{d\nu}{d\tau}(x) \right) h_i(x) + \nabla h_i(x) \right] = \mathbb{E}_{\nu} \left[\nabla \left(\frac{d\nu}{d\tau}(x) h_i(x) \right) \frac{1}{\frac{d\nu}{d\tau}(x)} \right] = \int_{\mathbb{S}^d} \nabla \left(\frac{d\nu}{d\tau}(x) h_i(x) \right) d\tau(x) 
\end{align}
Now, if we take the inner product of the right-hand side with the canonical basis vector $e_k \in \mathbb{R}^{d+1}$, we obtain 
\begin{align}
\begin{split}
    &\bigg\langle \int_{\mathbb{S}^d} \nabla \left(\frac{d\nu}{d\tau}(x) h_i(x) \right) d\tau(x), e_k \bigg\rangle = \int_{\mathbb{S}^d} \bigg\langle (I - x x^{\top}) \nabla \left(\frac{d\nu}{d\tau}(x) h_i(x) \right), e_k \bigg\rangle d\tau(x) \\ &= \int_{\mathbb{S}^d} \bigg\langle \nabla \left(\frac{d\nu}{d\tau}(x) h_i(x) \right), e_k - x_k x \bigg\rangle d\tau(x) = -\int_{\mathbb{S}^d}  \frac{d\nu}{d\tau}(x) h_i(x) \nabla \cdot (e_k - x_k x) d\tau(x),
\end{split}
\end{align}
where in the second equality we used that $(I - x x^{\top})$ the projection matrix to the tangent space of $\mathbb{S}^d$ at $x$, in the third equality we used that it is symmetric, and in the last equality we used integration by parts on $\mathbb{S}^d$ ($\nabla \cdot$ denotes the Riemannian divergence). 

To compute $\nabla \cdot (e_k - x_k x)$, remark that by the invariance to change of basis it is equal to the divergence of the function $g : \mathbb{R}^{d+1} \rightarrow \mathbb{R}^{d+1}$ defined as $x \rightarrow e_k - \frac{x_k x}{\|x\|^2}$, when restricted to $\mathbb{S}^d$. And we have
\begin{align}
    \nabla \cdot g (x) = \sum_{j=1}^{d+1} \partial_j g_j(x) = \sum_{j=1}^{d+1} \partial_j \left( e_{k,j} - \frac{x_k x_j}{\|x\|^2} \right) = - \left(\sum_{j=1}^{d+1} \frac{x_k}{\|x\|^2} \right) + \left(\sum_{j=1}^{d+1} \frac{2 x_k x_j^2}{\|x\|^4} \right) - \frac{x_k}{\|x\|^2}
\end{align}
For $x \in \mathbb{S}^d$, the right-hand side simplifies to $ - (d+1) x_k + 2 x_k - x_k = - d x_k$, which means that the right-hand side of \eqref{eq:prod_canonical} becomes
\begin{align}
    -\int_{\mathbb{S}^d}  \frac{d\nu}{d\tau}(x) h_i(x) (- d x_k) \ d\tau(x) = d \mathbb{E}_{\nu}[h_i(x) x_k].
\end{align}
That means that $\mathbb{E}_{\nu} \left[\nabla \log \left(\frac{d\nu}{d\tau}(x) \right) h_i(x) + \nabla h_i(x) - d h_i(x) x \right] = 0$, which concludes the proof.
\end{proof}

\begin{lemma}[Kernelized Stein discrepancy for probability measures on $\mathbb{S}^d$] \label{lem:ksd_expression}
For $K = \mathbb{S}^d$, and $\nu_1, \nu_2 \in \mathcal{P}(K)$ with continuous, almost everywhere differentiable log-densities, the kernelized Stein discrepancy $\text{KSD}(\nu_1, \nu_2)$ is equal to
\begin{align} 
\begin{split} \label{eq:KSD_def_sd}
     \sup_{h \in \mathcal{B}_{\mathcal{H}_0^d}} (\mathbb{E}_{\nu_1} [\text{Tr}(\mathcal{A}_{\nu_2} h(x))])^2 = \mathbb{E}_{x,x' \sim \nu_1}[(s_{\nu_2}(x) - s_{\nu_1}(x))^{\top} (s_{\nu_2}(x') - s_{\nu_1}(x')) k(x,x')] = \mathbb{E}_{x,x' \sim \nu_1} [u_{\nu_2}(x,x')], 
\end{split}
\end{align}
where $u_{\nu}(x,x') = (s_{\nu}(x) - d \cdot x)^\top(s_{\nu}(x') - d \cdot x') k(x,x') + (s_{\nu}(x) - d \cdot x)^\top \nabla_{x'} k(x,x') + (s_{\nu}(x') - d \cdot x')^\top \nabla_{x} k(x,x') + \text{Tr}(\nabla_{x,x'} k(x,x'))$.
\end{lemma}
\begin{proof}
The argument for the first equality is from Theorem 3.8 of \citet{liu2016akernelized}, but we rewrite it with our notation. Using the Stein identity, which holds by \autoref{lem:stein_operator}, we have
\begin{align}
\begin{split}
    &\mathbb{E}_{\nu_1} [\text{Tr}(\mathcal{A}_{\nu_2} h(x))] = \mathbb{E}_{\nu_1} [\text{Tr}(\mathcal{A}_{\nu_2} h(x) - \mathcal{A}_{\nu_1} h(x))] = \mathbb{E}_{\nu_1} [ (s_{\nu_2}(x) - s_{\nu_1}(x))^{\top} h(x)], \\
    &\sup_{h \in \mathcal{B}_{\mathcal{H}_0^d}} \mathbb{E}_{\nu_1} [ (s_{\nu_2}(x) - s_{\nu_1}(x))^{\top} h(x)] = \sup_{h \in \mathcal{B}_{\mathcal{H}_0^d}} \int_{\mathbb{S}^d} \frac{d\nu_1}{d\tau}(x) \sum_{i=1}^{d+1} (s_{\nu_2}^{(i)}(x) - s_{\nu_1}^{(i)}(x)) h_i(x) d\tau(x) \\ &= \sup_{h \in \mathcal{B}_{\mathcal{H}_0^d}} \sum_{i=1}^{d+1} \left\langle \int_{\mathbb{S}^d} \frac{d\nu_1}{d\tau}(x) (s_{\nu_2}^{(i)}(x) - s_{\nu_1}^{(i)}(x)) k(x,\cdot) d\tau(x), h_i(\cdot) \right\rangle_{\mathcal{H}_0} = \sqrt{\sum_{i=1}^{d+1} \left\|\int_{\mathbb{S}^d} \frac{d\nu_1}{d\tau}(x) (s_{\nu_2}^{(i)}(x) - s_{\nu_1}^{(i)}(x)) k(x,\cdot) d\tau(x) \right\|_{\mathcal{H}_0}^2} \\ &= \sqrt{\sum_{i=1}^{d+1} \int_{\mathbb{S}^d \times \mathbb{S}^d}  \frac{d\nu_1}{d\tau}(x) (s_{\nu_2}^{(i)}(x) - s_{\nu_1}^{(i)}(x)) k(x,x') \frac{d\nu_1}{d\tau}(x') (s_{\nu_2}^{(i)}(x') - s_{\nu_1}^{(i)}(x')) d\tau(x) d\tau(x')}.
\end{split}
\end{align}
Given the form of the Stein operator for functions on $\mathbb{S}^d$ (\autoref{lem:stein_operator}), the proof of the second equality of \eqref{eq:KSD_def_sd} is a straightforward analogy of the proof of Theorem 3.6 of \cite{liu2016akernelized}, which is for the Stein operator for functions on $\R^d$.
\end{proof}

\begin{restatable}{thm}{generalboundsd}
\label{thm:generalboundsd}
Let $K = \mathbb{S}^d$. Assume that the class $\mathcal{F}$ is such that $\sup_{f \in \mathcal{F}} \{ \|\nabla_i f\|_{\infty} | 1 \leq i \leq d+1 \} \leq \beta C_1$. Assume that $\mathcal{H} = \mathcal{B}_{\prod_{i=1}^{d+1} \mathcal{H}_i} = \{ (h_i)_{i=1}^{d+1} \ | \ h_i \in \mathcal{H}_i, \sum_{i=1}^{d+1} \|h_i\|_{\mathcal{H}_i} \leq 1\}$, where $\mathcal{H}_i$ are normed spaces of functions from $\mathbb{S}^d$ to $\R$. Assume that the following Rademacher complexity type bounds hold for $1 \leq i \leq d+1$: $\mathbb{E}_{\mathbf{\sigma}, S_n} \left[ \sup_{h_i \in \mathcal{B}_{\mathcal{H}_i}} \frac{1}{n} \sum_{j=1}^n  \sigma_{j} h_i(x_j) \right] \leq \frac{C_2}{\sqrt{n}},$ $\mathbb{E}_{\mathbf{\sigma}, S_n} \left[ \sup_{h_i \in \mathcal{B}_{\mathcal{H}_i}} \frac{1}{n} \sum_{j=1}^n \sigma_{j} \nabla_i h_i(x_j) \right] \leq \frac{C_3}{\sqrt{n}}$, and that $\|h_i\|_{\infty} \leq M, \|\nabla_i h_i\|_{\infty} \leq M$ for all $h_i \in \mathcal{H}_i$.

If we take $n$ samples $\{x_i\}_{i=1}^n$ of a target measure $\nu$ with almost everywhere differentiable log-density, and consider the Stein Discrepancy estimator (SDE) $\hat{\nu} := \nu_{\hat{f}}$, where $\hat{f}$ is the estimator defined in \eqref{eq:sd_estimator}, we have that with probability at least $1-\delta$, $\text{SD}_{\mathcal{H}}(\nu, \hat{\nu})$ is upper-bounded by
\begin{align}
\begin{split}
    \frac{4\sqrt{d+1} ((\beta C_1 + R d)C_2 + C_3)}{\sqrt{n}} + 2 M (\beta C_1 + 1 + Rd) \sqrt{\frac{(d+1)\log((d+1)/\delta)}{2n}} + \inf_{f \in \mathcal{F}} \text{SD}_{\mathcal{H}}(\nu, \nu_{f}).
\end{split}
\end{align}
\end{restatable}

\begin{proof} 
Notice that by the definition of the Stein operator,
\begin{align}
\begin{split}
    \text{Tr}(\mathcal{A}_{\nu_{f}} h(x)) &= \text{Tr}\left( \left( \nabla \log \left(\frac{d\nu_{f}}{d\tau}(x) \right) - d x \right) h(x)^{\top} + \nabla h(x) \right) = \text{Tr}\left(- (\nabla f(x) + d x ) h(x)^{\top} + \nabla h(x) \right) \\ &= \sum_{i=1}^{d+1} -( \nabla_i f(x) + d x_i) h_i(x) + \nabla_i h_i(x) 
\end{split}
\end{align}
Thus, 
\begin{align}
\begin{split} \label{eq:rademacher_argument}
    & \sup_{h \in \mathcal{H}} \mathbb{E}_{\nu} [\text{Tr}(\mathcal{A}_{\nu_{f}} h(x))] - \mathbb{E}_{\nu_n} [\text{Tr}(\mathcal{A}_{\nu_{f}} h(x))] \\ &=  \sup_{h \in \mathcal{H}} \sum_{i=1}^{d+1}  (\mathbb{E}_{\nu}[-( \nabla_i f(x) + d x_i) h_i(x) + \nabla_i h_i(x)] - \mathbb{E}_{\nu_n}[-(\nabla_i f(x) + d x_i) h_i(x) + \nabla_i h_i(x) ]) 
    \\ &=  \sup_{\substack{\sum_{i} |w_i|^2 \leq 1\\ h_i \in \mathcal{B}_{\mathcal{H}_i}}} \sum_{i=1}^{d+1}  w_i (\mathbb{E}_{\nu}[-(\nabla_i f(x) + d x_i) h_i(x) + \nabla_i h_i(x)] - \mathbb{E}_{\nu_n}[-(\nabla_i f(x) + d x_i) h_i(x) + \nabla_i h_i(x) ])
    \\ &= \sqrt{ \sum_{i=1}^{d+1} \left(\sup_{ h_i \in \mathcal{B}_{\mathcal{H}_i}} (\mathbb{E}_{\nu}[-(\nabla_i f(x) + d x_i) h_i(x) + \nabla_i h_i(x)] - \mathbb{E}_{\nu_n}[-(\nabla_i f(x) + d x_i) h_i(x) + \nabla_i h_i(x)]) \right)^2 }
    \\ &= \sqrt{ \sum_{i=1}^{d+1} \Phi_i(S_n)^2 }, 
\end{split}
\end{align}
where $\Phi_i(S_n) = \sup_{ h_i \in \mathcal{B}_{\mathcal{H}_i}} (\mathbb{E}_{\nu}[-(\nabla_i f(x) + d x_i) h_i(x) + \nabla_i h_i(x)] - \mathbb{E}_{\nu_n}[-(\nabla_i f(x) + d x_i) h_i(x) + \nabla_i h_i(x) ])$. For a fixed $i$, we can use a classical argument based on McDiarmid's inequality (c.f. \citet{mohri2012foundations}, Theorem 3.3) to obtain
\begin{align}
    \mathbb{P} \left( \Phi_i(S_n) - \mathbb{E}_{S'_n} \left[\Phi_i(S'_n) \right] \geq \epsilon \right) \leq \exp \left( \frac{-2\epsilon^2 n}{C_4^2} \right),
\end{align}
where $C_4 = M(\beta C_1 + 1 + R d)$ is a uniform upper-bound on $\{ \|-(\nabla_i f(x) + d x_i) h_i(x) + \nabla_i h_i(x)\|_{\infty} \ | \ h_i \in \mathcal{B}_{\mathcal{H}_i} \}$. 
Thus, using a union bound, we obtain that 
\begin{align}
    \mathbb{P} \left( \max_{1 \leq i \leq d+1} \left(\Phi_i(S_n) - \mathbb{E}_{S'_n} \left[\Phi_i(S'_n) \right] \right) \geq \epsilon \right) \leq (d+1) \exp \left( \frac{-2\epsilon^2 n}{C^2} \right),
\end{align}
and through a change of variables, that means that with probability at least $1-\delta$,
\begin{align}
\begin{split}
    &\max_{1 \leq i \leq d+1} \left(\Phi_i(S_n) - \mathbb{E}_{S'_n} \left[\Phi_i(S'_n) \right] \right) \leq C_4 \sqrt{\frac{\log((d+1)/\delta)}{2n}} \\ &\implies \max_{1 \leq i \leq d+1} \Phi_i(S_n) \leq \max_{1 \leq i \leq d+1} \mathbb{E}_{S'_n} \Phi_i(S'_n) + C_4 \sqrt{\frac{\log((d+1)/\delta)}{2n}}
    \\ &\implies \sqrt{ \sum_{i=1}^{d+1} \Phi_i(S_n)^2 } \leq \sqrt{d+1} \max_{1 \leq i \leq d+1} \Phi_i(S_n) \leq \sqrt{d+1} \max_{1 \leq i \leq d+1} \mathbb{E}_{S'_n} \Phi_i(S'_n) + C_4 \sqrt{\frac{(d+1)\log((d+1)/\delta)}{2n}}
\end{split}
\end{align}
All that is left is to upper-bound $\mathbb{E}_{S_n} \Phi_i(S_n)$ for any $i$ using Rademacher complexity bounds:
\begin{align}
\begin{split}
    &\mathbb{E}_{S_n} \left[ \sup_{h_i \in \mathcal{B}_{\mathcal{H}_i}}  (\mathbb{E}_{\nu}[-(\nabla_i f(x) + d x_i) h_i(x) + \nabla_i h_i(x)] - \mathbb{E}_{\nu_n}[-(\nabla_i f(x) + d x_i) h_i(x) + \nabla_i h_i(x) ]) \right]
    \\ &\leq \mathbb{E}_{S_n, S'_n} \left[ \sup_{h_i \in \mathcal{B}_{\mathcal{H}_i}} \frac{1}{n} \sum_{j=1}^n - ((\nabla_i f(x'_j) + d x'_{j,i}) h_i(x'_j) - (\nabla_i f(x_j) + d x_{j,i}) h_i(x_j)) + \nabla_i h_i(x'_j) - \nabla_i h_i(x_j) \right]
    \\ &= \mathbb{E}_{\mathbf{\sigma}, S_n, S'_n} \left[ \sup_{h_i \in \mathcal{B}_{\mathcal{H}_i}} \frac{1}{n} \sum_{j=1}^n \sigma_{j} \left( - ((\nabla_i f(x'_j) + d x'_{j,i}) h_i(x'_j) - (\nabla_i f(x_j) + d x_{j,i}) h_i(x_j)) + \nabla_i h_i(x'_j) - \nabla_i h_i(x_j) \right) \right],
\end{split}
\end{align}
and this is upper-bounded by
\begin{align}
\begin{split} \label{eq:rademacher_argument_2}
     &2 \mathbb{E}_{\mathbf{\sigma}, S_n} \left[ \sup_{h_i \in \mathcal{B}_{\mathcal{H}_i}} \frac{1}{n} \sum_{j=1}^n \sigma_{j} \left( -(\nabla_i f(x_j) + d x_{j,i}) h_i(x_j) + \nabla_i h_i(x_j) \right) \right] \\
    &\leq 2 \mathbb{E}_{\mathbf{\sigma}, S_n} \left[ \sup_{h_i \in \mathcal{B}_{\mathcal{H}_i}} \frac{1}{n} \sum_{j=1}^n \sigma_{j} (\nabla_i f(x_j) + d x_{j,i}) h_i(x_j) \right] + 2 \mathbb{E}_{\mathbf{\sigma}, S_n} \left[ \sup_{h_i \in \mathcal{B}_{\mathcal{H}_i}} \frac{1}{n} \sum_{j=1}^n \sigma_{j} \nabla_i h_i(x_j) \right] 
\end{split}
\end{align}
By Talagrand's Lemma (\citet{mohri2012foundations}, Theorem 5.7) and the uniform $L^{\infty}$ bound on $\{ \nabla_i f | f \in \mathcal{F}\}$ (notice that $y \mapsto (\beta \nabla_i f(x_j) + d x_{j,i}) y$ has Lipschitz constant uniformly upper-bounded by $\|\beta \nabla_i f(x_j) + d x_{j,i}\|_{\infty}$, which means that the assumptions of Talagrand's Lemma are fulfilled), we have
\begin{align}
    \mathbb{E}_{\mathbf{\sigma}, S_n} \left[ \sup_{h_i \in \mathcal{B}_{\mathcal{H}_i}} \frac{1}{n} \sum_{j=1}^n \sigma_{j} (\nabla_i f(x_j) + d x_{j,i}) h_i(x_j) \right] \leq (C_1 \beta + Rd) \mathbb{E}_{\mathbf{\sigma}, S_n} \left[ \sup_{h_i \in \mathcal{B}_{\mathcal{H}_i}} \frac{1}{n} \sum_{j=1}^n \sigma_{j} h_i(x_j) \right] \leq \frac{(\beta C_1 + R d) C_2}{\sqrt{n}}, 
\end{align}
where we used the Rademacher complexity bound of $\mathcal{B}_{\mathcal{H}_i}$.
Using the Rademacher complexity bound of $\nabla_i h_i$ as well, we conclude that the right-hand side of \eqref{eq:rademacher_argument_2} can be upper-bounded by
$\frac{2 (\beta C_1 + R d)C_2 + 2 C_3}{\sqrt{n}}$. 
Thus, with probability at least $1-\delta$, for all $h \in \mathcal{H}$, 
\begin{align} \label{eq:rademacher_sd}
    \left| \mathbb{E}_{\nu} [\text{Tr}(\mathcal{A}_{\nu_{f}} h(x))] - \mathbb{E}_{\nu_n} [\text{Tr}(\mathcal{A}_{\nu_{f}} h(x))] \right| \leq \frac{2\sqrt{d+1} ((\beta C_1 + R d)C_2 + C_3)}{\sqrt{n}} + C_4 \sqrt{\frac{(d+1)\log((d+1)/\delta)}{2n}}
\end{align}
We conclude the proof with an argument similar to the one of \autoref{thm:generalbound}:
\begin{align}
\begin{split}
    &\text{SD}_{\mathcal{H}}(\nu, \hat{\nu}) \\ &= \text{SD}_{\mathcal{H}}(\nu, \hat{\nu}) - \inf_{f \in \mathcal{F}} \text{SD}_{\mathcal{H}}(\nu, \nu_{f}) + \inf_{f \in \mathcal{F}} \text{SD}_{\mathcal{H}}(\nu, \nu_{f}) \\ &= \sup_{f \in \mathcal{F}} \left\{ \sup_{h \in \mathcal{H}} \mathbb{E}_{\nu} [\text{Tr}(\mathcal{A}_{\hat{\nu}} h(x))] - \sup_{h \in \mathcal{H}} \mathbb{E}_{\nu} [\text{Tr}(\mathcal{A}_{\nu_{f}} h(x))] \right\} + \inf_{f \in \mathcal{F}} \text{SD}_{\mathcal{H}}(\nu, \nu_{f})
    \\ &\leq \sup_{f \in \mathcal{F}} \left\{ \sup_{h \in \mathcal{H}} \mathbb{E}_{\nu_n} [\text{Tr}(\mathcal{A}_{\hat{\nu}} h(x))] - \sup_{h \in \mathcal{H}} \mathbb{E}_{\nu_n} [\text{Tr}(\mathcal{A}_{\nu_{f}} h(x))] \right\} + \frac{4\sqrt{d+1} ((\beta C_1 + R d)C_2 + C_3)}{\sqrt{n}} \\ &+ 2 C_4 \sqrt{\frac{\log((d+1)/\delta)}{2n}} + \inf_{f \in \mathcal{F}} \text{SD}_{\mathcal{H}}(\nu, \nu_{f}) \\ &= \frac{4\sqrt{d+1} ((\beta C_1 + R d)C_2 + C_3)}{\sqrt{n}} + 2 M (\beta C_1 + 1 + Rd) \sqrt{\frac{(d+1)\log((d+1)/\delta)}{2n}} + \inf_{f \in \mathcal{F}} \text{SD}_{\mathcal{H}}(\nu, \nu_{f}). 
\end{split}
\end{align}
In the second equality we use the definition of the Stein discrepancy (equation \eqref{eq:sd_def}). The inequality follows from \eqref{eq:rademacher_sd} applied on $\nu_{f}$ and on $\hat{\nu} = \nu_{\hat{f}}$. The last equality holds because of the definition of $\hat{f}$ and the definition of $C_4$.
\end{proof}

\coronesd*

\begin{proof}
Note that \autoref{lem:stein_operator} provides the expression for the Stein operator $\mathcal{A}_{\nu}$ on $\mathbb{S}^d$ and shows that for any $\nu \in \mathcal{P}(\mathbb{S}^d)$ with continuous and a.e. differentiable density, the class of continuous and a.e. differentiable functions $\mathbb{S}^d \rightarrow \R^{d+1}$ is contained in the Stein class of $\nu$ (which by definition is the set of functions $h$ such that the Stein identity $\mathbb{E}_{\nu} [A_{\nu} h] = 0$ holds).
Using the argument of Lemma 2.3 of \cite{liu2016akernelized}, we have that for any $\nu_1, \nu_2 \in \mathcal{P}(K)$, for any $h$ in the Stein class of $\nu_1$ we have
\begin{align}
\begin{split}
    &\mathbb{E}_{\nu_1}[\mathcal{A}_{\nu_2} h(x)] = \mathbb{E}_{\nu_1}[\mathcal{A}_{\nu_2} h(x) - \mathcal{A}_{\nu_1} h(x)] \\ &= \mathbb{E}_{\nu_1}[s_{\nu_2}(x)  h(x)^{\top}  + \nabla h(x) - d x h(x)^{\top} - (s_{\nu_1}(x)  h(x)^{\top}  + \nabla h(x) - d x h(x)^{\top})] = \mathbb{E}_{\nu_1}[(s_{\nu_2}(x) - s_{\nu_1}(x)) h(x)^{\top}],
\end{split}
\end{align}
which follows from the definition of the Stein operator and the Stein identity: $\mathbb{E}_{\nu_1}[\mathcal{A}_{\nu_1} h(x)] = 0$. Thus, for any $\nu \in \mathcal{P}(K)$,
\begin{align}
\begin{split} \label{eq:sd_cor1}
    \text{SD}_{\mathcal{B}_{\mathcal{F}_1^{d+1}}}(\nu, \nu_{f}) &= \sup_{h \in \mathcal{B}_{\mathcal{F}_1^{d+1}}} \mathbb{E}_{\nu}[\text{Tr}((s_{\nu_{f}}(x) - s_{\nu}(x)) h(x)^{\top})] \\ &= \sup_{h \in \mathcal{B}_{\mathcal{F}_1^{d+1}}} \sum_{i=1}^{d+1} \mathbb{E}_{\nu} \left[ \left(-\nabla_i f(x) - \nabla_i \log \left(\frac{d\nu}{d\tau}(x) \right) \right) h_i(x) \right] 
    \\ &= \sup_{\substack{\sum_i |w_i|^2 \leq 1, \\ |\gamma_i|_{\text{TV}} \leq 1 }} \sum_{i=1}^{d+1} \mathbb{E}_{\nu} \left[ \left(-\nabla_i f(x) - \nabla_i \log \left(\frac{d\nu}{d\tau}(x) \right) \right)  w_i \int_{\mathbb{S}_d} \sigma(\langle \theta, x \rangle) d\gamma_i(\theta) \right] 
    \\ &\leq \mathbb{E}_{\nu} \bigg[ \sup_{\substack{\sum_i |w_i|^2 \leq 1, \\ |\gamma_i|_{\text{TV}} \leq 1 }} \sum_{i=1}^{d+1} w_i \bigg(-\nabla_i f(x) - \nabla_i \log \left(\frac{d\nu}{d\tau}(x) \right) \bigg) \int_{\mathbb{S}_d} \sigma(\langle \theta, x \rangle) d\gamma_i(\theta) \bigg] 
    \\ &= \mathbb{E}_{\nu} \bigg[ \sup_{\substack{\sum_i |w_i|^2 \leq 1, \\ \{\theta^{(i)}\} \subset \mathbb{S}^d}} \sum_{i=1}^{d+1} w_i  \bigg(-\nabla_i f(x) - \nabla_i \log \left(\frac{d\nu}{d\tau}(x) \right) \bigg) \sigma(\langle \theta^{(i)}, x \rangle) \bigg] 
    \\ &= \mathbb{E}_{\nu} \bigg[ \bigg(\sum_{i=1}^{d+1} \sup_{\{\theta^{(i)}\} \subset \mathbb{S}^d} \bigg( \bigg(-\nabla_i f(x) - \nabla_i \log \left(\frac{d\nu}{d\tau}(x) \right) \bigg) \sigma(\langle \theta^{(i)}, x \rangle) \bigg)^2 \bigg)^{1/2}\bigg]
    \\ &= \mathbb{E}_{\nu} \bigg[ \bigg(\sum_{i=1}^{d+1}  \bigg(-\nabla_i f(x) - \nabla_i \log \left(\frac{d\nu}{d\tau}(x) \right) \bigg)^2 \bigg)^{1/2}\bigg]
    = \mathbb{E}_{\nu} \bigg[ \bigg\|  -\nabla f(x) - \nabla \log \left(\frac{d\nu}{d\tau}(x) \right) \bigg\|_2 \bigg]
\end{split}
\end{align}
Moreover, by \autoref{lem:rademacher_b_f1}:
\begin{align}
\begin{split} \label{eq:rademacher1_cor1}
\mathbb{E}_{\mathbf{\sigma}, S_n} \left[ \sup_{h \in \mathcal{B}_{\mathcal{F}_1}} \frac{1}{n} \sum_{j=1}^n \sigma_{j} h(x_j) \right] =
\mathcal{R}_n(\mathcal{B}_{\mathcal{F}_1}) \leq \frac{1}{\sqrt{n}}.
\end{split}
\end{align}
And
\begin{align}
\begin{split} \label{eq:rademacher2_cor1}
    &\mathbb{E}_{\mathbf{\sigma}, S_n} \left[ \sup_{h \in \mathcal{B}_{\mathcal{F}_1}} \frac{1}{n} \sum_{j=1}^n \sigma_{j} \nabla_i h(x_j) \right] = \mathbb{E}_{\mathbf{\sigma}, S_n} \left[ \sup_{ |\gamma|_{\text{TV}} \leq 1} \frac{1}{n} \sum_{j=1}^n \sigma_{j}  \int_{\mathbb{S}^d} \nabla_i\sigma(\langle \theta, x_j \rangle) d\gamma(\theta) \right] \\ &= \mathbb{E}_{\mathbf{\sigma}, S_n} \left[ \sup_{\theta \in \mathbb{S}^d, |w| \leq 1} \frac{w}{n} \sum_{j=1}^n \sigma_{j}  \mathds{1}_{\langle \theta, x_j \rangle \geq 0} \theta_i \right] = \mathbb{E}_{\mathbf{\sigma}, S_n} \left[ \sup_{\theta \in \mathbb{S}^d} \left|\frac{1}{n} \sum_{j=1}^n \sigma_{j}  \mathds{1}_{\langle \theta, x_j \rangle \geq 0} \theta_i \right| \right] \\ &\leq \mathbb{E}_{\mathbf{\sigma}, S_n} \left[ \sup_{\theta \in \mathbb{S}^d} \left|\frac{1}{n} \sum_{j=1}^n \sigma_{j}  \mathds{1}_{\langle \theta, x_j \rangle \geq 0} \right| \right] \leq C_2 \frac{\sqrt{d+1}}{\sqrt{n}},
\end{split}
\end{align}
where the last inequality follows from the Rademacher complexity bound on the hyperplane hypothesis, which is obtained through a VC dimension argument (\citet{bach2017breaking}, Section 5.1; \citet{bartlett2002rademacher}, Theorem 6). 
Moreover, $\|h\|_{\infty} \leq 1$ and $\|\nabla_i h\|_{\infty} \leq 1$ for all $h \in \mathcal{F}_1$.
The proof concludes by plugging \eqref{eq:sd_cor1}, \eqref{eq:rademacher1_cor1}, \eqref{eq:rademacher2_cor1} into \autoref{thm:generalboundsd}. Since $\mathcal{B}_{\mathcal{F}_2^{d+1}} \subset \mathcal{B}_{\mathcal{F}_1^{d+1}}$, all the upper-bounds of the proof hold for $\mathcal{H} = \mathcal{B}_{\mathcal{F}_2^{d+1}}$ as well.
\end{proof}

\begin{restatable}{thm}{ksdfirst} \label{thm:ksdfirst}
Let $K=\mathbb{S}^d$. Let $KSD$ be the kernelized Stein discrepancy for a positive definite kernel $k$ with continuous second order partial derivatives, such that for any non-zero function $g \in L^2(\mathbb{S}^d)$, $\int_{\mathbb{S}^d} \int_{\mathbb{S}^d} g(x) k(x,x') g(x') d\tau(x) d\tau(x') > 0$. If we take $n$ samples $\{x_i\}_{i=1}^n$ of a target measure $\nu$ with almost everywhere differentiable log-density, and consider the unbiased KSD estimator \eqref{eq:ksd_estimator2}, we have with probability at least $1-\delta$,
\begin{align}
    KSD(\nu, \hat{\nu}) &\leq \frac{2}{\sqrt{\delta n}} \sup_{f \in \mathcal{F}} (\text{Var}_{x \sim \nu}(\mathbb{E}_{x' \sim \nu} [\tilde{u}_{\nu_{f}}(x,x')]))^{1/2} \\ &+ \sqrt{\mathbb{E}_{x,x' \sim \nu} [k(x,x')^2]} \inf_{f \in \mathcal{F}} \mathbb{E}_{x \sim \nu} \left[\bigg\|\nabla \log \left(\frac{d\nu}{d\tau}(x) \right) - \beta \nabla f(x) \bigg\|^2 \right]
\end{align}
\end{restatable}

\begin{proof}
For the kernelized Stein discrepancy estimator we can write
\begin{align}
\begin{split}
    &KSD(\nu, \hat{\nu}) \\ &= KSD(\nu, \hat{\nu}) - \inf_{f \in \mathcal{F}} KSD(\nu, \nu_{f}) + \inf_{f \in \mathcal{F}} KSD(\nu, \nu_{f}) 
    \\ &= \sup_{f \in \mathcal{F}} \left\{ \mathbb{E}_{x,x' \sim \nu} [u_{\hat{\nu}}(x,x')] - \mathbb{E}_{x,x' \sim \nu} [u_{\nu_{f}}(x,x')] \right\} + \inf_{f \in \mathcal{F}} KSD(\nu, \nu_{f})
    \\ &= \sup_{f \in \mathcal{F}} \left\{ \mathbb{E}_{x,x' \sim \nu} [\tilde{u}_{\hat{\nu}}(x,x')] - \mathbb{E}_{x,x' \sim \nu} [\tilde{u}_{\nu_{f}}(x,x')] \right\} + \inf_{f \in \mathcal{F}} KSD(\nu, \nu_{f})
    \\ &\leq \sup_{f \in \mathcal{F}} \left\{ \frac{1}{n(n-1)} \left( \sum_{i \neq j} \tilde{u}_{\hat{\nu}}(x_i,x_j) - \sum_{i \neq j} \tilde{u}_{\nu_{f}}(x_i,x_j) \right) \right\} + \frac{2}{\sqrt{\delta n}} \sup_{f \in \mathcal{F}} (\text{Var}_{x \sim \nu}(\mathbb{E}_{x' \sim \nu} [u_{\nu_{f}}(x,x')]))^{1/2} \\ &+ \inf_{f \in \mathcal{F}} KSD(\nu, \nu_{f}) = \frac{2}{\sqrt{\delta n}} \sup_{f \in \mathcal{F}} (\text{Var}_{x \sim \nu}(\mathbb{E}_{x' \sim \nu} [u_{\nu_{f}}(x,x')]))^{1/2}  + \inf_{f \in \mathcal{F}} KSD(\nu, \nu_{f})
\end{split}
\end{align}
The third equality holds because of the definition of $\tilde{u}_{\nu}$ in terms of $u_{\nu}$.
In the first inequality we have used that for any $\tilde{\nu}$ (different from $\nu$) with almost-everywhere differentiable log-density, $\frac{1}{n(n-1)} \sum_{i \neq j} \tilde{u}_{\tilde{\nu}}(x_i,x_j)$ has expectation $\mathbb{E}_{x,x' \sim \nu} [\tilde{u}_{\hat{\nu}}(x,x')]$ and variance $\text{Var}_{x \sim \nu}(\mathbb{E}_{x' \sim \nu} [\tilde{u}_{\tilde{\nu}}(x,x')])/n$ by the theory of U-statistics (\citet{liu2016akernelized}, Theorem 4.1; \citet{serfling2009approximation}, Section 5.5). 
Thus, by Chebyshev's inequality, with probability at least $1-\delta$, we have that $\mathbb{E}_{x,x' \sim \nu} [\tilde{u}_{\tilde{\nu}}(x,x')] \leq \frac{1}{n(n-1)} \sum_{i \neq j} \tilde{u}_{\tilde{\nu}}(x_i,x_j) + \frac{1}{\sqrt{n \delta}} (\text{Var}_{x \sim \nu}(\mathbb{E}_{x' \sim \nu} [\tilde{u}_{\tilde{\nu}}(x,x')]))^{1/2}$.

Moreover, using the argument of Theorem 5.1 of \cite{liu2016akernelized}, by \autoref{lem:ksd_expression},
\begin{align}
\begin{split}
    KSD(\nu, \nu_{f}) &= \mathbb{E}_{x,x' \sim \nu}[(s_{\nu}(x) - s_{\nu_{f}}(x))^{\top} (s_{\nu}(x') - s_{\nu_{f}}(x')) k(x,x')] \\ &\leq \sqrt{\mathbb{E}_{x,x' \sim \nu} [k(x,x')^2]} \sqrt{\mathbb{E}_{x,x' \sim \nu} \left[\left((s_{\nu}(x) - s_{\nu_{f}}(x))^{\top} (s_{\nu}(x') - s_{\nu_{f}}(x')) \right)^2 \right]}
    \\ &\leq \sqrt{\mathbb{E}_{x,x' \sim \nu} [k(x,x')^2]} \sqrt{\mathbb{E}_{x,x' \sim \nu} \left[\|s_{\nu}(x) - s_{\nu_{f}}(x)\|^2 \|s_{\nu}(x') - s_{\nu_{f}}(x')\|^2 \right]}
    \\ &= \sqrt{\mathbb{E}_{x,x' \sim \nu} [k(x,x')^2]} \ \mathbb{E}_{x \sim \nu} \left[\|s_{\nu}(x) - s_{\nu_{f}}(x)\|^2 \right],
\end{split}
\end{align}
where $\mathbb{E}_{x \sim \nu} \left[\|s_{\nu}(x) - s_{\nu_{f}}(x)\|^2 \right]$ is known as the Fisher divergence. 
\end{proof}

\ksdsecond*

\begin{proof}
We apply \autoref{thm:ksdfirst}. We can bound
\begin{align}
\sup_{f \in \mathcal{F}} (\text{Var}_{x \sim \nu}(\mathbb{E}_{x' \sim \nu} [u_{\nu_{f}}(x,x')]))^{1/2} &\leq \sup_{f \in \mathcal{F}} (\mathbb{E}_{x \sim \nu}(\mathbb{E}_{x' \sim \nu} [u_{\nu_{f}}(x,x')])^2)^{1/2} \leq \sup_{f \in \mathcal{F}} (\mathbb{E}_{x \sim \nu}(\mathbb{E}_{x' \sim \nu} [u_{\nu_{f}}(x,x')^2]))^{1/2} \\ &\leq \sup_{f \in \mathcal{F}} \sup_{x, x' \in \mathbb{S}^d}  |u_{\nu_{f}}(x,x')| \leq ((\beta C_1 + d)^2 C_2 + 2C_3(\beta C_1 + d)),
\end{align}
and $\sqrt{\mathbb{E}_{x,x' \sim \nu} [k(x,x')^2]} \leq C_2$.
\end{proof}

\begin{lemma} \label{lem:approx_derivative2}
For a function $g : \mathbb{S}^d \rightarrow \R$, we define the partial derivative $\partial_i g : \mathbb{S}^d \rightarrow \R$ as the restriction to $\mathbb{S}^d$ of the partial derivative of the polynomial power series extension of $g$ to $\mathbb{R}^n$ (i.e. the extension of a spherical harmonic to $\mathbb{R}^n$ is the polynomial whose restriction to $\mathbb{S}^d$ is equal to the spherical harmonic (\citet{atkinson2012spherical}, Definition 2.7)). We denote by $\partial g = (\partial_i g)_{i=1}^{d+1}$ the vector of partial derivatives of $g$. The Riemannian gradient
$\nabla g : \mathbb{S}^d \rightarrow \R^{d+1}$, which is intrinsic (does not depend on the extension chosen), fulfills
\begin{align}
\nabla g(x) = (\nabla_i g(x))_{i=1}^{d+1} := \left( \partial_i g(x) - \sum_{i=1}^{d+1} \partial_j g(x) x_j x_i \right)_{i=1}^{d+1}. 
\end{align}
That is, $\nabla g(x)$ is the projection of $\partial g(x)$ to the tangent space of $\mathbb{S}^d$ at $x$.

For $\delta$ greater than a constant depending only on $d$,
for any function $g : \mathbb{S}^d \rightarrow \R$ such that for all $x, y \in \mathbb{S}^d$ we have $|g(x)| \leq \eta$ and $|g(x) - g(y)| \leq \eta \|x-y\|_2$, and $\|\nabla g(x)\|_2 \leq \eta$ and $\|\nabla g(x) - \nabla g(y)\|_2 \leq L \|x-y\|_2$, and $g$ is even, there exists $\hat{g} \in \mathcal{F}_2$ such that $\|\hat{g}\|_{\mathcal{F}_2} \leq \delta$ and
\begin{align} \label{eq:l_infty_bound_gradients}
&\sup_{x \in \mathbb{S}^d} |\hat{g}(x) - g(x)| \leq C(d) \eta \left( \frac{\delta}{\eta} \right)^{-2/(d+1)} \log \left(\frac{\delta}{\eta} \right), \\
&\sup_{x \in \mathbb{S}^d} \|\nabla \hat{g}(x) - \nabla g(x) \|_2 \leq C(d) (L + \eta) \left( \frac{\delta}{\eta} \right)^{-2/(d+1)} \log \left(\frac{\delta}{\eta} \right),
\end{align}
where $C(d)$ are constants depending only on the dimension $d$. 
\end{lemma}

\begin{proof}
We will use some ideas and notation of the proof of Prop. 3 of \citet{bach2017breaking}. We can decompose $g(x) = \sum_{k \geq 0} g_k(x)$, where $g_k(x) = N(d,k) \int_{\mathbb{S}^d} g(y) P_k(\langle x, y \rangle) d\tau(y)$. $g_k$ is the $k$-th spherical harmonic of $g$ and $P_k$ is the $k$-th Legendre polynomial in dimension $d+1$. Analogously, for any $i$ between $1$ and $d+1$ we can decompose $\nabla_i g(x) = \sum_{k \geq 0} (\nabla_i g)_k(x)$, where $(\nabla_i g)_k(x) = N(d,k) \int_{\mathbb{S}^d} \nabla_i g(y) P_k(\langle x, y \rangle) d\tau(y)$. Define $\widetilde{\nabla_i g} : \mathbb{R}^{d+1} \rightarrow \R$ to be the spherical harmonic extension of $\nabla_i g$.

Like \citet{bach2017breaking}, we define $\hat{g}(x) = \int_{\mathbb{S}^d} \sigma(\langle \theta, x \rangle) \hat{h}(\theta) d\tau(\theta)$, where $\hat{h}(x) = \sum_{k, \lambda_k \neq 0} \lambda_k^{-1} r^k g_k(x)$ for some $r \in (0,1)$. Equivalently, $\hat{g}(x) = \sum_{k, \lambda_k \neq 0} r^k g_k(x)$. Since $g_k$ is a homogeneous polynomial of degree $k$ (\citet{atkinson2012spherical}, Definition 2.7), we have that $\hat{g}(x) = \sum_{k, \lambda_k \neq 0} g_k(r x) = g(r x)$.

With this choice of $\hat{g}$, the first equation of \eqref{eq:l_infty_bound_gradients} holds by Prop. 3 of \citet{bach2017breaking}.

Using this characterization of $\hat{g}$, by the chain rule we compute the Riemannian gradient 
\begin{align}
\begin{split}
\nabla \hat{g}(x) &= \partial \hat{g}(x) - \langle \partial \hat{g}(x), x \rangle x = \partial (g \circ (y \rightarrow r y))(x) - \langle \partial (g \circ (y \rightarrow r y))(x), x \rangle x = r \partial g(rx) - r \langle \partial g(rx), x \rangle x 
\end{split}
\end{align}
The polynomial power series extension $\widetilde{\nabla g}$ of $\nabla g$ is by definition equal to $\nabla g(x) = \partial g(x) - \langle \partial g(x), x \rangle x = \sum_{k \geq 0} (\partial g)_k(x) - \langle (\partial g)_k(x), x \rangle x$ for $x \in \mathbb{S}^d$. Since the terms of $\sum_{k \geq 0} (\partial g)_k(x) - \langle (\partial g)_k(x), x \rangle x$ are polynomials on $x$, this expression is equal to the polynomial power series of $\nabla g$ by uniqueness of the polynomial power series. Thus, for all $x \in \R^{d+1}$, 
\begin{align} \label{eq:harmonic_extension_nabla}
\widetilde{\nabla g}(x) = \sum_{k \geq 0} (\partial g)_k(x) - \langle (\partial g)_k(x), x \rangle x = \sum_{k \geq 0} \partial (g_k)(x) - \langle \partial (g_k)(x), x \rangle x = \partial g(x) - \langle \partial g(x), x \rangle x.
\end{align}
The second equality follows from \autoref{lem:spherical_decomp_derivative}, which states that $\partial (g_k) = (\partial g)_k$. Hence, by \eqref{eq:harmonic_extension_nabla}, we have $r \partial g(rx) - r \langle \partial g(rx), rx \rangle rx = r \widetilde{\nabla g}(rx) = r \sum_{k \geq 0} (\nabla g)_k(rx) = r \sum_{k \geq 0} r^k (\nabla g)_k(x)$.
Thus, in analogy with \cite{bach2017breaking}, we have
\begin{align}
\begin{split}
    &r \partial g(rx) - r \langle \partial g(rx), rx \rangle rx = r \sum_{k \geq 0} r^k (\nabla g)_k(x) = r \sum_{k \geq 0} r^k N(d,k) \int_{\mathbb{S}^d} \nabla g(y) P_k(\langle x, y \rangle) d\tau(y) \\ &= r  \int_{\mathbb{S}^d} \nabla g(y) \left(\sum_{k \geq 0} r^k N(d,k) P_k(\langle x, y \rangle) \right) d\tau(y) = r  \int_{\mathbb{S}^d} \nabla g(y) \frac{1-r^2}{(1+r^2 - 2r(\langle x, y \rangle))^{(d+1)/2}} d\tau(y).
\end{split}
\end{align}
Hence, keeping the analogy with \citet{bach2017breaking} (and \citet{bourgain1988projection}, Equation 2.13), we obtain that 
\begin{align}
\begin{split}
    &\left\|\nabla g(x) - r \partial g(rx) - r \langle \partial g(rx), rx \rangle rx \right\|_2 = \left\| \int_{\mathbb{S}^d} (\nabla g(x) - r\nabla g(y)) \frac{1-r^2}{(1+r^2 - 2r(\langle x, y \rangle))^{(d+1)/2}} d\tau(y) \right\|_2 \\ &\leq \int_{\mathbb{S}^d} \|\nabla g(x) - r\nabla g(y)\|_2 \frac{1-r^2}{(1+r^2 - 2r(\langle x, y \rangle))^{(d+1)/2}} d\tau(y) \\ &\leq \int_{\mathbb{S}^d} \|\nabla g(x) - \nabla g(y)\|_2 \frac{1-r^2}{(1+r^2 - 2r(\langle x, y \rangle))^{(d+1)/2}} d\tau(y) + (1-r)\int_{\mathbb{S}^d} \|\nabla g(y)\|_2 \frac{1-r^2}{(1+r^2 - 2r(\langle x, y \rangle))^{(d+1)/2}} d\tau(y) \\ &\leq C_2(d) (1-r) \text{Lip}(\nabla g) \int_0^1 \frac{t^{d}}{(1-r)^{d+1} + t^{d+1}} dt + (1-r)\int_{\mathbb{S}^d} \left(\sup_{x \in \mathbb{S}^d} \|\nabla g(x)\|_{2} \right) \left(\sum_{k \geq 0} r^k N(d,k) P_k(\langle x, y \rangle) \right) d\tau(y) \\ &\leq C_3(d) \text{Lip}(\nabla g) (1-r) \log \left(1/(1-r) \right) + (1-r) \left(\sup_{x \in \mathbb{S}^d} \|\nabla g(x)\|_{2} \right) \leq C_4(d)(1-r) (\eta+ L\log(1/(1-r)))
\end{split}    
\end{align}
In the last equality we have used that $\partial_i g$ is $L$-Lipschitz by assumption. And
\begin{align} 
\begin{split} \label{eq:l_infty_difference}
    \|\nabla g(x) - \nabla \hat{g}(x)\|_2^2 &= \|\nabla g(x) - r \partial g(rx) - r \langle \partial g(rx), x \rangle x\|_2^2 \\ &\leq \|\nabla g(x) - r \partial g(rx) - r \langle \partial g(rx), x \rangle x\|_2^2 + \|(1-r^2) r \langle \partial g(rx), x \rangle x\|_2^2 
    \\ &\leq \|\nabla g(x) - r \partial g(rx) - r \langle \partial g(rx), rx \rangle rx\|_1^2 \leq (C_4(d)(1-r) (\eta+ L\log(1/(1-r))))^2.
\end{split}
\end{align}
In the second equality we have used that $\nabla g(x) - \nabla \hat{g}(x)$ is orthogonal to $x$ (because it belongs to the tangent space at $x$), and the Pythagorean theorem.
As in \cite{bach2017breaking}, for $\delta > 0$ large enough the argument is concluded by taking $1-r = (C_1(d) \eta /\delta)^{2/(d+1)} \in (0,1)$, which means that the (square root of the) error in the right-hand side of \eqref{eq:l_infty_difference} is $C_4(d)  (C_1(d) \eta/\delta)^{2/(d+1)} \left(\eta + L\log (C_1(d) \eta/\delta)^{-2/(d+1)} \right) \leq C_5(d) (L+\eta) (\delta/\eta)^{-2/(d+1)} \log (\delta/\eta)$. 

Using that $g$ is $\eta$-Lipschitz, by the argument of \citet{bach2017breaking} we have that $\|\hat{h}\|_{L^2(\mathbb{S}^d)} \leq C_1(d) \eta (1-r)^{(-d-1)/2}$, where $C_1(d)$ is a constant that depends only on $d$ and consequently $\|\hat{g}\|_{\mathcal{F}_2} \leq C_1(d) \eta (1-r)^{(-d-1)/2}$.
And for our choice of $r$, this bound becomes $\|\hat{g}\|_{\mathcal{F}_2} \leq C_1(d) \eta ((C_1(d) \eta /\delta)^{2/(d+1)})^{(-d-1)/2} = \delta$.

\end{proof}

\begin{lemma} \label{lem:spherical_decomp_derivative}
For $g : \mathbb{S}^d \rightarrow \R$ with spherical harmonic decomposition $g(x) = \sum_{k \geq 0} g_k(x)$ and with partial derivative with spherical harmonic decomposition $\partial_i g(x) = \sum_{k \geq 0} (\partial_i g)_k(x)$, we have $(\partial_i g)_k(x) = \partial_i (g_k)(x)$.
\end{lemma}
\begin{proof}
Remark that the spherical harmonics on $\mathbb{S}^d$ can be characterized as the restrictions of the homogeneous harmonic polynomials on $\mathbb{R}^{d+1}$ (\citet{atkinson2012spherical}, Definition 2.7). $k$-th degree homogeneous polynomials are of sums of monomials of the form $\alpha_{i_1, \dots, i_r} x_1^{i_1} \cdot \cdots \cdot x_r^{i_r}$, where $\sum_{l=1}^r i_l = k$, and harmonic polynomials are those such that $\Delta p = \sum_{i=1}^{d+1} \frac{\partial^2 p}{\partial x_i^2} = 0$. Thus, for all $k \geq 0$, $g_k$ can be seen as the restrictions to $\mathbb{S}^d$ of homogeneous harmonic polynomials of degree $k$. 

Notice that the $i$-th partial derivative of a homogeneous harmonic polynomial $p$ of degree $k$ is a homogeneous harmonic polynomial of degree $k-1$. That is because by commutation of partial derivatives, we have 
\begin{align}
    \Delta (\partial_i p) = \sum_{j=1}^{d+1} \partial_{jj} \partial_i p = \sum_{j=1}^{d+1} \partial_i \partial_{jj} p = \partial_i (\Delta p) = 0. 
\end{align}
Thus, $\partial_i (g_k)$ are homogeneous harmonic polynomials of degree $k-1$, which means that their restrictions to $\mathbb{S}^d$ are spherical harmonics. Since $\partial_i g(x) = \partial_i (\sum_{k \geq 0} g_k(x)) = \sum_{k \geq 0} \partial_i (g_k)(x)$ and the spherical harmonic decomposition is unique, $\partial_i (g_k)$ must be precisely the spherical harmonic components of $\partial_i g$.
\end{proof}

\corsdsphere*

\begin{proof}
We will use \autoref{thm:coronesd}. Let $g : \mathbb{S}^d \rightarrow \R$ be defined as $g(x) = \sum_{j=1}^J g_j(x) = \sum_{j=1}^J \phi_j(U_j x)$, where $\phi_j : \{ x \in \R^{k+1} | \|x\|_2 \leq 1 \} \rightarrow \R$. 

Let $\hat{\phi}_j : \mathbb{S}^k \rightarrow \R$ be the restriction of $\phi_j$ to $\mathbb{S}^k$. By \autoref{lem:approx_derivative2}, 
there exists $\hat{\psi}_j : \mathbb{S}^k \rightarrow \R$ such that $\hat{\psi}_j \in \mathcal{F}_2$ and $\|\hat{\psi}_j\|_{\mathcal{F}_2} \leq \beta/J$, and 
\begin{align} \label{eq:psi_phi_2}
    &\sup_{x \in \mathbb{S}^d} | \hat{\phi}_j(x) - \hat{\psi}_j(x)| \leq C(k) (L+\eta) \left( \frac{\beta}{\eta J} \right)^{-2/(k+1)} \log \left(\frac{\beta}{\eta J} \right), \\
    &\sup_{x \in \mathbb{S}^d} \|\nabla \hat{\phi}_j(x) - \nabla \hat{\psi}_j(x)\|_2 \leq C(k) (L+\eta) \left( \frac{\beta}{\eta J} \right)^{-2/(k+1)} \log \left(\frac{\beta}{\eta J} \right)
\end{align}
Moreover, if we denote by $\psi_j : \{ x \in \R^{k+1} | \|x\|_2 \leq 1 \} \rightarrow \R$ the 1-homogeneous extension of $\hat{\psi}_j$, we can write the (Euclidean) gradient of $\psi_j$ at the point $r x$ (with $r \in [0,1]$, $x \in \mathbb{S}^d$) in terms of the (Riemannian) gradient of $\hat{\psi}_j$ at $x$:
\begin{align}
    \nabla \psi_j (r x) = r \nabla \hat{\psi}_j(x) + \hat{\psi}_j(x)
\end{align}
Thus, by \autoref{eq:psi_phi_2}, and renaming $C(k)$,
\begin{align}
    \sup_{\|x\|_2 \leq 1} \|\nabla \phi_j(x) - \nabla \psi_j(x)\|_2 &\leq \sup_{x \in \mathbb{S}^d} \|\nabla \hat{\phi}_j(x) - \nabla \hat{\psi}_j(x)\|_2 + \sup_{x \in \mathbb{S}^d} |\hat{\phi}_j(x) - \hat{\psi}_j(x)| \\ &\leq C(k) (L+\eta) \left( \frac{\beta}{\eta J} \right)^{-2/(k+1)} \log \left(\frac{\beta}{\eta J} \right)
\end{align}
Hence, if we define $\tilde{g}_j : \mathbb{S}^d \rightarrow \R$ as $\tilde{g}_j(x) := \psi_j(U_j x)$, we check that $\tilde{g}_j$ belongs to $\mathcal{F}_1$: if $\hat{\psi}_j$
is such that $\forall x \in \mathbb{S}^d$, $\hat{\psi}_j(x) = \int_{\mathbb{S}^k} \sigma(\langle \theta, x \rangle) d\gamma(\theta)$, then $\psi_j(x) = \int_{\mathbb{S}^k} \sigma(\langle \theta, x \rangle) d\gamma(\theta)$ when $\|x\|_2 \leq 1$, and
\begin{align}
    \tilde{g}_j(x) = \psi_j(U_j x) = \int_{\mathbb{S}^k} \sigma(\langle \theta, U_j x \rangle) d\gamma(\theta) = \int_{\mathbb{S}^k} \sigma(\langle  U_j^{\top} \theta, x \rangle) d\gamma(\theta) = \int_{\mathbb{S}^d} \sigma(\langle \theta', x \rangle) d\gamma'(\theta')
\end{align}
This also shows that $\tilde{g}_j$ has $\mathcal{F}_1$ norm $\|\tilde{g}_j\|_{\mathcal{F}_1} \leq \|\hat{\psi}_j\|_{\mathcal{F}_2} \leq \beta/J$, which would mean that $ \tilde{g} = \sum_{j=1}^J \tilde{g}_j \in \mathcal{F}_1$ and $\|\tilde{g}\|_{\mathcal{F}_1} \leq \beta$. 
Moreover,
\begin{align}
\begin{split} \label{eq:g_f_star_2}
    \sup_{x \in \mathbb{S}^d} \|\nabla \tilde{g}_j(x) - \nabla g_j(x)\|_{2} &= \sup_{x \in \mathbb{S}^d} \|\nabla (\psi_j \circ U_j) (x) - \nabla (\phi_j \circ U_j) (x)\|_2 \\ &\leq \sup_{x \in \mathbb{S}^d} \|U_j^{\top}(\nabla \psi_j (U_j x) - \nabla \phi_j (U_j x))\|_2 = \sup_{x \in \mathbb{S}^d} \|\nabla \psi_j (U_j x) - \nabla \phi_j (U_j x)\|_2 \\ &\leq \sup_{\|y\|_2 \leq 1 } \|\nabla \psi_j(y) - \nabla \phi_j(y)\|_2 \leq C(k)(L + \eta)\left( \frac{ \beta}{\eta J} \right)^{-2/(k+1)} \log \left(\frac{ \beta}{\eta J} \right)
\end{split}
\end{align}
The first inequality holds because the Riemannian gradient is the orthogonal projection of the Euclidean gradient of the extension, and orthogonal projections are 1-Lipschitz. The following equality holds because $U_j$ has orthonormal rows. The second inequality holds because for all $x \in \mathbb{S}^d$, $\|U_j x\|_2 \leq \|x\|_2 = 1$ by the fact that $U_j$ has orthonormal rows, and the third inequality holds by \eqref{eq:psi_phi_2}.

Thus, for part (i), we have
\begin{align}
\begin{split} \label{eq:intermediate_ineq}
    &\inf_{f \in \mathcal{B}_{\mathcal{F}_1}} \mathbb{E}_{\nu} \bigg[ \bigg\|  -\beta \nabla f(x) - \nabla \log \left(\frac{d\nu}{d\tau}(x) \right) \bigg\|_2 \bigg] \leq \inf_{f \in \mathcal{B}_{\mathcal{F}_1}} \sup_{x \in \mathbb{S}^d} \bigg\|\beta \nabla f(x) - \nabla \log \left(\frac{d\nu}{d\tau}(x) \right)\bigg\|_{2} \\ &\leq \sup_{x \in \mathbb{S}^d} \|\nabla \tilde{g}(x) - \nabla g(x)\|_{2} \leq \sum_{j=1}^J \sup_{x \in \mathbb{S}^d} \|\nabla \tilde{g}_j - \nabla g_j\|_{2} \leq C(k) J(L + \eta)\left( \frac{\beta}{\eta J} \right)^{-2/(k+1)} \log \left(\frac{\beta}{\eta J} \right)
\end{split}
\end{align}
Plugging this into \autoref{thm:coronesd} and using that $\sup_{f \in \mathcal{B}_{\mathcal{F}_1}} \{ \|\partial_i f\|_{\infty} | 1 \leq i \leq d+1 \} \leq 1$, 
we obtain
\begin{align}
\begin{split}
    D_{\text{KL}}(\nu || \hat{\nu}) &\leq \frac{4 \sqrt{d+1} (\beta + C_2 \sqrt{d+1} + d)}{\sqrt{n}} +  2(\beta + d + 1)\sqrt{\frac{(d+1)\log(\frac{d+1}{\delta})}{2n}} \\ &+ C(k) J(L + \eta)\left( \frac{\beta}{\eta J} \right)^{-2/(k+1)} \log \left(\frac{\beta}{\eta J} \right).
\end{split}
\end{align}
If we optimize this bound with respect to $\beta$ as in the proof of \autoref{cor:statf1}, we obtain
\begin{align}
\begin{split}
    &\frac{4(C_2 \sqrt{d+1} + d)}{\sqrt{n}} + 2(d+1)\sqrt{\frac{\log(\frac{d+1}{\delta})}{2n}} +  \left(\frac{2B}{k+1} \right)^{\frac{k+1}{k+3}} \left( \frac{A}{\sqrt{n}}\right)^{\frac{2}{k+3}} \\ &+ B^{\frac{k+1}{k+3}} \left( \frac{A(k+1)}{2 \sqrt{n}}\right)^{\frac{2}{k+3}} \log \left( \frac{1}{\eta J} \left( \frac{2B \sqrt{n}}{A (k+1)} \right)^{\frac{k+1}{k+3}}\right),
\end{split}
\end{align}
and the optimal $\beta$ is $\left( 2B \sqrt{n}/(A (k+1)) \right)^{\frac{k+1}{k+3}}$, where
\begin{align}
    A = 4\sqrt{d+1} + \sqrt{2(d+1)\log((d+1)/\delta)}, \quad B = C(k) J (L+\eta) (\eta J)^{\frac{2}{k+1}}.
\end{align}
For part (ii), we plug 
\begin{align}
\begin{split} \label{eq:intermediate_ineq2}
    &\inf_{f \in \mathcal{B}_{\mathcal{F}_1}} \mathbb{E}_{\nu} \bigg[ \bigg\|  -\beta \nabla f(x) - \nabla \log \left(\frac{d\nu}{d\tau}(x) \right) \bigg\|_2^2 \bigg] \leq \inf_{f \in \mathcal{B}_{\mathcal{F}_1}} \sup_{x \in \mathbb{S}^d} \bigg\|\beta \nabla f(x) - \nabla \log \left(\frac{d\nu}{d\tau}(x) \right)\bigg\|_{2}^2 \\ &\leq \sup_{x \in \mathbb{S}^d} \|\nabla \tilde{g}(x) - \nabla g(x)\|_{2}^2 \leq \left(\sum_{j=1}^J \sup_{x \in \mathbb{S}^d} \|\nabla \tilde{g}_j - \nabla g_j\|_{2} \right)^2 \leq \left(C(k) J(L + \eta)\left( \frac{\beta}{\eta J} \right)^{-2/(k+1)} \log \left(\frac{\beta}{\eta J} \right) \right)^2
\end{split}
\end{align}
into \autoref{thm:ksdsecond}, and we obtain (using the notation of \autoref{thm:ksdsecond}) that with probability at least $1-\delta$,
\begin{align}
\begin{split}
    KSD(\nu, \hat{\nu}) &\leq \frac{2}{\sqrt{\delta n}} ((\beta C_1 + d)^2 C_2 + 2C_3(\beta C_1 + d) + C_4) + C_2 \left(C(k) J(L + \eta)\left( \frac{\beta}{\eta J} \right)^{-2/(k+1)} \log \left(\frac{\beta}{\eta J} \right) \right)^2 \\ &\leq \frac{2}{\sqrt{\delta n}} \left((\beta C_1 + d + C_3)C_2 + \bigg|2C_3(d-C_2) + C_4 - \frac{C_3^2}{d^2}C_2\bigg| \right)^2 \\ &+ C_2 \left( C(k) J(L + \eta)\left( \frac{\beta}{\eta J} \right)^{-2/(k+1)} \log \left(\frac{\beta}{\eta J} \right) \right)^2 = \frac{(A\beta + B)^2}{\sqrt{n}} + D^2 \beta^{-4/(k+1)} \log \left(\frac{\beta}{\eta J} \right)^2,
\end{split}
\end{align}
where $A, B, D$ are defined appropriately. If we set $\beta$ to minimize $\frac{A \beta + B}{n^{1/4}} + D \beta^{-2/(k+1)}$, we obtain $\beta = \left( \frac{2 n^{1/4} D}{A (k+1)} \right)^{\frac{k+1}{k+3}}$, and the right-hand side becomes
\begin{align}
    \left( A^{\frac{2}{k+3}} \left( \frac{2 D}{k+1} \right)^{\frac{k+1}{k+3}} n^{-\frac{1}{2(k+3)}} + \frac{B}{n^{1/4}} \right)^2  + D^2  \left( \frac{A (k+1)}{2 D} \right)^{\frac{4}{k+3}} n^{-\frac{1}{k+3}} \log \left(\frac{1}{\eta J} \left( \frac{2 n^{1/4} D}{A (k+1)} \right)^{\frac{k+1}{k+3}} \right)^2.
\end{align}
\end{proof}

%% file: convergence.tex
\subsection{$\mathcal{F}_1$ EBMs dynamics}

For a (Fr\'echet-) differentiable functional $F : \mathcal{P}(\R^{d+2}) \rightarrow \R$, the Wasserstein gradient flow $(\mu_t)_{t \geq 0}$ of $F$ is the generalization of gradient flows to the metric space $\mathcal{P}(\R^{d+2})$ endowed with the Wasserstein distance $W_2^2(\mu_1, \mu_2) := \inf_{\pi \in \Pi(\mu_1, \mu_2)} \int_{\R^{d+2} \times \R^{d+2}} \|x-y\|_2^2 \ d\pi(x,y)$ \citep{ambrosio2008gradient}. One characterization of Wasserstein gradient flows is as the pushforward $\mu_t = (\Phi_t)_{\#} \mu_0$ of the initial measure $\mu_0$ by the evolution operator $\Phi_t$ which maps initial conditions $(w_0, \theta_0)$ to the solution at time $t$ of the ODE:
\begin{align}
    \frac{d(w,\theta)}{dt} = -\nabla \left(\frac{\delta}{\delta \mu} F(\mu_t)\right) (w,\theta), 
\end{align}
where $\frac{\delta}{\delta \mu} F(\mu) : \R^{d+2} \rightarrow \R$ is the Fr\'echet differential or first variation of $F$ at $\mu$. 

For any $m > 0$, we define the $m$-particle gradient flow $t \rightarrow u_m(t) = ((w^{(i)}_t, \theta^{(i)}_t))_{i=1}^{m}$ as the solution of the ODE
\begin{align}
    \frac{d(w^{(i)}_t, \theta^{(i)}_t)}{dt} = - \nabla \left(\frac{\delta}{\delta \mu} F(\mu_{m,t})\right) (w^{(i)}_t, \theta^{(i)}_t),
\end{align}
where $\mu_{m,t} = \frac{1}{m}\sum_{i=1}^m \delta_{(w^{(i)}_t, \theta^{(i)}_t)}$. For the functional $F$ defined in \eqref{eq:penalized_f1_R}, we have that $\nabla \left(\frac{\delta}{\delta \mu} F(\mu_{m,t})\right) (w^{(i)}_t, \theta^{(i)}_t)$ is equal to $\langle dR(\frac{1}{m} \sum_{j=1}^m \Phi(w^{(j)}_t,\theta^{(j)}_t)), \nabla \Phi(w^{(i)}_t, \theta^{(i)}_t) \rangle + \lambda (w^{(i)}_t, \theta^{(i)}_t)$, which is equal to $m$ times the gradient of the function $G((w^{(i)}, \theta^{(i)})_{i=1}^{m}) := F(\frac{1}{m} \sum_{j=1}^m \Phi(w^{(j)},\theta^{(j)}))$ with respect to $(w^{(i)}, \theta^{(i)})$. Thus, $u_m(t)$ is simply the gradient flow of $G$ (up to a time reparametrization).

\begin{restatable}{thm}{thm:chizatbach}
\label{thm:chizatbach} [\cite{chizat2018global}, Thm. 3.3; informal]
Let $R$ be a convex differentiable loss defined on a Hilbert space with differential $dR$ Lipschitz on bounded sets and bounded on sublevel sets which satisfies a technical Sard-type regularity assumption. Let ${(\mu_t)}_{t \geq 0}$ be a Wasserstein gradient flow corresponding to $F$ in \eqref{eq:penalized_f1_R}, such that the support of $\mu_0$ is contained in $B(0, r_b)$ and separates the spheres $r_a \mathbb{S}^{d+1}$ and $r_b \mathbb{S}^{d+1}$ for some $0 < r_a < r_b$. If ${(\mu_t)}_t$ converges to $\mu_{\infty}$ in $W_2$, then $\mu_{\infty}$ is a global minimizer of $F$. Moreover, if $(\mu_{m,t})_{t \geq 0}$ is the empirical measure of $(u_m(t))_{t \geq 0}$ and $\mu_{m,0} \rightarrow \mu_0$ weakly, we have $\lim_{t, m \rightarrow \infty} F(\mu_{m,t}) = \lim_{m,t \rightarrow \infty} F(\mu_{m,t}) = F(\mu_{\infty})$.
\end{restatable}

\autoref{thm:chizatbach} states that when the number $m$ of particles (read neurons) goes to infinity, the function value of the gradient flow of the function $G((w^{(i)}, \theta^{(i)})_{i=1}^{m})$ converges to a global optimum of $F$ over $\mathcal{P}(\R^{d+2})$. Remark that Algorithm \ref{alg:F1_ebm} corresponds to the gradient descent algorithm on $G((w^{(i)}, \theta^{(i)})_{i=1}^{m})$ with noisy gradient estimates. Thus, in the small stepsize and exact gradient limits, the iterates of Algorithm \ref{alg:F1_ebm} approximate the gradient flow of $G((w^{(i)}, \theta^{(i)})_{i=1}^{m})$. This reasoning provides an informal justification that Algorithm \ref{alg:F1_ebm} should have a sensible behavior in the appropriate limits.

\begin{observation} \label{obs:chizatconditions}
While \autoref{thm:chizatbach} assumes that $R$ is defined on a Hilbert space, this assumption is not convenient in our case because $R(f) = \frac{1}{n} \sum_{i=1}^n f(x_i) + \log\left(\int_{K} e^{-f(x)} dx \right)$ is not well defined on $L^2(\R^d)$, as it involves pointwise evaluations. However, following the argument of \cite{chizat2018global}, up to the technical Sard-type regularity assumption, it suffices to show that $F(\mu)$ is a convex differentiable loss with a first variation $\frac{\delta}{\delta \mu} F(\mu)$ such that
\begin{itemize}
\item The restriction of $\frac{\delta}{\delta \mu} F(\mu)$ to $(w,\theta) \in \mathbb{S}^{d+1}$ fulfills $\bigg\|\frac{\delta}{\delta \mu} F(\mu)(\cdot) - \frac{\delta}{\delta \mu} F(\mu')(\cdot) \bigg\|_{C^1(\mathbb{S}^{d+1})} \leq L \|h_2(\mu)-h_2(\mu')\|_{BL}$,
where $h_2 : \mathcal{M}(\R^{d+2}) \rightarrow \mathcal{M}(\mathbb{S}^{d+1})$ is defined as $\int_{\mathbb{S}^{d+1}} \phi(x) \ dh_2(\mu)(x) = \int_{\R^{d+2}} |y|^2 \phi(y/|y|) \ d\mu(y)$ and $\| \cdot \|_{BL}$ is the bounded Lipschitz norm. 
\item The restriction of $\frac{\delta}{\delta \mu} F(\mu)$ to $(w,\theta) \in \mathbb{S}^{d+1}$ is bounded on sublevel sets of $F(\mu)$ in $L^{\infty}$ norm. 
\end{itemize}
\end{observation}

To apply \autoref{thm:chizatbach}, we must check that the two statements in \autoref{obs:chizatconditions} hold. Since for the maximum likelihood loss we have:
\begin{align}
    \frac{\delta}{\delta \mu} F(\mu) (w,\theta) = \frac{1}{n} \sum_{i=1}^n \Phi(w,\theta)(x_i) - \frac{\int_K \Phi(w,\theta)(x) \exp\left(- \int_{\R^{d+2}} \Phi(w',\theta')(x) d\mu(w',\theta') \right) \ d\tau(x)}{\int_K \exp\left(- \int_{\R^{d+2}} \Phi(w',\theta')(x) d\mu(w',\theta') \right) \ d\tau(x)} + \lambda(w^2 + \|\theta\|_2^2),
\end{align}
we obtain that for all $(w,\theta) \in \R^{d+2}$ and $\mu, \mu' \in \mathcal{P}(\R^{d+2})$,
\begin{align}
\begin{split}
    &\frac{\delta}{\delta \mu} F(\mu) (w,\theta) - \frac{\delta}{\delta \mu} F(\mu') (w,\theta) \\ = &- \frac{\int_K \Phi(w,\theta)(x) \exp\left(- \int_{\R^{d+2}} \Phi(w',\theta')(x) d\mu(w',\theta') \right) \ d\tau(x)}{\int_K \exp\left(- \int_{\R^{d+2}} \Phi(w',\theta')(x) d\mu(w',\theta') \right) \ d\tau(x)} \\ &+ \frac{\int_K \Phi(w,\theta)(x) \exp\left(- \int_{\R^{d+2}} \Phi(w',\theta')(x) d\mu'(w',\theta') \right) \ d\tau(x)}{\int_K \exp\left(- \int_{\R^{d+2}} \Phi(w',\theta')(x) d\mu'(w',\theta') \right) \ d\tau(x)}
    \\ &= \frac{\int_K \Phi(w,\theta)(x) \exp\left(- \int_{\R^{d+2}} \Phi(w',\theta')(x) d\mu_t(w',\theta') \right) \left(\int_{\R^{d+2}} \Phi(w'',\theta'')(x) d(\mu-\mu')(w'',\theta'') \right) \ d\tau(x)}{\int_K \exp\left(- \int_{\R^{d+2}} \Phi(w',\theta')(x) d\mu_t(w',\theta') \right) \ d\tau(x)} \\ &- \frac{\int_K \Phi(w,\theta)(x) \exp\left(- \int_{\R^{d+2}} \Phi(w',\theta')(x) d\mu_t(w',\theta') \right) \ d\tau(x)}{ \int_K \exp\left(- \int_{\R^{d+2}} \Phi(w',\theta')(x) d\mu_t(w',\theta') \right) \ d\tau(x)} \\ &\cdot \frac{\int_K \int_{\R^{d+2}} \Phi(w'',\theta'')(x) d(\mu-\mu')(w'',\theta'') \exp\left(- \int_{\R^{d+2}} \Phi(w',\theta')(x) d\mu_t(w',\theta') \right) \ d\tau(x)}{ \int_K \exp\left(- \int_{\R^{d+2}} \Phi(w',\theta')(x) d\mu_t(w',\theta') \right) \ d\tau(x)}
\end{split}
\end{align}
where $\mu_t = t \mu + (1-t) \mu'$. For $(w,\theta) \in \mathbb{S}^{d+1}$ and for all $x \in K$, we have $|\Phi(w,\theta)(x)| = |w \sigma(\langle \theta, x \rangle)| \leq \text{diam}(K)/2$ and $\|\nabla_{(w,\theta)} \Phi(w,\theta)(x)\|_2 = \|(\sigma(\langle \theta, x \rangle), w \mathds{1}(\langle \theta, x \rangle))\|_2 \leq \sqrt{\text{diam}(K)^2 +1}$, which means that $\|\Phi(\cdot)(x) \|_{C_1} \leq \sqrt{\text{diam}(K)^2 +1}$. Moreover, for $x \in K$, $\int_{\R^{d+2}} \Phi(w'',\theta'')(x) d(\mu-\mu')(w'',\theta'') = \int_{\mathbb{S}^{d+1}} \Phi(w'',\theta'')(x) d(h_2(\mu)-h_2(\mu'))(w'',\theta'') \leq \text{diam}(K) \|h_2(\mu)-h_2(\mu')\|_{BL}/2$. Hence,
\begin{align}
\begin{split}
    \bigg\|\frac{\delta}{\delta \mu} F(\mu)(\cdot) - \frac{\delta}{\delta \mu} F(\mu')(\cdot) \bigg\|_{C^1(\mathbb{S}^{d+1})} &\leq 2 \max_{x \in K} \|\Phi(\cdot)(x) \|_{C_1(\mathbb{S}^{d+1})} \bigg|\int_{\R^{d+2}} \Phi(w'',\theta'')(x) d(\mu-\mu')(w'',\theta'') \bigg| \\ &\leq \sqrt{\text{diam}(K)^2 +1} \text{diam}(K) \|h_2(\mu)-h_2(\mu')\|_{BL}
\end{split}
\end{align}
This shows the first point in \autoref{obs:chizatconditions}. For the second point, we have the following bound:
\begin{align}
\begin{split}
    &\bigg\|\frac{\delta}{\delta \mu} F(\mu) (w,\theta) \bigg\|_{\infty} \leq 2 \sup_{(w,\theta) \in \mathbb{S}^{d+1}} |\Phi(w,\theta)(x)| + \lambda
\end{split}
\end{align}

\subsection{$\mathcal{F}_2$ EBMs dynamics}
$\mathcal{F}_2$ is an RKHS with kernel $k(x,y) = \int_{\mathbb{S}^d} \sigma(\langle x, \theta \rangle) \sigma(\langle y, \theta \rangle) dp(\theta)$ and for the ReLU unit this kernel has a closed-form expression \citep{cho2009kernel}. Thus, one approach to optimize EBMs with energies over $\mathcal{F}_2$-balls is to apply the representer theorem and to write an optimizer $f \in \mathcal{F}_2$ as $f(\cdot) = \sum_{i=1}^n \alpha_i k(x_i, \cdot)$ for some $\alpha \in \R^{n}$, as well as $\|f\|_{\mathcal{F}_2}^2 = \sum_{i=1}^n \alpha_i \alpha_j k(x_i, x_j)$. Then, $f$ becomes a finite-dimensional linear function of $\alpha$, and thus any loss $F$ that is convex in $f$ is also convex on $\alpha$. However, this approach scales quadratically with the number of samples and in practical terms, it is quite far from the way neural networks are typically trained. 

The approach that we use to optimize EBMs over $\mathcal{F}_2$-balls is to sample random features $(\theta_i)_{i=1}^m$ on $\mathbb{S}^d$ from a probability measure with density $q(\cdot)$ and consider an approximate kernel $k_m(x,y) = \frac{1}{m} \sum_{i=1}^m \frac{1}{q(\theta_i)} \sigma(\langle x, \theta_i \rangle) \sigma(\langle y, \theta_i \rangle)$ \citep{rahimi2008random, bach2017onthe}. The functions in the finite dimensional RKHS $\mathcal{H}_m$ with kernel $k_m$ are of the form $h(x) = \sum_{i=1}^m  v_i (q(\theta_i) m)^{-1/2} \sigma(\langle x, \theta_i \rangle)$ with norm $\|h\|_{\mathcal{H}_m} = \|v\|_2$, or through a change of variables, $h(x) = \frac{1}{m}\sum_{i=1}^m w_i \sigma(\langle x, \theta_i \rangle)$ with norm $\|h\|_{\mathcal{H}_m} = \|(w_i\sqrt{q(\theta_i)})_{i=1}^m\|_2/\sqrt{m}$. 

Thus, learning a distribution with log-densities restricted in a ball of $\mathcal{H}_m$ reduces to learning the outer layer weights $(w_i)_{i=1}^n$. 
Namely, for $R$ as in \autoref{subsec:algorithms_f1}, we optimize the loss 
\begin{align}
G((w_i)_{i=1}^m) := R\left( \frac{1}{m} \sum_{i=1}^{m} w_i \sigma(\langle \theta_i, \cdot \rangle) \right) + \frac{\lambda}{m} \sum_{i=1}^{m} w_i^2 q(\theta_i),
\end{align}
which is convex.
The gradient flow for $G$ (with scaled gradient $m \nabla_i G((w_j)_{j=1}^m)$) is  
\begin{align} \label{eq:gradient_flow_h_m}
\frac{dw_i}{dt} = \left\langle dR\left( \frac{1}{m} \sum_{j=1}^{m} w_j \sigma(\langle \theta_j, \cdot \rangle) \right), \sigma(\langle \theta_i, \cdot \rangle) \right\rangle + 2 \lambda w_i q(\theta_i),
\end{align}
and we can approximate it by gradient descent, which converges to the optimum $(w_i^{\star})_{i=1}^m$ if the gradients are exact and the stepsize is well chosen. 

The connection between learning in $\mathcal{H}_m$ balls and learning in $\mathcal{F}_2$ balls is not straightforward. Applying Proposition 2 of \citet{bach2017onthe} and making use of the eigenvalue decay of the $\mathcal{F}_2$ kernel on $\mathbb{S}^d$ \citep{bach2017breaking}, for an appropriate choice of $q$ we have that for all $f \in \mathcal{B}_{\mathcal{F}_2}$, there exists $\hat{f} \in \mathcal{H}_m$ with $\|\hat{f}\|_{\mathcal{H}_m} \leq 2$ such that $\|f - \hat{f}\|_{L^2(p)} \leq O\left((m/\log(m))^{-(d+3)/2} \right)$. This $L^2$ error bound is sufficient to produce a quantitative result for least squares regression. However, for the three losses considered in this paper we would need bounds for $\|f - \hat{f}\|_{\infty}$ and $\|\nabla f - \nabla \hat{f}\|_{\infty}$, which do not seem to be available (\citet{bach2017onthe} does provide a bound on $\|f - \hat{f}\|_{\infty}$, but under the assumption that kernel eigenfunctions have a common $L^{\infty}$ norm bound, which does not hold for spherical harmonics in $\mathbb{S}^d$). 

Nonetheless, a mean-field qualitative approach analogous to the $\mathcal{F}_1$ case is still possible (see Proposition 2.6 of \citet{chizat2018global}). The learning objective in $\mathcal{F}_2$ can be written as
\begin{align}
    F(h) := R \left( \int_{\mathbb{S}^d} \sigma(\langle \theta, x \rangle) h(\theta) d\tilde{\tau}(\theta) \right) + \lambda \int_{\mathbb{S}^d} h^2(\theta) d\tilde{\tau}(\theta),
\end{align}
and the mean-field dynamics is
\begin{align}
\begin{split} \label{eq:f2_mean_field}
    \frac{d h_t(\theta)}{dt} = &- \left\langle dR\left( \int_{\mathbb{S}^d} \sigma(\langle \theta', \cdot \rangle) h_t(\theta') d\tilde{\tau}(\theta') \right), \sigma(\langle \theta, \cdot \rangle) \right\rangle - 2 \lambda h_t(\theta)
\end{split}
\end{align}
If we choose $q(\cdot) = 1$, we have that \eqref{eq:gradient_flow_h_m} is the $m$-particle approximation of \eqref{eq:f2_mean_field}.
Let $h^{\star}$ be the global minimizer of $F$, which is reached at a linear rate by \eqref{eq:f2_mean_field} because $F$ is strongly convex. Skipping through the details, the argument of Lemma C.15 of \citet{chizat2018global} could be adapted to yield:
\begin{align}
\begin{split}
    \lim_{t,m \rightarrow \infty} G((w_{t,i})_{i=1}^m) &= \lim_{m,t \rightarrow \infty} G((w_{t,i})_{i=1}^m) = \lim_{m \rightarrow \infty} G((w_{i}^{\star})_{i=1}^m) = F(h^{\star}). 
\end{split}
\end{align}

%% file: additional_exp_neutral.tex
In this section, we show plots corresponding to additional experiments. \autoref{fig:1n_ksd_f1sd} shows results for KSD and $\mathcal{F}_1$-SD training in the case $J=1, w_1^{*} = 10$. Compared to the plots for $J=1, w_1^{*} = 2$ shown in \autoref{fig:one_positive_neuron}, the separation between the $\mathcal{F}_1$ and $\mathcal{F}_2$ EBMs becomes much more apparent.
\autoref{fig:2n_ksd_f1sd} shows results for KSD and $\mathcal{F}_1$-SD training in the case $J=2, w_1^{*} = -2.5$. The separation between the $\mathcal{F}_1$ and $\mathcal{F}_2$ EBMs is smaller than in the case $J=2, w_1^{*} = -5$ shown in \autoref{fig:two_negative_neurons}.

\begin{figure}
    \begin{minipage}[c]{11.5cm}
    \includegraphics[width=.45\textwidth]{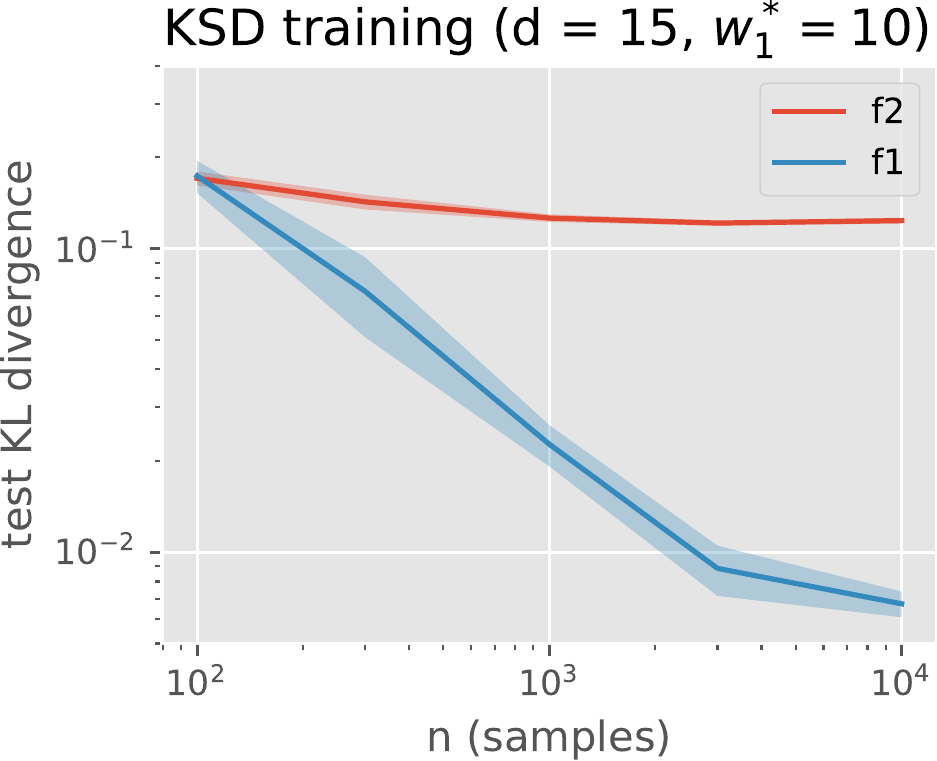}
    \includegraphics[width=.45\textwidth]{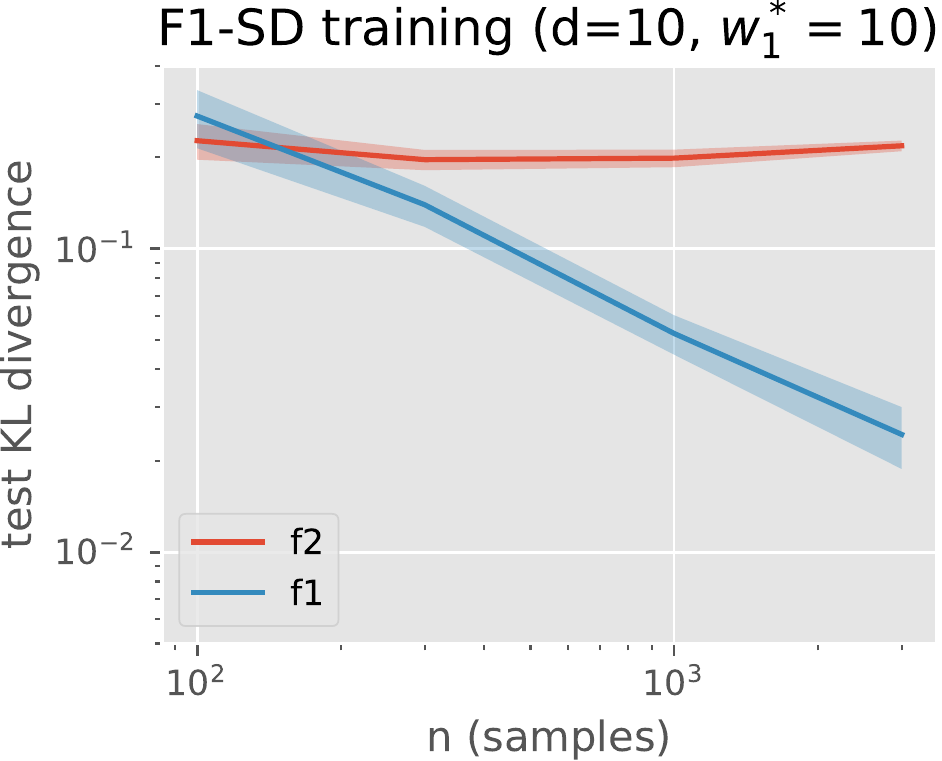} \\
    \includegraphics[width=.45\textwidth]{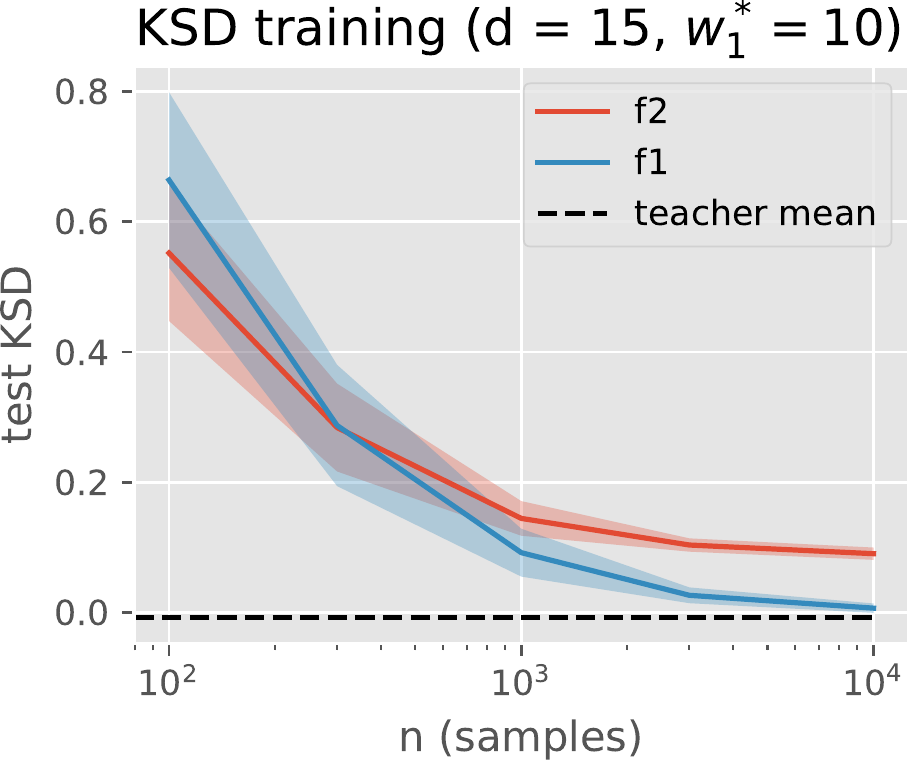}
    \includegraphics[width=.45\textwidth]{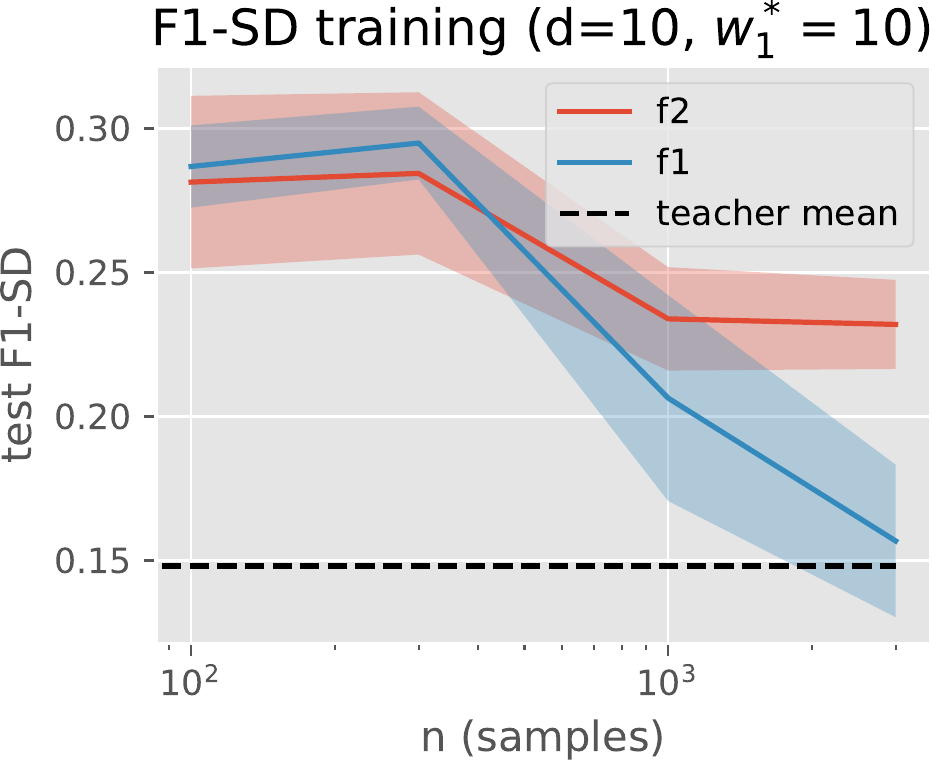}
    \end{minipage}
    \begin{minipage}[c]{5cm}
    \caption{(Left column) Test cross entropy and test KSD for model trained with KSD at different training sample sizes $n$, in $d=15$. The teacher model has one neuron of positive weight $w_1^{*} = 10$ and random position in the hypersphere. (Right column) Test cross entropy and test $\mathcal{F}_1$-SD for model trained with $\mathcal{F}_1$-SD at different training sample sizes $n$, in $d=10$. The teacher model has one neuron of positive weight $w_1^{*} = 10$  and random position in the hypersphere.}
    \label{fig:1n_ksd_f1sd}
    \end{minipage}
\end{figure}
\begin{figure}
    \begin{minipage}[c]{11.5cm}
    \includegraphics[width=.45\textwidth]{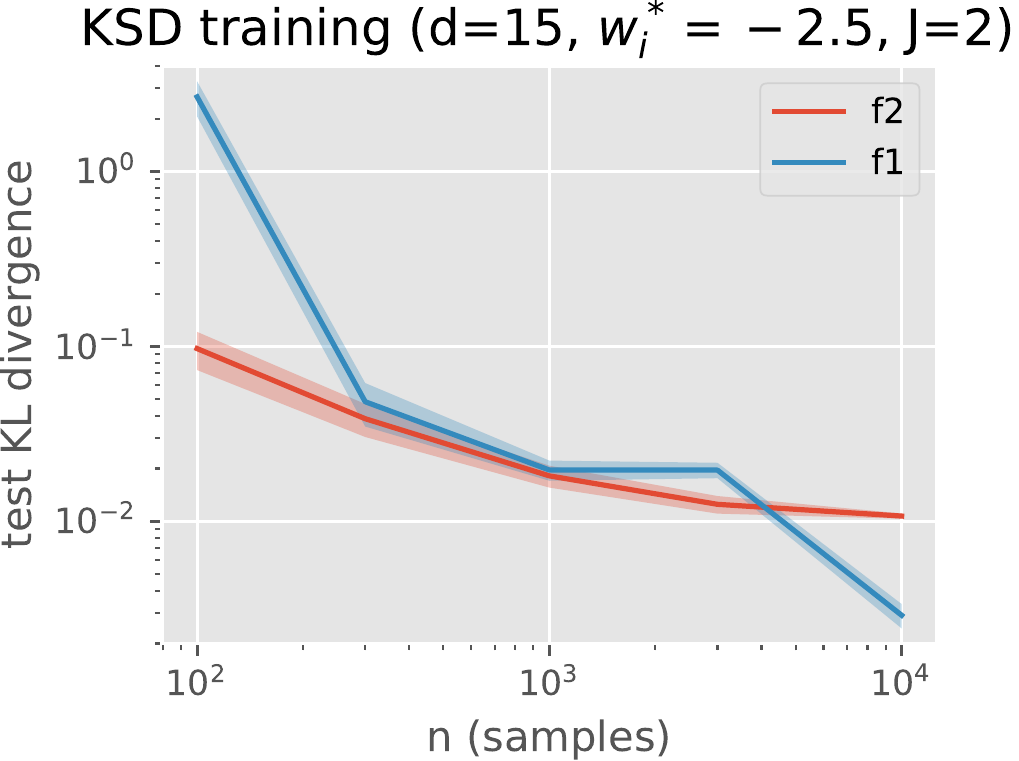}
    \includegraphics[width=.45\textwidth]{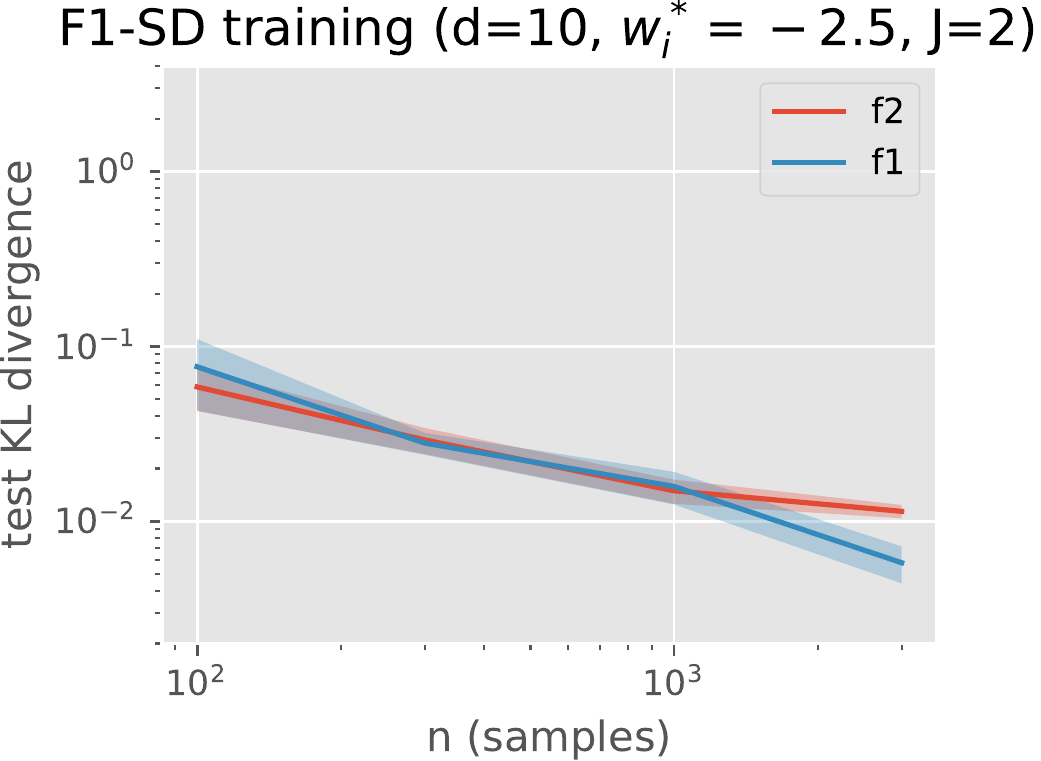} \\
    \includegraphics[width=.45\textwidth]{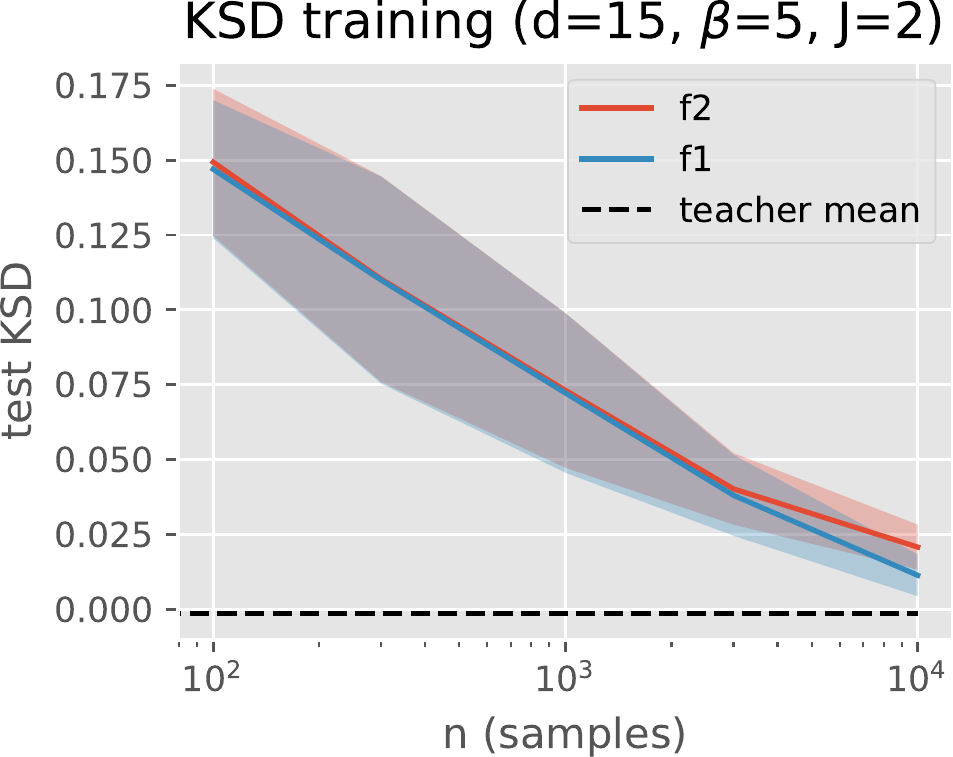}
    \includegraphics[width=.45\textwidth]{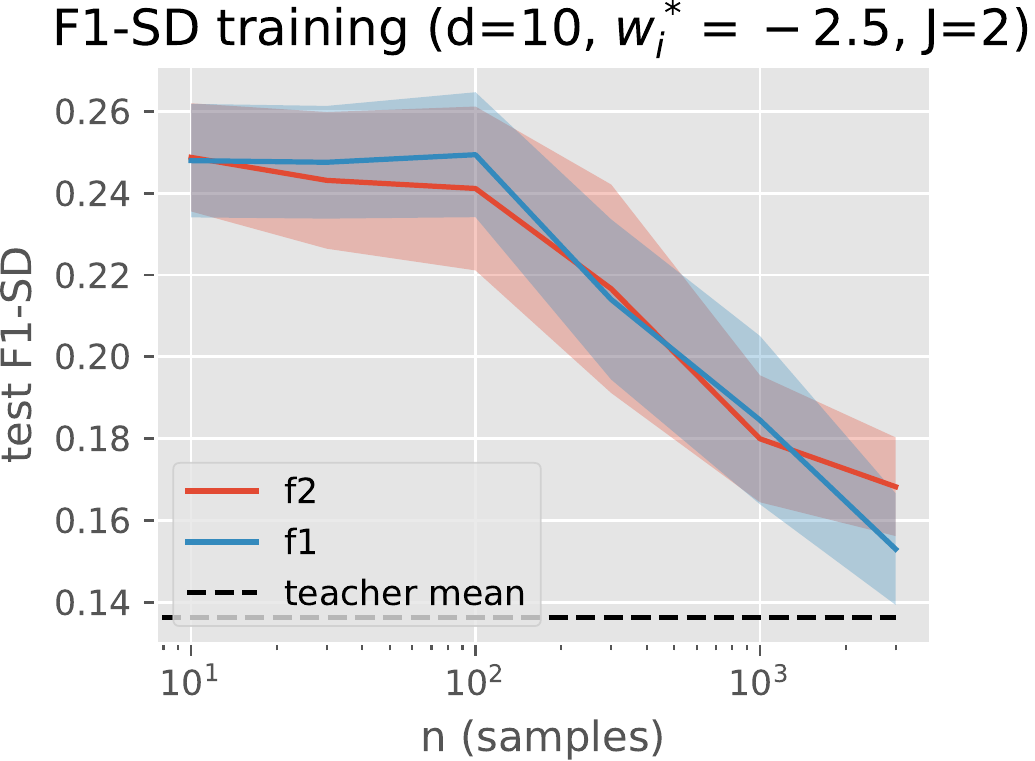}
    \end{minipage}
    \begin{minipage}[c]{5cm}
    \caption{(Left column) Test cross entropy and test KSD for model trained with KSD at different training sample sizes $n$, in $d=15$. The teacher model has two neurons of negative weights $w_1^{*}, w_2^{*} = -2.5$ and random positions in the hypersphere. (Right column) Test cross entropy and test $\mathcal{F}_1$-SD for model trained with $\mathcal{F}_1$-SD at different training sample sizes $n$, in $d=10$. The teacher model has two neurons of negative weights $w_1^{*}, w_2^{*} = -2.5$ and random positions in the hypersphere.}
    \label{fig:2n_ksd_f1sd}
    \end{minipage}
\end{figure}

%% file: duality.tex
In this section we present the dual problems of $\min_{f \in \mathcal{F}} H(\nu_n, \nu_{f})$ (i.e. problem \eqref{eq:hat_mu}) for the cases $\mathcal{F} = \mathcal{F}_1, \mathcal{F}_2$ (\autoref{subsec:unconstrained}), $\mathcal{F} = \mathcal{B}_{\mathcal{F}_1}(\beta)$ (\autoref{subsec:f1_constrained}) and $\mathcal{F} = \mathcal{B}_{\mathcal{F}_2}(\beta)$ (\autoref{subsec:f2_constrained}). The dual problems take the form of entropy maximization under hard constraints, $L^\infty$ and $L^2$ moment penalizations, respectively. The tools used involve a generalized minimax theorem and Fenchel duality, which was also used for results of the same flavor in finite dimensions (c.f. \cite{mohri2012foundations}). The proofs are in \autoref{sec:duality_proofs}.

\subsection{Duality for the unconstrained problem} \label{subsec:unconstrained}
Consider the following entropy maximization problem under generalized moment constraints:
\begin{align}
\begin{split} \label{eq:unconstrained}
    \min_{\nu \in \mathcal{P}(K)} \quad &\beta^{-1} D_{KL}(\nu || \tau) \\
    \text{s.t.} \quad &\forall \theta \in \mathbb{S}^d, \ \int \sigma(\langle \theta, x \rangle) \ d\nu(x) = \frac{1}{n} \sum_{i=1}^n \sigma(\langle \theta, x_i \rangle),
\end{split}
\end{align}
recalling that $\tau$ is the uniform probability measure over $K$ and letting $\beta > 0$ be arbitrary.
The constraints in this problem can be interpreted either (i) as an equality constraint in $\mathcal{C}(\mathbb{S}^d)$, i.e., the set of continuous functions on $\mathbb{S}^d$, or (ii) as an equality constraint in $L^2(\mathbb{S}^d)$, i.e., the set of square-integrable functions on $\mathbb{S}^d$. Each interpretation yields a different dual problem.

By the Riesz-Markov-Kakutani representation theorem, the set of signed Radon measures $\mathcal{M}(\mathbb{S}^d)$ can be seen as the continuous dual of $\mathcal{C}(\mathbb{S}^d)$. Hence, in the case (i), the Lagrangian for problem \eqref{eq:unconstrained} is $L_1 : \mathcal{M}(K) \times \mathcal{M}(\mathbb{S}^d) \times \mathcal{C}(K) \times \R \rightarrow \R$ defined as $L_1(\nu, \mu, g, \lambda) = \int \log \left( \frac{d\nu}{d\tau}(x) \right) d\nu(x) + \int \left( \int \sigma(\langle \theta, x \rangle) \ d\nu(x) - \frac{1}{n} \sum_{i=1}^n \sigma(\langle \theta, x_i \rangle) \right) \ d\mu(\theta) - \int g(x) d\nu(x) + \lambda \left( \int d\nu(x) - 1 \right)$, and the dual problem is
\begin{align}
\begin{split} \label{eq:unconstrained_f1_minimizer}
    \sup_{\mu \in \mathcal{M}(\mathbb{S}^d)} &-\frac{1}{n} \sum_{i=1}^n \int \sigma(\langle \theta, x_i \rangle) \ d\mu(\theta) - \frac{1}{\beta} \log\left(\int \exp \left(-\beta \int \sigma(\langle \theta, x \rangle) \ d\mu(\theta) \right) d\tau(x) \right)
\end{split}
\end{align}
which is equivalent to the MLE problem \eqref{eq:hat_mu} when $\mathcal{F} = \mathcal{F}_1$.

Let $\tilde{\tau}$ to denote the uniform probability measure over $\mathbb{S}^d$. In the case (ii), the Lagrangian for problem \eqref{eq:unconstrained} is $L_2 : \mathcal{M}(K) \times L^2(\mathbb{S}^d) \times \mathcal{C}(K) \times \R \rightarrow \R$ defined as $L_2(\nu, h, g, \lambda) = \int \log \left( \frac{d\nu}{d\tau}(x) \right) d\nu(x) + \int \left( \int \sigma(\langle \theta, x \rangle) \ d\nu(x) - \frac{1}{n} \sum_{i=1}^n \sigma(\langle \theta, x_i \rangle) \right) h(\theta)\ d\tilde{\tau}(\theta) - \int g(x) d\nu(x) + \lambda \left( \int d\nu(x) - 1 \right)$, and the dual problem is
\begin{align}
\begin{split} \label{eq:unconstrained_f2_minimizer}
    \sup_{h \in L^2(\mathbb{S}^d)} &-\frac{1}{n} \sum_{i=1}^n \int \sigma(\langle \theta, x_i \rangle) h(\theta)\ d\tilde{\tau}(\theta) - \frac{1}{\beta} \log\left(\int \exp \left(-\beta \int \sigma(\langle \theta, x \rangle) h(\theta)\ d\tilde{\tau}(\theta) \right) d\tau(x) \right)
\end{split}
\end{align}
which is equivalent to the MLE problem \eqref{eq:hat_mu} when $\mathcal{F} = \mathcal{F}_2$.

The following theorem shows that problems \eqref{eq:unconstrained}-\eqref{eq:unconstrained_f1_minimizer}-\eqref{eq:unconstrained_f2_minimizer} have the same optimal value.

\begin{restatable}{thm}{unconstrainedduality} \label{prop:unconstrainedduality}
Strong duality holds between the entropy maximization problem \eqref{eq:unconstrained} and each of the two dual problems \eqref{eq:unconstrained_f1_minimizer}-\eqref{eq:unconstrained_f2_minimizer}.
\end{restatable}

\subsection{Duality for the $\mathcal{F}_1$-ball constrained problem} \label{subsec:f1_constrained}
Using $\nu_n = \frac{1}{n} \sum_{i=1}^n \delta_{x_i}$ to denote the empirical measure, consider the following problem, which can be seen as an $L^\infty$-penalized version of \eqref{eq:unconstrained}:
\begin{align}
    \begin{split} \label{eq:primal_problem_f1}
    \min_{\nu \in \mathcal{P}(K)} \beta^{-1} D_{KL}(\nu||\tau) + \max_{\theta \in \mathbb{S}^d} \left| \int \sigma(\langle \theta, x \rangle) \ d(\nu-\nu_n)(x) \right|
\end{split}
\end{align}
As shown in \autoref{thm:strongdualityf1}, the dual of this problem is a modified version of \eqref{eq:unconstrained_f1_minimizer} in which $\mu$ is constrained to have TV norm bounded by 1:
\begin{align}
\begin{split} \label{eq:dual_problem_f1}
    \max_{\substack{\mu \in \mathcal{M}(\mathbb{S}^d) \\ |\mu|_{\text{TV}} \leq 1}} &-\frac{1}{n} \sum_{i=1}^n \int \sigma(\langle \theta, x_i \rangle) \ d\mu(\theta) - \frac{1}{\beta} \log\left(\int \exp \left(- \beta \int \sigma(\langle \theta, x \rangle) \ d\mu(\theta) \right) d\tau(x) \right)
\end{split}
\end{align}
Remark that by the definition of $\mathcal{F}_1$, the problem \eqref{eq:dual_problem_f1} is equivalent to MLE problem \eqref{eq:hat_mu} in the case $\mathcal{F} = \mathcal{B}_{\mathcal{F}_1}(\beta)$.
\begin{restatable}{thm}{strongdualityfone} \label{thm:strongdualityf1}
The problem \eqref{eq:dual_problem_f1} is the dual of the problem \eqref{eq:primal_problem_f1}, and strong duality holds. Moreover, the solution $\nu^{\star}$ of \eqref{eq:dual_problem_f1} is unique and its density satisfies
\begin{align}
    \frac{d\nu^{\star}}{d\tau}(x) = \frac{1}{Z_{\beta}}\exp \left( - \beta \int \sigma(\langle \theta, x \rangle) \ d\mu^{\star}(\theta) \right),
\end{align}
where $\mu^{\star}$ is a solution of \eqref{eq:primal_problem_f1} and $Z_{\beta}$ is a normalization constant.
\end{restatable}

\subsection{Duality for the $\mathcal{F}_2$-ball constrained problem} \label{subsec:f2_constrained}
The following problem can be seen as an $L^2$-penalized version of \eqref{eq:unconstrained}:
\begin{align}
    \begin{split} \label{eq:primal_problem_f2}
    \min_{\nu \in \mathcal{P}(K)} &\beta^{-1} D_{KL}(\nu||\tau) + \left( \int_{\mathbb{S}^d} \left( \int \sigma(\langle \theta, x \rangle) \ d(\nu - \nu_n)(x) \right)^2 d\tilde{\tau}(\theta) \right)^{1/2}
\end{split}
\end{align}
And as shown in \autoref{thm:strongdualityf2}, the dual of this problem is a modified version of \eqref{eq:unconstrained_f2_minimizer} in which $h$ is constrained to have $L^2$ norm bounded by 1:
\begin{align}
\begin{split} \label{eq:dual_problem_f2}
    &\max_{\substack{h \in L^2(\mathbb{S}^d) \\ \|h\|_{L^2} \leq 1}} -\frac{1}{n} \sum_{i=1}^n \int \sigma(\langle \theta, x_i \rangle) h(\theta)\ d\tilde{\tau}(\theta) - \frac{1}{\beta} \log\left(\int \exp \left(-\beta \int \sigma(\langle \theta, x \rangle) h(\theta)\ d\tilde{\tau}(\theta) \right) d\tau(x) \right)
\end{split}
\end{align}
Remark that by the definition of $\mathcal{F}_2$, the problem \eqref{eq:dual_problem_f1} is equivalent to MLE problem \eqref{eq:hat_mu} in the case $\mathcal{F} = \mathcal{B}_{\mathcal{F}_2}(\beta)$.
\begin{restatable}{thm}{strongdualityftwo} \label{thm:strongdualityf2}
The problem \eqref{eq:dual_problem_f2} is the dual of the problem \eqref{eq:primal_problem_f2}, and strong duality holds. Moreover, the solution $\nu^{\star}$ of \eqref{eq:dual_problem_f2} is unique and its density satisfies
\begin{align}
    \frac{d\nu^{\star}}{d\tau}(x) = \frac{1}{Z_{\beta}}\exp \left( - \beta \int \sigma(\langle \theta, x \rangle) \ h^{\star}(\theta) d\tilde{\tau}(\theta) \right),
\end{align}
where $h^{\star}$ is a solution of \eqref{eq:primal_problem_f2} and $Z_{\beta}$ is a normalization constant.
\end{restatable}

%% file: proofs_duality.tex
\begin{restatable}{thm}{kneser} [\citet{kneser1952surun}] 
\label{thm:kneser}
Let $X$ be a non-empty compact convex subset of a locally convex topological vector space space $E$ and $Y$ a non-empty convex subset of a locally convex topological vector space space $F$. Let the function $f : X \times Y \rightarrow \R$ be such that:
\begin{enumerate}[label=(\roman*)]
    \item For each $y \in Y$, the function $x \mapsto f(x,y)$ is upper semicontinuous and concave,
    \item For each $x \in X$, the function $y \mapsto f(x,y)$ is convex.
\end{enumerate}
Then we have
\begin{align}
    \sup_{x \in X} \inf_{y \in Y} f(x,y) = \inf_{y \in Y} \max_{x \in X} f(x,y).
\end{align}
\end{restatable}

\begin{lemma} \label{lem:diff_entropy}
The KL divergence $D_{KL}(\nu||\tau) = \int \log \left(\frac{d\nu}{d\tau} \right) d\nu$ is convex and lower semicontinuous in $\nu$.
\end{lemma}
\begin{proof}
See Theorem 1 of \cite{posner_random}. 
\end{proof}

\begin{observation} \label{obs:1}
Notice that for any functional $f : \mathcal{M}(K) \rightarrow \R$, 
we have 
\begin{align}
    &\min_{\nu \in \mathcal{P}(K)} f(\nu) = \min_{\nu \in \mathcal{P}(K)} f(\nu) - \int g(x) d\nu(x) + \lambda \left( \int d\nu(x) - 1 \right) \\ &= \min_{\nu \in \mathcal{M}(K)} \sup_{\lambda \in \R, g \in \mathcal{C}(K) : g \geq 0} f(\nu) - \int g(x) d\nu(x) + \lambda \left( \int d\nu(x) - 1 \right).
\end{align}
\end{observation}

\unconstrainedduality*

\begin{proof}
We start with \eqref{eq:unconstrained_f1_minimizer}. First, we prove that it is indeed the dual problem of \eqref{eq:unconstrained}. As stated in the main text, the problem \eqref{eq:unconstrained} admits a Lagrangian $L_1 : \mathcal{M}(K) \times \mathcal{M}(\mathbb{S}^d) \times \mathcal{C}(K) \times \R \rightarrow \R$ defined as 
\begin{align}
\begin{split} \label{eq:L_1_def}
    L_1(\nu, \mu, g, \lambda) &= \beta^{-1} \int \log \left( \frac{d\nu}{d\tau}(x) \right) d\nu(x) + \int \left( \int \sigma(\langle \theta, x \rangle) \ d\nu(x) - \frac{1}{n} \sum_{i=1}^n \sigma(\langle \theta, x_i \rangle) \right) \ d\mu(\theta) - \int g(x) d\nu(x) \\ &+ \lambda \left( \int d\nu(x) - 1 \right)
\end{split}
\end{align}
The Lagrange dual function is 
\begin{align}
\begin{split} \label{eq:lagrange_df}
    F_1(\mu, g, \lambda) = \inf_{\nu \in \mathcal{M}(K)} L_1(\nu, \mu, g, \lambda) &= - \beta^{-1} \int \exp \left( - \beta \left( \int \sigma(\langle \theta, x \rangle) \ d\mu(\theta) + g(x) - \lambda \right) -1 \right) \ d\tau(x) \\ &- \frac{1}{n} \sum_{i=1}^n \int \sigma(\langle \theta, x_i \rangle) \ d\mu(\theta)  - \lambda,
\end{split}
\end{align}
where we have used that at the optimal $\nu$, the first variation of $L_1$ w.r.t. $\nu$ must be zero:
\begin{align}
\begin{split}
    &0 = \frac{d}{d\nu} L_1(\nu, \mu, g, \lambda) = \beta^{-1} \log \left( \frac{d\nu}{d\tau}(x) \right) + \beta^{-1} + \int \sigma(\langle \theta, x \rangle) \ d\mu(\theta) - g(x) + \lambda, \\ 
    &\implies \frac{d\nu}{d\tau}(x) = \exp \left( - \beta \left( \int \sigma(\langle \theta, x \rangle) \ d\mu(\theta) + g(x) - \lambda \right) -1 \right).
\end{split}
\end{align}
The Lagrange dual problem is
\begin{align}
\begin{split} \label{eq:lagrange_dual_p}
    &\sup_{\mu, \lambda, g \geq 0} F_1(\mu, g, \lambda) 
    \\ &= \sup_{\mu, \lambda, g \geq 0}  - \beta^{-1} \int \exp \left( - \beta \left( \int \sigma(\langle \theta, x \rangle) \ d\mu(\theta) + g(x) - \lambda \right) -1 \right) \ d\tau(x) - \frac{1}{n} \sum_{i=1}^n \int \sigma(\langle \theta, x_i \rangle) \ d\mu(\theta)  - \lambda 
    \\ &= \sup_{\mu, \lambda}   - \beta^{-1} \int \exp \left( - \beta \left( \int \sigma(\langle \theta, x \rangle) \ d\mu(\theta) - \lambda \right) -1 \right) \ d\tau(x) - \frac{1}{n} \sum_{i=1}^n \int \sigma(\langle \theta, x_i \rangle) \ d\mu(\theta)  - \lambda 
    \\ &= \sup_{\mu \in \mathcal{M}(\mathbb{S}^d)} - \frac{1}{n} \sum_{i=1}^n \int \sigma(\langle \theta, x_i \rangle) \ d\mu(\theta)  - \frac{1}{\beta} \log \left( \int \exp \left( - \beta \int \sigma(\langle \theta, x \rangle) \ d\mu(\theta) \right) \ d\tau(x) \right),
\end{split}
\end{align}
and the right-hand side is precisely \eqref{eq:unconstrained_f1_minimizer}. In the second equality we used that the optimal choice for $g$ is $g = 0$. In the third equality we used that the optimal $\lambda$ must satisfy the first-order optimality condition:
\begin{align}
\begin{split}
    &\int \exp \left( - \beta \left( \int \sigma(\langle \theta, x \rangle) \ d\mu(\theta) - \lambda \right) - 1 \right) \ d\tau(x) - 1 = 0, \\
    &\implies e^{\beta \lambda} = \left( \int \exp \left( - \beta \int \sigma(\langle \theta, x \rangle) \ d\mu(\theta) - 1 \right) \ d\tau(x) \right)^{-1} \\ &\implies \lambda = - \frac{1}{\beta} \log \left( \int \exp \left( - \beta \int \sigma(\langle \theta, x \rangle) \ d\mu(\theta) - 1 \right) \ d\tau(x) \right) 
\end{split}
\end{align}
To prove strong duality, we need to show that
\begin{align} \label{eq:strong_duality_unconstrained}
    \inf_{\nu \in \mathcal{M}(K)} \sup_{\mu \in \mathcal{M}(\mathbb{S}^d), \lambda \in \R, g \in \mathcal{C}(K) : g \geq 0} L_1(\nu, \mu, g, \lambda) = \sup_{\mu \in \mathcal{M}(\mathbb{S}^d), \lambda \in \R, g \in \mathcal{C}(K) : g \geq 0} \inf_{\nu \in \mathcal{M}(K)} L_1(\nu, \mu, g, \lambda).
\end{align}
If we define $\tilde{L}_1 : \mathcal{P}(K) \times \mathcal{M}(\mathbb{S}^d) \rightarrow \R$ as $\tilde{L}_1(\nu, \mu) = L_1(\nu, \mu, 0, 0)$, we have that the assumptions of \autoref{thm:kneser} hold for $-\tilde{L}_1$. Indeed, by \autoref{lem:diff_entropy} we have that $-\tilde{L}_1(\cdot, \mu)$ is a concave and upper semicontinuous function of $\nu$.
And by Prokhorov's theorem, $\mathcal{P}(K)$ is a compact subset of the locally convex topological vector space of signed Radon measures with the topology of weak convergence (tightness follows from the fact that $K$ is compact). Thus,
\begin{align} \label{eq:kneser_result}
    \inf_{\nu \in \mathcal{P}(K)} \sup_{\mu \in \mathcal{M}(\mathbb{S}^d)} \tilde{L}_1(\nu, \mu) = \sup_{\mu \in \mathcal{M}(\mathbb{S}^d)} \min_{\nu \in \mathcal{P}(K)} \tilde{L}_1(\nu, \mu).
\end{align}
On the one hand, notice that by \autoref{obs:1},
\begin{align} \label{eq:rhs}
    \inf_{\nu \in \mathcal{M}(K)} \sup_{\mu \in \mathcal{M}(\mathbb{S}^d), \lambda \in \R, g \in \mathcal{C}(K) : g \geq 0} L_1(\nu, \mu, g, \lambda) = \inf_{\nu \in \mathcal{P}(K)} \sup_{\mu \in \mathcal{M}(\mathbb{S}^d)} L_1(\nu, \mu, 0, 0) = \inf_{\nu \in \mathcal{P}(K)} \sup_{\mu \in \mathcal{M}(\mathbb{S}^d)} \tilde{L}_1(\nu, \mu)
\end{align}
On the other hand,
\begin{align}
\begin{split} \label{eq:lhs}
    \sup_{\mu \in \mathcal{M}(\mathbb{S}^d), \lambda \in \R, g \in \mathcal{C}(K) : g \geq 0} \inf_{\nu \in \mathcal{M}(K)} L_1(\nu, \mu, g, \lambda) &= \sup_{\mu \in \mathcal{M}(\mathbb{S}^d)} \inf_{\nu \in \mathcal{M}(K)} \sup_{\lambda \in \R, g \in \mathcal{C}(K) : g \geq 0} L_1(\nu, \mu, g, \lambda) \\ &= \sup_{\mu \in \mathcal{M}(\mathbb{S}^d)} \min_{\nu \in \mathcal{P}(K)} L_1(\nu, \mu, 0, 0) = \sup_{\mu \in \mathcal{M}(\mathbb{S}^d)} \min_{\nu \in \mathcal{P}(K)} \tilde{L}_1(\nu, \mu).
\end{split}
\end{align}
where we have used \autoref{lem:nu_lambda_g} in the first equality, \autoref{obs:1} in the second equality and the definition of $\tilde{L}_1$ in the third equality.
Thus, the strong duality \eqref{eq:strong_duality_unconstrained} follows from plugging \eqref{eq:rhs} and \eqref{eq:lhs} into \eqref{eq:kneser_result}.

To show that \eqref{eq:unconstrained_f2_minimizer} is also a dual problem of \eqref{eq:unconstrained}, we consider the Lagrangian $L_2 : \mathcal{M}(K) \times L^2(\mathbb{S}^d) \times \mathcal{C}(K) \times \R \rightarrow \R$
defined as 
\begin{align} 
\begin{split} \label{eq:L_2_def}
    L_2(\nu, h, g, \lambda) &= \beta^{-1} \int \log \left( \frac{d\nu}{d\tau}(x) \right) d\nu(x) + \int \left( \int \sigma(\langle \theta, x \rangle) \ d\nu(x) - \frac{1}{n} \sum_{i=1}^n \sigma(\langle \theta, x_i \rangle) \right) h(\theta)\ d\tilde{\tau}(\theta) - \int g(x) d\nu(x) \\ &+ \lambda \left( \int d\nu(x) - 1 \right) - r \left( \int h^2(\theta) \ d\tilde{\tau}(\theta) - 1 \right)
\end{split}
\end{align}
The reasoning to obtain the dual problem \eqref{eq:unconstrained_f2_minimizer} is analogous. Strong duality in this case can be stated as
\begin{align} \label{eq:strong_duality_unconstrained_2}
    \inf_{\nu \in \mathcal{M}(K)} \sup_{h \in L^2(\mathbb{S}^d), \lambda \in \R, g \in \mathcal{C}(K) : g \geq 0} L_2(\nu, h, g, \lambda) = \sup_{h \in L^2(\mathbb{S}^d), \lambda \in \R, g \in \mathcal{C}(K) : g \geq 0} \inf_{\nu \in \mathcal{M}(K)} L_1(\nu, h, g, \lambda).
\end{align}
Analogously, we define $\tilde{L}_2 : \mathbb{P}(K) \times L^2(\mathbb{S}^d) \rightarrow \R$ as $\tilde{L}_2(\nu, h) = L_2(\nu, h, 0, 0)$, and we have that the assumptions of \autoref{thm:kneser} hold for $-\tilde{L}_2$ as well, implying that $\inf_{\nu \in \mathcal{P}(K)} \sup_{h \in L^2(\mathbb{S}^d)} \tilde{L}_2(\nu, h) = \sup_{h \in L^2(\mathbb{S}^d)} \min_{\nu \in \mathcal{P}(K)} \tilde{L}_2(\nu, h).$ The concluding argument is also analogous.
\end{proof}

\begin{lemma} \label{lem:nu_lambda_g}
For all $\mu \in \mathcal{M}(\mathbb{S}^d)$,
\begin{align}
\sup_{\lambda \in \R, g \in \mathcal{C}(K) : g \geq 0} \inf_{\nu \in \mathcal{M}(K)} L_1(\nu, \mu, g, \lambda) = \min_{\nu \in \mathcal{M}(K)} \sup_{\lambda \in \R, g \in \mathcal{C}(K) : g \geq 0} L_1(\nu, \mu, g, \lambda).
\end{align}
\end{lemma}

\begin{proof}
First, notice that by \eqref{eq:lagrange_df} and \eqref{eq:lagrange_dual_p},
\begin{align}
\begin{split} \label{eq:max_min_eq}
    &\sup_{\lambda \in \R, g \in \mathcal{C}(K) : g \geq 0} \inf_{\nu \in \mathcal{M}(K)} L_1(\nu, \mu, g, \lambda) = \sup_{\lambda \in \R, g \in \mathcal{C}(K) : g \geq 0} F_1(\mu,g,\lambda) \\ &=
    - \frac{1}{n} \sum_{i=1}^n \int \sigma(\langle \theta, x_i \rangle) \ d\mu(\theta)  - \frac{1}{\beta} \log \left( \int \exp \left( - \beta \int \sigma(\langle \theta, x \rangle) \ d\mu(\theta) \right) \ d\tau(x) \right),
\end{split}
\end{align}
And by \autoref{obs:1},
\begin{align} 
\begin{split} \label{eq:min_max_eq}
    \min_{\nu \in \mathcal{M}(K)} \sup_{\lambda \in \R, g \in \mathcal{C}(K) : g \geq 0} L_1(\nu, \mu, g, \lambda) = \min_{\nu \in \mathcal{P}(K)} L_1(\nu, \mu, 0, 0) = \min_{\nu \in \mathcal{P}(K)} L_1(\nu, \mu, 0, 0)
\end{split}
\end{align}
If $\nu^{\star} \in \mathcal{P}(K)$ is a minimizer of $\min_{\nu \in \mathcal{P}(K)} L_1(\nu, \mu, 0, 0)$, it must fulfill
\begin{align}
\begin{split}
    &\exists C \in \R: \quad C = \frac{dL_1}{d\nu}(\nu^{\star},\mu,0,0) = \beta^{-1} \log \left( \frac{d\nu}{d\tau}(x) \right) + \beta^{-1} + \int \sigma(\langle \theta, x \rangle) \ d\mu(\theta), \\ 
    &\implies \frac{d\nu^{\star}}{d\tau(x)}(x) = \exp \left( - \beta \int \sigma(\langle \theta, x \rangle) \ d\mu(\theta) + \beta C - 1 \right),
\end{split}
\end{align}
where $- \beta C + 1 = \log \left(\int_K \exp \left( - \beta \int \sigma(\langle \theta, x \rangle) \ d\mu(\theta) \right) d\tau(x) \right)$. Hence, 
\begin{align}
\begin{split}
    &L_1(\nu^{\star}, \mu, 0, 0) = \beta^{-1} \int \log \left( \frac{d\nu^{\star}}{d\tau(x)}(x) \right) d\nu^{\star}(x) + \int \left( \int \sigma(\langle \theta, x \rangle) \ d\nu^{\star}(x) - \frac{1}{n} \sum_{i=1}^n \sigma(\langle \theta, x_i \rangle) \right) \ d\mu(\theta) \\ &= - \frac{1}{n} \sum_{i=1}^n \int \sigma(\langle \theta, x_i \rangle) \ d\mu(\theta) + \int \left(C - \beta^{-1} \right) \ d\nu^{\star}(x)
    \\ &= - \frac{1}{n} \sum_{i=1}^n \int \sigma(\langle \theta, x_i \rangle) \ d\mu(\theta) - \frac{1}{\beta} \log \left(\int_K \exp \left( - \beta \int \sigma(\langle \theta, x \rangle) \ d\mu(\theta) \right) d\tau(x) \right)
\end{split}
\end{align}
If we plug this into the right-hand side of \eqref{eq:min_max_eq}, we obtain exactly the right-hand side of \eqref{eq:max_min_eq}, concluding the proof. 
\end{proof}

\begin{restatable}{thm}{fenchel} [Fenchel strong duality; \citet{borwein2005techniques}, pp. 135-137] 
\label{thm:fenchel}
Let $X$ and $Y$ be Banach spaces, 
$f:X\to \mathbb {R} \cup \{+\infty \}$ and 
$g:Y\to \mathbb {R} \cup \{+\infty \}$ be convex functions and 
$A:X\to Y$ be a bounded linear map. Define the Fenchel problems:
\begin{align}
\begin{split}
p^{*} &= \inf _{x\in X}\{f(x)+g(Ax)\} \\
d^{*} &= \sup _{y^{*}\in Y^{*}}\{-f^{*}(A^{*}y^{*})-g^{*}(-y^{*})\},
\end{split}
\end{align}
where $f^{*}(x^{*}) = \sup_{x \in X} \{ \langle x, x^{*} \rangle - f(x)\}, \ g^{*}(y^{*}) = \sup_{y \in Y} \{ \langle y, y^{*} \rangle - g(y)\}$ are the convex conjugates of $f,g$ respectively, and $A^{*} : Y^{*} \to X^{*}$ is the adjoint operator. Then, $p^{*} \geq d^{*}$. Moreover if $f,g,$ and $A$ satisfy either
\begin{enumerate}
\item $f$ and $g$ are lower semi-continuous and 
$0\in \operatorname{core} (\operatorname{dom} g-A \operatorname{dom} f)$ where 
$\operatorname{core}$ is the algebraic interior and 
$\operatorname{dom} h$, where $h$ is some function, is the set 
$\{z:h(z)<+\infty \}$, 
\item or
$A\operatorname{dom} f\cap \operatorname{cont} g\neq \emptyset$  where 
$\operatorname{cont}$ are is the set of points where the function is continuous.
\end{enumerate}
Then strong duality holds, i.e. 
$p^{*}=d^{*}$. If 
$d^{*}\in \mathbb {R}$ then supremum is attained.
\end{restatable}

\strongdualityfone*
\begin{proof}
One way to prove \autoref{thm:strongdualityf1} (and \autoref{thm:strongdualityf2}) would be to develop an argument based on a modification of the Lagrangian function $L_1$ (resp. $L_2$) that encodes the $\mathcal{F}_1$ restriction (resp. $\mathcal{F}_2$), and to reduce the problem once again to a min-max duality result like \autoref{thm:kneser}. However, this method turns out to be rather cumbersome, and we resort to an alternative approach that harnesses the power of Fenchel duality theory and yields a much faster proof. In fact, our proof structure is similar to Theorem 12.2 of \citet{mohri2012foundations}, which focuses on the finite-dimensional case and deals with a slightly different problem. As shown by \autoref{thm:fenchel}, the Fenchel strong duality sufficient conditions are very similar in the Euclidean and in the Banach space settings.

We will use \autoref{thm:fenchel} with $X = \mathcal{M}(K)$, i.e. the Banach space of signed Radon measures, and $Y = \mathcal{C}(\mathbb{S}^d)$, the Banach space of continuous functions on $\mathbb{S}^d$. Define $f: \mathcal{M}(K) \to \R \cup \{+\infty \}$ as
\begin{align}
    f(\nu) = 
    \begin{cases}
        \beta^{-1} D_{KL}(\nu||\tau) &\text{if } \nu \in \mathcal{P}(K), \\
        +\infty &\text{otherwise}
    \end{cases}
\end{align}
Define $g: \mathcal{C}(\mathbb{S}^d) \to \R \cup \{+\infty \}$ as
\begin{align}
    g(\phi) = \max_{\theta \in \mathbb{S}^d} \left|\phi(\theta) - \int_{K} \sigma(\langle \theta, x \rangle) \ d\nu_n(x) \right|,  
\end{align}
and $A : \mathcal{M}(K) \to \mathcal{C}(\mathbb{S}^d)$ as $(A \nu)(\theta) = \int_{K} \sigma(\langle \theta, x \rangle) d\nu(x)$. Remark that $A$ is a bounded linear operator, which implies that it has an adjoint operator. By the Riesz-Markov-Kakutani representation theorem, we have that $\mathcal{C}(\mathbb{S}^d)^{*} = \mathcal{M}(\mathbb{S}^d)$, which means that the adjoint of $A$ is of the form $A^{*} : \mathcal{M}(\mathbb{S}^d) \to \mathcal{M}(K)^{*}$. By the definition of the adjoint operator, we have that for any $\mu \in \mathcal{M}(\mathbb{S}^d), \nu \in \mathcal{M}(K)$,
\begin{align} \label{eq:A_adjoint}
    \langle A^{*} \mu, \nu \rangle_{\mathcal{M}(K)^{*}, \mathcal{M}(K)} = \langle \mu, A \nu \rangle_{\mathcal{M}(\mathbb{S}^d), \mathcal{C}(\mathbb{S}^d)} = \int_{\mathbb{S}^d} \int_{K} \sigma(\langle \theta, x \rangle) \ d\nu(x) d\mu(\theta) = \int_{K} \int_{\mathbb{S}^d} \sigma(\langle \theta, x \rangle) \ d\mu(\theta) d\nu(x)
\end{align}
Notice that $\mathcal{C}(K) \subseteq \mathcal{M}(K)^{*}$ by the fact that a vector space has a natural embedding in its continuous bidual (but the continuous bidual is in general larger). Through this identification, \eqref{eq:A_adjoint} implies that we can write $A^{*} \mu (x) = \int_{\mathbb{S}^d} \sigma(\langle \theta, x \rangle) d\mu(\theta)$. 

Our goal now is to compute the convex conjugates $f^{*}$ and $g^{*}$. By the argument of Lemma B.37 of \cite{mohri2012foundations}, which works in the infinite-dimensional case as well, the convex conjugate $f^{*}: \mathcal{C}(K) \to \R \cup \{+\infty \}$ is shown to be:
\begin{align}
    f^{*}(\psi) = \frac{1}{\beta} \log \left(\int_K \exp \left(\beta \psi(x) \right) d\tau(x) \right)
\end{align}
Remark that $f^{*}$ has domain $\mathcal{M}(K)^{*}$, which is larger than $\mathcal{C}(K)$. However, knowing the restriction of $f^{*}$ to $\mathcal{C}(K)$ will suffice for our purposes.

Moreover, $g^{*} : \mathcal{M}(\mathbb{S}^d) \to \R \cup \{+\infty \}$ fulfills:
\begin{align}
\begin{split}
    g^{*}(\mu) &= \sup_{\phi \in \mathcal{C}(\mathbb{S}^d)} \left\{\int_{\mathbb{S}^d} \phi \ d\mu - \max_{\theta \in \mathbb{S}^d} \left|\phi(\theta) - \int_{K} \sigma(\langle \theta, x \rangle) \ d\nu_n(x) \right| \right\} 
    \\ &= \sup_{\phi \in \mathcal{C}(\mathbb{S}^d)} \left\{\int_{\mathbb{S}^d} \phi \ d\mu - \sup_{\substack{\mu' \in \mathcal{M}(\mathbb{S}^d), \\ |\mu'|_{TV} \leq 1}} \int_{\mathbb{S}^d} \left(\phi(\theta) - \int_{K} \sigma(\langle \theta, x \rangle) \ d\nu_n(x) \right) d\mu'(\theta) \right\}
    \\ &= \sup_{\phi \in \mathcal{C}(\mathbb{S}^d)} \inf_{\substack{\mu' \in \mathcal{M}(\mathbb{S}^d), \\ |\mu'|_{TV} \leq 1}} \left\{\int_{\mathbb{S}^d} \phi(\theta) \ d(\mu-\mu')(\theta) + \int_{\mathbb{S}^d} \int_{K} \sigma(\langle \theta, x \rangle) \ d\nu_n(x) d\mu'(\theta) \right\}
    \\ &= \inf_{\substack{\mu' \in \mathcal{M}(\mathbb{S}^d), \\ |\mu'|_{TV} \leq 1}} \sup_{\phi \in \mathcal{C}(\mathbb{S}^d)} \left\{\int_{\mathbb{S}^d} \phi(\theta) \ d(\mu-\mu')(\theta) + \int_{\mathbb{S}^d} \int_{K} \sigma(\langle \theta, x \rangle) \ d\nu_n(x) d\mu'(\theta) \right\}
    \\ &= 
    \begin{cases}
        \int_{\mathbb{S}^d} \int_{K} \sigma(\langle \theta, x \rangle) \ d\nu_n(x) d\mu(\theta) &\text{if } |\mu|_{TV} \leq 1\\
        +\infty &\text{otherwise}
    \end{cases}
\end{split}
\end{align}
In the first equality we have used the definition of $g$, in the fourth equality we have used \autoref{thm:kneser} (remark that $\{\mu' \in \mathcal{M}(\mathbb{S}^d) : |\mu'|_{TV} \leq 1\}$ is compact in the weak convergence topology), and in the fifth equality we have used that $\sup_{\phi \in \mathcal{C}(\mathbb{S}^d)} \left\{\int_{\mathbb{S}^d} \phi(\theta) \ d(\mu-\mu')(\theta) \right\} = +\infty$ unless $\mu=\mu'$. 

With these definitions, notice that problem \eqref{eq:primal_problem_f1} can be rewritten as $\inf_{\nu \in \mathcal{M}(K)} \{f(\nu)+g(A\nu)\}$ and problem \eqref{eq:dual_problem_f1} can be rewritten as $\sup_{\mu \in \mathcal{M}(\mathbb{S}^d)} \{-f^{*}(-A^{*}\mu)-g^{*}(\mu)\}$. Thus, strong duality between \eqref{eq:primal_problem_f1} and \eqref{eq:dual_problem_f1} follows from Fenchel strong duality, which holds by checking condition 2 of \autoref{thm:fenchel}. We have to see that $A\operatorname{dom} f\cap \operatorname{cont} g\neq \emptyset$. Consider $\phi(\cdot) = \int_{K} \sigma(\langle \cdot, x \rangle) \ d\nu(x) \in \mathcal{C}(\mathbb{S}^d)$ for some $\nu \in \mathcal{P}(K)$ absolutely continuous w.r.t. $\tau$. Then, we have that $\phi \in A \operatorname{dom} f$. Moreover, since $g$ is a continuous function (in the supremum norm topology), $\operatorname{cont} g = \mathcal{C}(\mathbb{S}^d)$ and hence $\phi \in \operatorname{cont} g$ as well, which means that the intersection is not empty.

Notice that in our case $d^{*} = \sup_{\mu \in \mathcal{M}(\mathbb{S}^d)} \{-f^{*}(-A^{*}\mu)-g^{*}(\mu)\} \in \R$, which by \autoref{thm:fenchel} implies that the supremum is attained: let $\mu^{\star}$ be one maximizer. We show that $p^{*} = \inf_{\nu \in \mathcal{M}(K)} \{f(\nu)+g(A\nu)\} = \inf_{\nu \in \mathcal{P}(K)} \{f(\nu)+g(A\nu)\}$ admits a minimizer by the direct method of the calculus of variations: notice that $f$ and $g \circ A$ are lower semicontinuous in the topology of weak convergence ($f$ by \autoref{lem:diff_entropy} and $g \circ A$ because it is a maximum of continuous functions, and thus its sublevel sets are closed because they are the intersection of closed sublevel sets), and $\mathcal{P}(K)$ is compact.

We now show that $\frac{d\nu^{\star}}{d\tau}(x) = \frac{1}{Z_{\beta}}\exp \left( - \beta \int \sigma(\langle \theta, x \rangle) \ d\mu^{\star}(\theta) \right)$, where $\nu^{\star}$ and $\mu^{\star}$ are solutions of \eqref{eq:primal_problem_f1} and \eqref{eq:dual_problem_f1}, respectively, that we know exist by the previous paragraph.
Recall the argument to prove Fenchel weak duality:
\begin{align}
\begin{split}
    \sup_{\mu \in \mathcal{M}(\mathbb{S}^d)} \{-f^{*}(-A^{*}\mu)-g^{*}(\mu)\} &= -f^{*}(-A^{*}\mu^{\star})-g^{*}(\mu^{\star}) \\&= - \sup_{\nu \in \mathcal{M}(K)} \left\{\langle - A^{*}\mu^{\star}, \nu \rangle - f(\nu) \right\} - \sup_{\phi \in \mathcal{C}(\mathbb{S}^d)} \left\{\langle \mu^{\star}, \phi \rangle - g(\phi) \right\} \\&\leq - \sup_{\nu \in \mathcal{M}(K)} \left\{\langle - A^{*}\mu^{\star},\nu \rangle - f(\nu) + \langle \mu^{\star}, A \nu \rangle - g(A \nu) \right\} \\&= - \sup_{\nu \in \mathcal{M}(K)} \left\{\langle - f(\nu) - g(A \nu) \right\} = \inf_{\nu \in \mathcal{M}(K)} \{f(\nu)+g(A\nu)\} = f(\nu^{\star})+g(A\nu^{\star})
\end{split}
\end{align}
Thus, for strong duality to hold we must have that $\nu^{\star} = \argmax_{\nu \in \mathcal{M}(K)} \left\{\langle - A^{*}\mu^{\star}, \nu \rangle - f(\nu) \right\}$, and the corresponding Euler-Lagrange condition is $\frac{d\nu^{\star}}{d\tau}(x) = \frac{1}{Z_{\beta}}\exp \left( - \beta \int \sigma(\langle \theta, x \rangle) \ d\mu^{\star}(\theta) \right)$.
\end{proof}

\strongdualityftwo*

\begin{proof}
The proof is largely analogous to the proof of \autoref{thm:strongdualityf1}. We use \autoref{thm:fenchel} with $X = \mathcal{M}(K)$ as before, and $Y = L^2(\mathbb{S}^d)$, the Hilbert space of square-integrable functions on $\mathbb{S}^d$ under the base measure $\tilde{\tau}$, which is of course self-dual. We define $f$ as before, and $g : L^2(\mathbb{S}^d) \to \R \cup \{+\infty \}$ as
\begin{align}
    g(\phi) = \left( \int_{\mathbb{S}^d} \left( \phi(\theta) - \int_K \sigma(\langle \theta, x \rangle) \ d\nu_n(x) \right)^2 d\tilde{\tau}(\theta) \right)^{1/2},  
\end{align}
and consequently, $g^{*} : L^2(\mathbb{S}^d) \to \R \cup \{+\infty \}$ fulfills:
\begin{align}
\begin{split}
    g^{*}(\psi) &= \sup_{\phi \in L^2(\mathbb{S}^d)} \left\{\int_{\mathbb{S}^d} \phi \psi \ d\tilde{\tau} - \left( \int_{\mathbb{S}^d} \left( \phi(\theta) - \int_K \sigma(\langle \theta, x \rangle) \ d\nu_n(x) \right)^2 d\tilde{\tau}(\theta) \right)^{1/2} \right\} 
    \\ &= \sup_{\phi \in L^2(\mathbb{S}^d)} \left\{\int_{\mathbb{S}^d} \phi \psi \ d\tilde{\tau} - \sup_{\substack{\hat{\psi} \in L^2(\mathbb{S}^d), \\ \|\hat{\psi}\|_2 \leq 1}} \int_{\mathbb{S}^d} \left(\phi(\theta) - \int_{K} \sigma(\langle \theta, x \rangle) \ d\nu_n(x) \right) \hat{\psi}(\theta) d\tilde{\tau}(\theta) \right\}
    \\ &= \sup_{\phi \in L^2(\mathbb{S}^d)} \inf_{\substack{\hat{\psi} \in L^2(\mathbb{S}^d), \\ \|\hat{\psi}\|_2 \leq 1}} \left\{\int_{\mathbb{S}^d} \phi(\theta) (\psi(\theta) - \hat{\psi}(\theta)) \ d\tilde{\tau}(\theta) + \int_{\mathbb{S}^d} \int_{K} \sigma(\langle \theta, x \rangle) \ d\nu_n(x) \ \hat{\psi}(\theta) \ d\tilde{\tau}(\theta) \right\},
\end{split}
\end{align}
and using \autoref{thm:kneser} once more, this is equal to:
\begin{align}
\begin{split}
    \\ &\inf_{\substack{\hat{\psi} \in L^2(\mathbb{S}^d), \\ \|\hat{\psi}\|_2 \leq 1}} \sup_{\phi \in L^2(\mathbb{S}^d)} \left\{\int_{\mathbb{S}^d} \phi(\theta) (\psi(\theta) - \hat{\psi}(\theta)) \ d\tilde{\tau}(\theta) + \int_{\mathbb{S}^d} \int_{K} \sigma(\langle \theta, x \rangle) \ d\nu_n(x) \ \hat{\psi}(\theta) \ d\tilde{\tau}(\theta) \right\}
    \\ &= 
    \begin{cases}
        \int_{\mathbb{S}^d} \int_{K} \sigma(\langle \theta, x \rangle) \ d\nu_n(x) \ \psi(\theta) \ d\tilde{\tau}(\theta) &\text{if } \|\psi\|_{2} \leq 1,\\
        +\infty &\text{otherwise}
    \end{cases}
\end{split}
\end{align}
With these definitions, notice that problem \eqref{eq:primal_problem_f2} can be rewritten as $\inf_{\nu \in \mathcal{M}(K)} \{f(\nu)+g(A\nu)\}$ and problem \eqref{eq:dual_problem_f2} can be rewritten as $\sup_{\psi \in L^2(\mathbb{S}^d)} \{-f^{*}(-A^{*}\psi)-g^{*}(\psi)\}$. The rest of the proof is analogous.
\end{proof}